%% file: NC-VSloss.tex
\pdfoutput=1
\documentclass[11pt,twoside]{article} 
\usepackage{geometry}
\usepackage{comment}
\usepackage{algorithm}
\usepackage{algpseudocode}
\usepackage{amsthm}
\newtheorem{theorem}{Theorem}
\newtheorem{corollary}{Corollary}[theorem]
\newtheorem{remark}{Remark}
\usepackage{multicol}
\usepackage{xcolor}         
\usepackage{tikz}
\usepackage[T1]{fontenc}
\usepackage{lmodern}
\usepackage{tcolorbox,wrapfig}
\usepackage[labelfont={bf},font=small]{caption}
\usepackage{subcaption}
\usepackage{rotating}
\tcbuselibrary{skins}
\usepackage[utf8]{inputenc} 
\usepackage[T1]{fontenc}    
\usepackage{hyperref}       
\usepackage{url}            
\usepackage{booktabs}       
\usepackage{amsfonts}       
\usepackage{nicefrac}       
\usepackage{microtype}      

\usepackage{graphicx}
\usepackage{lipsum}
\usepackage{makecell}
\usepackage{tabu}
\usepackage{wrapfig}
\usepackage{amsmath}
\usepackage{longtable}

\usepackage[style=alphabetic,natbib=true,backend=bibtex]{biblatex}
\input{./sections/commands}

\author{Tina Behnia$^{\dagger\mathsection}$, Ganesh Ramachandra Kini$^{\ddag\mathsection}$, Vala Vakilian$^{\dagger\mathsection}$, Christos Thrampoulidis$^\dagger$
 \vspace{8pt}
 \\
$^\dagger$University of British Columbia, Canada 
\vspace{3pt}
\\
$^\ddag$University of California, Santa Barbara, USA 
\thanks{This work is supported by an NSERC Discovery Grant, NSF Grant CCF-2009030, and by a CRG8-KAUST award. The authors also acknowledge use of the Sockeye cluster by UBC Advanced Research Computing.\newline 
$\mathsection$: equal contribution \newline
Code available at: \url{https://github.com/valavakilian/Implicit_geometry} \newline
}
}

\title{On the Implicit Geometry of Cross-Entropy Parameterizations \\ for Label-Imbalanced Data}

\begin{document}
\maketitle


%
%
%
%

\begin{abstract}
	\input{./sections/abstract}
\end{abstract}

\addtocontents{toc}{\protect\setcounter{tocdepth}{0}}
\vspace{-0.4cm}
\section{Introduction}\label{sec:intro}
\input{./sections/intro}

\vspace{-0.3cm}
\section{Background}\label{sec:background}
\input{./sections/background}
\vspace{-0.3cm}
\section{An Implicit Geometry View}\label{sec:setup}
\input{./sections/setup}
\section{CS-SVM Geometries}\label{sec:UF-SVM}
\input{./sections/new_theorem}


\section{Numerical Results}\label{sec:num_results}
\input{sections/results_ufm.tex}
\input{sections/results_deepnet.tex}
\input{sections/results_la_and_wce.tex}
\section{On Generalization}\label{sec:test_results}
\input{sections/test_results.tex}
\section{Concluding Remarks}\label{sec:rel2}
\input{./sections/more_related_work}




\printbibliography
\newpage
\clearpage
\appendix
\onecolumn
\addtocontents{toc}{\protect\setcounter{tocdepth}{3}}
\tableofcontents

\input{./sections/Zhat-V2}

\section{Proof of Theorem \ref{thm:SVM-VS}}\label{sec:proof_SVM}
\input{./sections/proof_SVM-V2}
\section{Closed-Form Formulas for the \texorpdfstring{$(\deltab,R)$}--SELI geometry}\label{sec:SELI_properties}
\input{./sections/SELI_properties}
\section{Numerical Results}\label{sec:Additional_Experiments}
\input{./sections/Additional_Experiments}

\end{document}

%% file: sections/commands.tex
\usepackage{hyperref,graphicx,amsmath,amsfonts,amssymb,bm,url,breakurl,epsf,color,MnSymbol,mathbbol,fmtcount,semtrans,caption,multirow,comment, boldline}
\usepackage{wrapfig}
\usepackage{enumitem}
\usepackage{amssymb}

\usepackage[utf8]{inputenc} 
\usepackage[T1]{fontenc}    
\usepackage{url}            
\usepackage{booktabs}       
\usepackage{amsfonts}       
\usepackage{nicefrac}       
\usepackage{microtype}      
\usepackage{dsfont}

\usepackage{xcolor}

\newlist{myitemize}{enumerate}{1}
\setlist[myitemize]{font=\color{darkred}\bfseries\itshape}

  \usepackage{mathtools}

\usepackage{titlesec}

\usepackage{tikz}
\usepackage{pgfplots}
\usetikzlibrary{pgfplots.groupplots}

\usepackage{movie15}

\usepackage{caption}
\usepackage[bottom,hang,flushmargin]{footmisc} 

\newcommand{\tsn}[1]{{\left\vert\kern-0.25ex\left\vert\kern-0.25ex\left\vert #1 
    \right\vert\kern-0.25ex\right\vert\kern-0.25ex\right\vert}}

\definecolor{darkred}{RGB}{150,0,0}
\definecolor{darkgreen}{RGB}{0,150,0}
\definecolor{darkblue}{RGB}{0,0,200}
\hypersetup{colorlinks=true, linkcolor=darkred, citecolor=darkgreen, urlcolor=darkblue}

\newenvironment{fminipage}%
  {\begin{Sbox}\begin{minipage}}%
  {\end{minipage}\end{Sbox}\fbox{\TheSbox}}

\newcommand{\delmaj}{\delta_{\text{maj}}}
\newcommand{\delmin}{\delta_{\text{minor}}}

\newcommand{\LDT}{\text{LDT}}
\newcommand{\CDT}{\text{CDT}}
\newcommand{\CE}{\text{CE,binary}}

\newcommand{\Xib}{\mathbf{\Xi}}

\newcommand{\Mbhat}{\hat{\M}}

\newtheorem{lemma}{Lemma}[section]
\newtheorem{propo}{Proposition}
\newtheorem{definition}{Definition}

\newcommand{\Cos}[2]{\operatorname{cos}(#1,#2)}

\newcommand{\wmaj}{\w_{\rm{maj}}}
\newcommand{\wmin}{\w_{\rm{minor}}}
\newcommand{\hmaj}{\h_{\rm{maj}}}
\newcommand{\hmin}{\h_{\rm{minor}}}

\newcommand{\nmin}{n_{\text{min}}}

\newcommand{\rhobar}{\overline{\rho}}

\newcommand{\SEL}{{SEL}}

\newcommand{\SELI}{{SELI}}

\newcommand{\dmin}{\delta_{\operatorname{minor}}}
\newcommand{\dmaj}{\delta_{\operatorname{maj}}}

\newcommand{\Sec}{Sec.}

\DeclareMathOperator{\tr}{tr}

\DeclareMathOperator{\diag}{diag}

\newcommand{\cut}[1]{\textcolor{red}{}}
\newcommand{\W}{\mathbf{W}}

\newcommand{\M}{\mathbf{M}}
\newcommand{\Z}{\mathbf{Z}}
\newcommand{\Ub}{\mathbf{U}}

\newcommand{\Rb}{\mathbf{R}}

\newcommand{\Wopt}{\W^*}
\newcommand{\Hopt}{\Hb^*}
\newcommand{\Mopt}{\Mb^*}
\newcommand{\wopt}{\w^*}
\newcommand{\hopt}{\h^*}

\newcommand{\Hb}{{\mathbf{H}}}
\newcommand{\Mb}{{\mathbf{M}}}
\newcommand{\Hhat}{{\hat{\Hb}}}
\newcommand{\What}{{\hat{\W}}}

\newcommand{\Zhat}{{\hat{\Z}}}

\newcommand{\Sb}{\mathbf{S}}
\newcommand{\X}{\mathbf{X}}

\newcommand{\Vb}{\mathbf{V}}

\newcommand{\Db}{\mathbf{D}}
\newcommand{\Pb}{\mathds{P}}

\newcommand{\mub}{\boldsymbol{\mu}}

\newcommand{\thetab}{{\boldsymbol{\theta}}}

\newcommand{\x}{\mathbf{x}}
\newcommand{\ub}{\mathbf{u}}
\newcommand{\w}{\mathbf{w}}

\newcommand{\vb}{\mathbf{v}}

\newcommand{\Bb}{\mathbf{B}}

\newcommand{\eb}{\mathbf{e}}

\newcommand{\h}{\mathbf{h}}

\newcommand{\Lc}{\mathcal{L}}

\newcommand{\Qc}{\mathbb{Q}}

\newcommand{\beq}{\begin{equation}}
\newcommand{\eeq}{\end{equation}}
\newcommand{\bea}{\begin{align}}
\newcommand{\eea}{\end{align}}

\newcommand{\R}{\mathbb{R}}
\newcommand{\N}{\mathbb{N}}

\newcommand{\nn}{\notag}

    \newcommand{\Lambdab}{\boldsymbol{\Lambda}}                     

 \newcommand{\deltab}{\boldsymbol\delta}

\DeclarePairedDelimiterX{\inp}[2]{\langle}{\rangle}{#1, #2}

\newcommand{\wt}{\widetilde}

\newcommand{\Id}{\mathds{I}}
\newcommand{\ones}{\mathds{1}}

\providecommand{\norm}[1]{\lVert#1\rVert}

%% file: sections/abstract.tex
Various logit-adjusted parameterizations of the cross-entropy (CE) loss have been proposed as alternatives to weighted CE for training large models on label-imbalanced data far beyond the zero train error regime. The driving force behind those designs has been the theory of \emph{implicit bias}, which for linear(ized) models, explains why they successfully induce bias on the optimization path towards solutions that favor minorities. Aiming to extend this theory to non-linear models, we investigate the \emph{implicit geometry} of classifiers and embeddings that are learned by different CE parameterizations. Our main result characterizes the global minimizers of a non-convex cost-sensitive SVM classifier for the unconstrained features model, which serves as an abstraction of deep nets. We derive closed-form formulas for the angles and norms of classifiers and embeddings as a function of the number of classes, the imbalance and the minority ratios, and the loss hyperparameters. Using these, we show that  logit-adjusted parameterizations can be appropriately tuned to learn symmetric geometries irrespective of the imbalance ratio. We complement our analysis with experiments and an empirical study of convergence accuracy in deep-nets.

%% file: sections/intro.tex
In the modern overparameterized regime, when training continues beyond zero-training error, traditional techniques, such as oversampling minorities or minimizing a weighted cross-entropy (CE) loss can be ineffective in mitigating label-imbalances  \citep{byrd2019effect,sagawa2020investigation}. In a growing literature, several alternatives have been proposed to guarantee equitable performance across majorities and minorities \citep[e.g.,][]{Menon,CDT,VS,TengyuMa,khan2017cost,lin2018focal,KimKim,kang2020decoupling}. Among these, the vector-scaling (VS) loss \citep{VS,CDT} introduces multiplicative hyperparameters on the logits of the CE loss. 



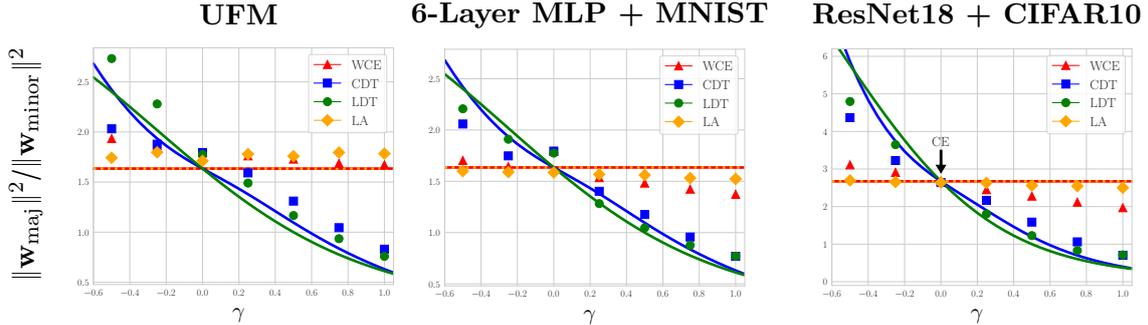
\begin{figure*}[ht]
\centering
\begin{subfigure}{.32\textwidth}
  \centering
  \begin{tikzpicture}
	\node at (0,0.0)
    {\includegraphics[width=\linewidth]{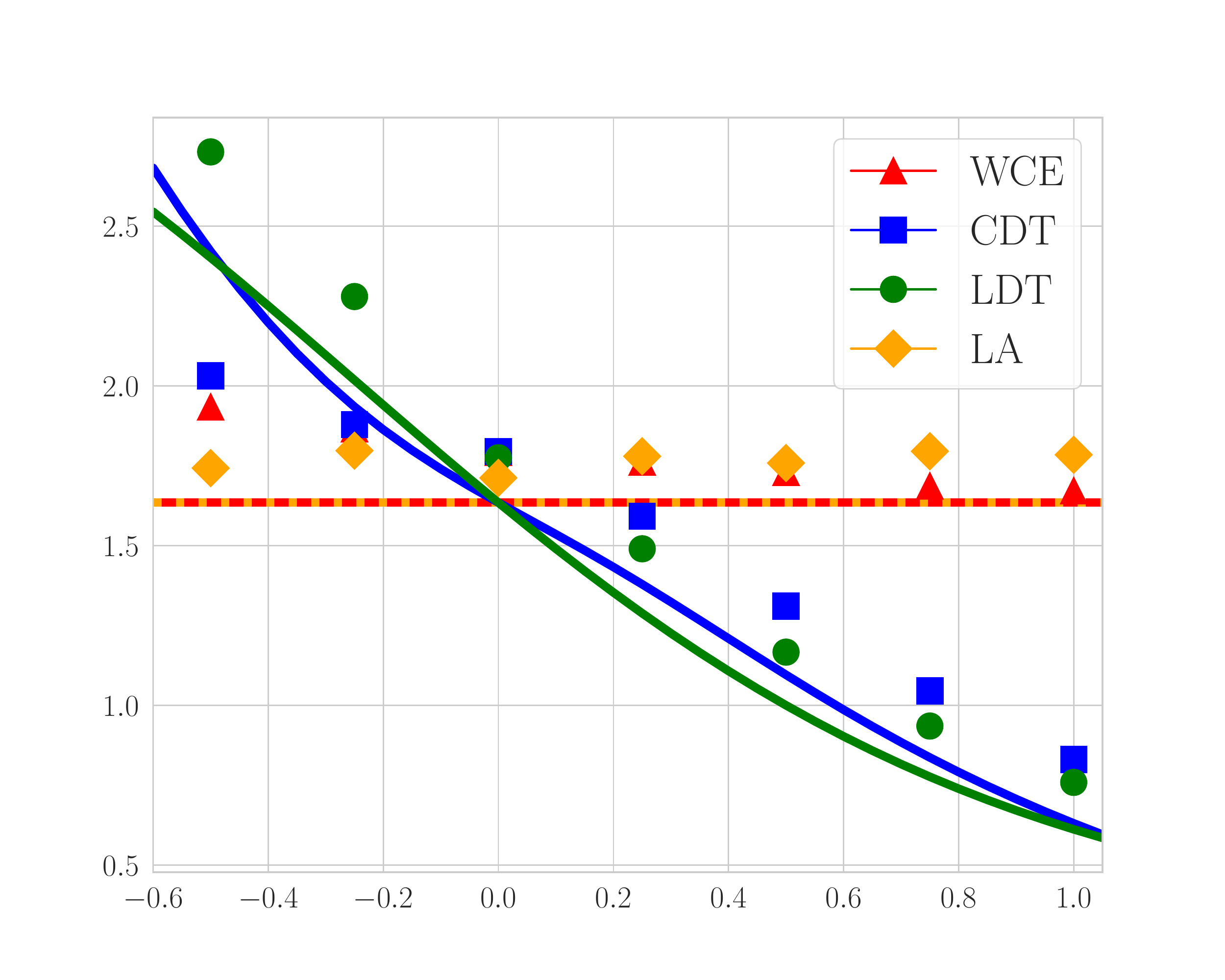}};
    \node at (-0.0,2) [scale=1.0]{\textbf{UFM}};
	\node at (-0.0,-2) [scale=0.75]{$\gamma$};
	\node at (-2.8,-0.0)  [scale=1.0, rotate=90]{$ \norm{\wmaj}^2 / \norm{\wmin}^2$};
  \end{tikzpicture}
\end{subfigure}%
\begin{subfigure}{.32\textwidth}
  \centering
  \begin{tikzpicture}
	\node at (0,0.0)
    {\includegraphics[width=\linewidth]{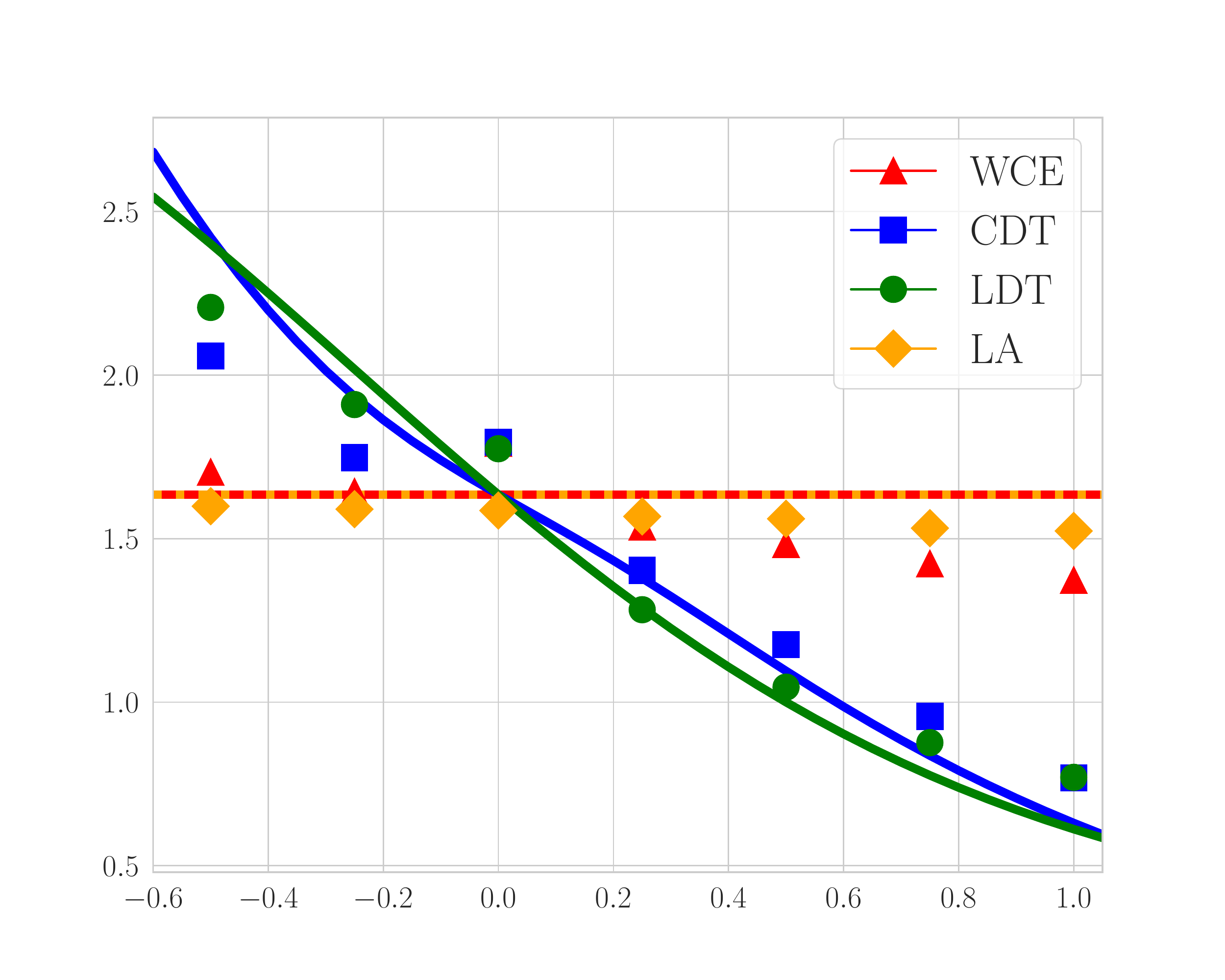}};
    \node at (-0.0,2) [scale=1.0]{\textbf{6-Layer MLP + MNIST}};
	\node at (-0.0,-2) [scale=0.75]{$\gamma$};
  \end{tikzpicture}
\end{subfigure}%
\begin{subfigure}{.32\textwidth}
  \centering
  \begin{tikzpicture}
	\node at (0,0.0)
    {\includegraphics[width=\linewidth]{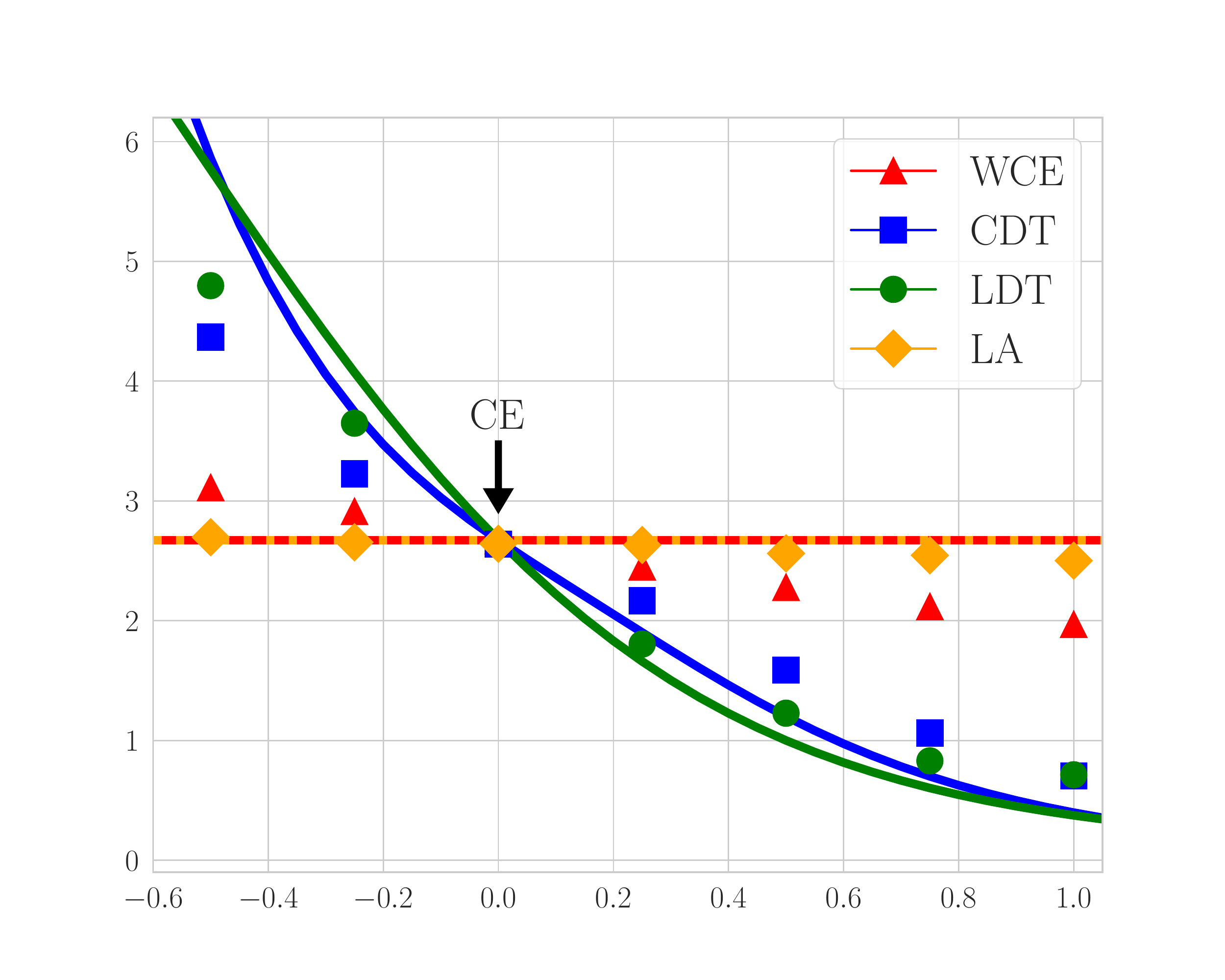}};
    \node at (-0.0,2) [scale=1.0]{\textbf{ResNet18 + CIFAR10}};
    \node at (-0.0,-2) [scale=0.75]{$\gamma$};
  \end{tikzpicture}
\end{subfigure}
\captionsetup{width=0.95\linewidth}
\vspace{-0.3cm}
\caption{Ratio of classifier norms between majorities and minorities for ($R=10$, $\rho$=1/2)-STEP imbalanced data {(see Defn.~\ref{def:step})} on (Left) UFM, (Middle) 6-layer MLP and MNIST, (Right) ResNet18 and CIFAR10. We train four different CE parameterizations with varying hyperparameter $\gamma\in[-0.5,1]$: (1) weighted CE  with  ${{\omega_\text{minor}}}/{\omega_\text{maj}} := R^\gamma$; (2,3) CDT, LDT losses with $\Delta := {\delta_{\text{maj}}}/{\delta_{\text{minor}}} = R^\gamma$; (4) LA loss with $\iota_\text{maj}-\iota_\text{minor} = \gamma \log{R}$. $\gamma=0$ corresponds to CE loss. Markers denote empirically measured quantities. Solid lines follow theoretical formulas (Eqn.~\eqref{eq:eq_intro}). See text for details. }
\label{fig:UFM_CDT_LDT_WCE_Comparison}
\vspace{-8pt}
\end{figure*}

The idea behind this parameterization is rooted in the theory of \emph{implicit bias}, which seeks characterizing the bias introduced by gradient-based algorithms during training \citep{soudry2018implicit,ji2018risk,lyu2019gradient}. Specifically for binary linear models, \cite{VS} uncovers a favorable bias of the VS loss towards classifiers with larger margin for the minority.  However, this leaves open the question how the VS loss changes the learned model in non-linear settings where embeddings and classifiers are jointly learned. Unfortunately, implicit bias characterizations for non-linear models are more obscure compared to the linear case \citep{lyu2019gradient,ji2020directional}. In particular, it is unclear how to gain concrete insights from them on the way the learned models affect minorities.

This paper investigates the \emph{implicit geometry} of classifiers and embeddings learned by CE parameterizations when trained on imbalanced data. The notion of implicit geometry,\footnote{{
Initially, \cite{NC} referred to their discovery as ``Neural Collapse'' (NC). Later, to differentiate between the geometries learned by CE for balanced and imbalanced data, \cite{seli} proposed the terms ETF and SELI geometries for the former and latter, respectively. We show here that different CE parameterizations result in yet different geometries, 
prompting us to adopt the more general term ``implicit geometry''.
}} pioneered by \cite{NC} and further investigated by many others \citep[e.g.,][]{fang2021exploring,galanti2021role,graf2021dissecting,han2021neural,hui2022limitations,ULPM,lu2020neural,mixon2020neural,tirer2022extended,xie2022neural,zhu2021geometric,zhou2022optimization,seli}, is intimately related to that of implicit bias. On the one hand, it is more restrictive as it focuses only on the classifiers and on the embeddings, rather than the weights of the entire model. Also, it is insensitive to the specific architecture or dataset. On the other hand, it offers a more explicit characterization that describes the involved geometry of the weights and promises to be ``cross-situationally invariant'' across architectures and datasets \citep{NC}.

\vspace{3pt}
\noindent\textbf{Contributions.}~We study two parameterizations of the CE loss: (i) the class-dependent temperature (CDT) loss \citep{CDT}, which is a special case of the VS loss \citep{VS}, and (ii) the label-dependent temperature (LDT) loss, which we introduce here as an alternative to the CDT loss. For both losses, we study the implicit geometry of learned features and classifiers when trained on label-imbalanced data without explicit regularization beyond zero training error. To do this, we rely on the unconstrained features model (UFM) \citep{mixon2020neural,fang2021exploring}, which serves as a proxy for large overparameterized models and has been used recently to study the implicit geometry of the CE loss (see \emph{Related work}).
Relying on the implicit bias results, we relax the question of implicit geometry of the solutions found by {stochastic gradient descent} (SGD), to a question about the geometry of the global minimizers of a non-convex \emph{Cost-Sensitive Support-Vector Machines} (CS-SVM) problem, which takes different forms for the CDT and LDT losses. {Our main result characterizes the global minimizers of the CDT and LDT CS-SVM problems in terms of a new geometry, which we call the \emph{$(\deltab,R)$-geometry} and is parameterized by a vector $\deltab$ of hyperparameters and the data imbalance ratio $R$. The new geometry has the following favorable properties: (i) It includes the previously discovered ETF \citep{NC} and SELI \citep{seli} geometries as special cases. Also, it captures both CDT and LDT. (ii) It admits an explicit characterization that involves closed-form formulas of the norms and angles in terms of the number of classes, the minority ratio, the imbalance ratio, and the vector of hyperparameters. (iii) It reveals appropriate tuning recipes for the hyperparameters to learn symmetric geometries with respect to minorities and majorities irrespective of the imbalance ratio. (iv) It shows that LDT and CDT can both mitigate minority collapse, i.e., the collapse of minority classifiers in the large imbalance-ratio limit.} Beyond these, we also show numerically that SGD training on the UFM converges to the uncovered geometries. However, we observe that convergence slows down for increasing imbalance ratios and increasing values of the hyperparameters. This observation motivates further theoretical and algorithmic investigations towards faster training with CE parameterizations. As  evidence of the utility of our geometry characterizations for the UFM, we present results on  deep-learning architectures and complex imbalanced datasets. {Additionally, we show preliminary findings regarding how the implicit geometry of different CE parametrizations might impact their generalization. Finally, building on the implicit geometry of the LDT loss, we propose a simple post-hoc rebalancing strategy that improves test performance over vanilla LDT training.}

\vspace{3pt}
\noindent\textbf{Example.}~ Fig.~\ref{fig:UFM_CDT_LDT_WCE_Comparison} provides a graphical illustration of the impact of different CE parameterizations on the implicit geometry. Here, we focus on classifiers and specifically their norms. 

In Fig.~\ref{fig:UFM_CDT_LDT_WCE_Comparison}(Right), we train a ResNet18 on a (10,1/2)-STEP imbalanced CIFAR10 dataset (see Defn.~\ref{def:step}). For the training we use four different parameterizations of the CE loss, namely the weighted CE (wCE), CDT \citep{CDT} (Eqn.~\eqref{eq:CDT_loss}), LDT (Eqn. \eqref{eq:LDT_loss}) and LA \citep{TengyuMa,Menon} losses. Each of these, comes with a set of corresponding hyperparameters, which we control by varying a single parameter $\gamma\in\R$ in the interval $[-0.5,1]$. For $\gamma=0$, all the losses reduce to standard CE loss. For each loss and for each value of $\gamma$, {we compute the ratio of the classifier norms for each pair of majority-minority classes, and the markers report the average of these ratios.}
First, observe for $\gamma=0$ (CE) that $\|\wmaj\|_2\approx2.8\|\wmin\|_2$. This is different from the case of balanced classes where ETF geometry suggests $\|\wmaj\|_2\approx \|\wmin\|_2$ \citep{NC}. The fact that, under class imbalances, CE loss learns classifiers with larger norm for majorities compared to minorities has been empirically observed in the imbalanced deep-learning literature \citep{KimKim,kang2020decoupling,Menon} and various heuristic methods have been proposed to mitigate this effect towards favoring minorities. One of these, the LA loss \citep{Menon} is seen here to have minimal effect on changing the classifiers' imbalance ratio. The wCE loss has similar behavior as the ratio reduces only marginally with increasing $\gamma$. On the other hand,  both CDT and LDT offer flexibility in tuning the ratio over a wide range by varying $\gamma$: as $\gamma$ increases the norm of minorities increases relative to the majorities. Interestingly, for appropriate $\gamma$ values the ratio can be made $1$ (as in the balanced case).

Fig.~\ref{fig:UFM_CDT_LDT_WCE_Comparison}(Middle) repeats the above experiment on a 6-layer MLP with imbalance MNIST data. The behavior is analogous: For CE the ratio is $\approx2.8$, while appropriately tuning LDT and CDT losses can tweak the classifiers' geometry and change the norm ratio. 

Finally, Fig.~\ref{fig:UFM_CDT_LDT_WCE_Comparison}(Left) repeats the experiment on the synthetic unconstrained features model (UFM) (see Sec.~\ref{sec:setup}). 
Observe that the behavior is remarkably reflective of the trends seen previously on ResNet/MLP architectures and CIFAR10/MNIST data. Compared to the latter, the UFM is amenable to mathematical analysis. Specialized to classifiers' norms, our analysis yields the following explicit formulas for the CDT/LDT solutions of the UFM for  hyperparameter $\Delta:=R^\gamma$:
\begin{align}\nn
    \text{CDT:} \ &\frac{\|\wmaj\|_2^2}{\|\wmin\|_2^2} = \frac{\frac{\sqrt{R}}{\Delta}(k-2){\big({1+\Delta^{2}}\big)^{3/2}}+2\Delta^2\sqrt{R+1}}{(k-2)\big({1+\Delta^{2}}\big)^{3/2}+2\sqrt{R+1}},\\
    \text{LDT:} \  &\frac{\|\wmaj\|_2^2}{\|\wmin\|_2^2} = \frac{(k-2)\sqrt{R}+{\sqrt{(R+\Delta^2)/2}}}{(k-2)\Delta+{\sqrt{(R+\Delta^2)/2}}}.\label{eq:eq_intro}
\end{align}
The solid blue (CDT) and green (LDT) curves graph those formulas for $k=10$ classes and imbalance ratio $R=10.$ Note that the very same formulas capture the empirical trend for UFM Fig.~\ref{fig:UFM_CDT_LDT_WCE_Comparison}(Left) and also for MLP and ResNet in Fig.~\ref{fig:UFM_CDT_LDT_WCE_Comparison}(Middle,Right). For LDT simply setting $\Delta=\sqrt{R}$ ($\gamma=1/2$) makes the norms of majorities and minorities equal. We will prove that the same choice in fact also guarantees maximal angle separation and alignment between classifiers and embeddings. On the other hand, for CDT, the value of $\Delta$ (eqv. $\gamma$) making {$\norm{\wmaj}_2 \approx \norm{\wmin}_2$} depends on $k$ and $R$ in general.


\vspace{3pt}
{\noindent\textbf{Related works.}}~In their inspiring work, \citet{NC} discover that the geometry of classifiers and embeddings that are learned by overparameterized models trained with CE far beyond zero-training error can be characterized in terms of a few simple properties. (i) \emph{Neural Collapse (NC)}: the embeddings collapse to their class means. (ii) \emph{Simplex Equiangular Tight-Frame (ETF) geometry}: the classifiers align with the embeddings of the corresponding class, they all have the same norm, and, they are maximally separated from each other. Notably, this characterization is shown to be cross-situationally invariant across different architectures and datasets. Important follow-up works \citep{mixon2020neural,fang2021exploring,graf2021dissecting} introduce the Unconstrained Features Model (UFM), as a proxy model to complex deep-nets, and uses it \citep{zhu2021geometric,zhou2022optimization,ULPM,seli,zhou2022all} to give (partial) theoretical justification of the discovery made by \cite{NC}. Extensions of the geometry characterization to mean-square loss and of the UFM to mean-square loss are also studied in \cite{mixon2020neural,zhou2022optimization,tirer2022extended}. A line of work also investigates potential connections to generalization \citep{hui2022limitations,han2021neural} and transfer-learning \citep{galanti2021role, galanti2022improved}. For example, \cite{MaDoWeNeed} and \cite{galanti2022improved} investigate the role of the NC property on model adaptation to downstream fine-grained generalization tasks. However, such connections of NC and the implicit geometry to generalization are generally not yet well understood or formalized.
All the aforementioned works on geometry characterization in presence of NC assume that data are balanced. On the other hand, when data are imbalanced, \cite{fang2021exploring} shows a \emph{minority collapse} phenomenon, i.e., the minority classifiers collapse to each other as the imbalance ratio $R$ grows to infinity. The complete geometry of both classifiers and embeddings at finite imbalance ratios was only very recently characterized in \cite{seli} under the name: \emph{Simplex-Encoded Label Interpolation (SELI)} geometry. The SELI geometry is parameterized by the imbalance ratio $R$, and it includes the ETF geometry as a special case. It also recovers  the minority collapse when evaluating angles asymptotically in $R$. Extending this literature, we formulate a new and more general geometry (which includes SELI and ETF as special cases) and show that it describes the learned embeddings and classifiers of two CE parameterizations, the CDT and the LDT losses. Closely related are also the works \cite{xie2022neural,yang2022we} which design loss functions for class-imbalanced learning in an attempt to enforce a geometry alike the ETF geometry for balanced data. However, they do not characterize the joint geometry of classifiers and embeddings as we do here. Besides, the loss functions that they consider are different in nature from the CDT and LDT losses. The latter originate from \cite{TengyuMa,Menon,CDT,VS}, which propose various logit-adjustments to the CE loss with the goal of mitigating label imbalances. Specifically, the CDT loss is proposed in \cite{CDT} and is a special case of the VS loss in \cite{VS}. Here, we also introduce a new loss, the LDT loss, and show that it forms a canonical extension of the binary VS loss of \cite{VS}. Unlike those prior works limiting their analytical studies to binary and linear models, our implicit geometry  approach allows further investigating multiclass and feature-learning regimes.
{The impact of the VS loss on the implicit geometry is also examined in independent research by \citet{lu2022importance}. However, their findings are restricted to an infinite imbalance ratio and a particular parameterization, unlike our results which hold for all finite values of the imbalance ratio and loss hyperparameters. Additionally, our proof techniques differ from theirs.}
%

\vspace{3pt}
\noindent\textbf{Notation.}~
For matrix $\Vb\in\R^{m\times n}$, $\Vb[i,j]$ denotes its $(i,j)$-th entry, $\vb_j$ denotes the $j$-th {column}, $\Vb^T$ its transpose.
$\Vb_{j:k}\in\R^{m\times (k-j+1)}$ chooses columns $j,j+1,\ldots,k$ of $\Vb$, and $\Vb^T_{j:k}\in\R^{n\times (k-j+1)}$ does so on $\Vb^T$.
We denote $\|\Vb\|_F$, and $\tr(\Vb)$ the Frobenius norm and trace of $\Vb$. We use $\Vb \propto \mathbf{X}$ whenever the two matrices are equal up to a scalar constant. {For a vector $\mathbf{v}\in\R^k$, $\text{diag}(\mathbf{v})\in\R^{k\times k}$ is the diagonal matrix with $\mathbf{v}$ on its diagonal.} $\otimes$ denotes Kronecker products.
We use $\ones_m$ to denote an  $m$-dimensional vector of all ones and $\Id_m$ for the $m$-dimensional identity matrix. For vectors/matrices with all zero entries, we simply write $0$, as dimensions are easily understood from context. {Finally, we denote the set of positive rational numbers by $\mathbb{Q}_+$.}


%% file: sections/background.tex
The Vector-Scaling (VS) loss is the following parameterization of the CE loss \citep{VS}: 
\begin{align}\label{eq:vs_multiclass}
    \Lc_\text{VS}(\W , \thetab)=:\sum_{i\in[n]}\log(1+\sum_{c\neq y_i}e^{-(\delta_{y_i}\w_{y_i}-\delta_c\w_c)^T\h_\thetab(\x_i)+\iota_{y_i}-\iota_c}).
\end{align}
Here $\x_i,i\in[n]$ are $n$ examples, $\h_\thetab(.)$ is the feature map parameterized by trainable parameters $\thetab$ (e.g. weights of hidden layers of a neural network), $y_i\in[k], i\in[n]$ are  labels, and $\w_c, c\in [k]$ are  classifier vectors (e.g. head of the network) in a $k$-class classification setting. The parameters $\delta_c$, and $\iota_c, c\in[k]$ are multiplicative and additive hyperparameters, respectively. Setting $\delta_c=1, \iota_c=0$,  recovers the CE loss. {Setting $\delta_c=1$ and only varying $\iota_c$ gives the LA loss \citep{Menon}, while setting $\iota_c=0$ and only varying $\delta_c$ gives the CDT loss \citep{CDT}.}

\vspace{3pt}
\noindent\textbf{Prior art: Binary linear classification.}~In a binary setting with fixed feature map (non-trainable $\thetab$) \cite{VS} studies the implicit bias of  binary VS loss.
\vspace{-5pt}
\begin{propo}[\cite{VS}]\label{propo:vs_binary}
 Consider a fixed feature map $\h_\thetab$, {binary labels  $\upsilon_i\in\{\pm1\}$}, $\h_i:=\h_\thetab(x_i) $ for $ i\in[n]$ {and hyperparameters $\delta_{\pm1}$}. Then GD with sufficiently small learning rate on the binary VS loss
$\Lc_\text{\emph{VS,binary}}(\w) := \sum\nolimits_{i\in[n]}\log(1+e^{-\delta_{\upsilon_i}\upsilon_i\w^T\h_i+\iota_{\upsilon_i}})$
converges (asymptotically in the number of training steps) in direction to the Cost-Sensitive SVM (CS-SVM) classifier: 
 \begin{align*}
 \arg\min_\w \norm{\w}_2~~\text{\emph{subj. to}}~\upsilon_i\delta_{\upsilon_i}\w^T\h_i\geq 1, i\in[n].
 \end{align*}
%
\end{propo}
Prop.~\ref{propo:vs_binary} explicitly describes how the hyperparameters affect training asymptotically: the GD path is implicitly biased towards a classifier that assigns margins to the two classes with relative ratio $\delta_{-1}/\delta_{+1}$. Thus, tuning $\delta_{-1}>\delta_{+1}$ if class {$\upsilon=+1$} is minority, favors the minority by assigning larger margin to it. Note, the additive hyperparameters $\iota_c$ do \emph{not} have any effect on the implicit bias {asymptotically.} Our focus here is on the asymptotic training regime, hence onwards we restrict attention to the multiplicative hyperparameters.

\vspace{3pt}
\noindent\textbf{Open problem: Beyond linear models.} Prop.~\ref{propo:vs_binary} is limited to a setting with fixed features. While an extension of the loss itself to the learned-feature setting is easy to heuristically derive (see~\eqref{eq:vs_multiclass}), it is an open question to explicitly characterize the effect of the hyperparameters on the learned solution. For instance, how do they affect the relative margin between majorities and minorities or between minorities and minorities?

%% file: sections/setup.tex
\vspace{-0.3cm}
To better understand the impact of different CE modifications, we propose studying their implicit geometry, i.e., the geometry of classifiers and embeddings learned (asymptotically in the number of training steps) by GD. 
For this, we adopt the  \emph{unconstrained features model} (UFM)  \citep{mixon2020neural,fang2021exploring}. To describe the model, 
let $
\W_{d\times k} = [\w_1,\w_2, \cdots,\w_k] 
$ and $
\Hb_{d\times n} = [\h_1, \h_2, \cdots, \h_n]
$ be the matrix of $k$ classifiers and $n$ feature embeddings corresponding to each example in the training set. Here, $d\geq k-1$ is the feature dimension. We assume each class $c\in[k]$ has $n_c\geq 1$ examples and $\sum_{c\in[k]}n_c=n.$ Without loss of generality, we assume examples are ordered. Formally, defining $n_0=0$, examples $i= \sum_{{c^\prime}=0}^{c-1}n_{c^\prime}+1,\ldots,\sum_{{c^\prime}=0}^{c}n_{c^\prime}$ are in class $c$. 
In the UFM, features $\h_i,i\in[n]$ are trained \emph{jointly} with the weights $\w_c,c\in[k]$ and are \emph{unconstrained}, i.e. trained without abiding by an explicit parameterization by some weight vector $\thetab$ (as in \eqref{eq:vs_multiclass}).
%



\vspace{3pt}
\noindent\textbf{CDT and LDT losses on the UFM.}~Consider training on the UFM with the following two parameterization of the CE loss:
\begin{subequations}\label{eq:two_losses}
\begin{align}\label{eq:CDT_loss}
    \Lc_\text{CDT}(\W^T\Hb;\deltab)&:= \sum_{i\in[n]} \log\big(1+\sum_{c\neq y_i}e^{-(\delta_{y_i}\w_{y_i}-\delta_c\w_c)^T\h_i}\big), 
\\
\label{eq:LDT_loss}
    \Lc_\text{LDT}(\W^T\Hb;\deltab)&:= \sum_{i\in[n]} \log\big(1+\sum_{c\neq y_i}e^{-(\delta_{y_i}(\w_{y_i}-\w_c))^T\h_i}\big).
\end{align}
\end{subequations}
Both losses are parameterized by a positive vector $
\deltab = \begin{bmatrix}
	\delta_1,\delta_2,\ldots,\delta_k
\end{bmatrix}^T\in\R_+^k
$ of multiplicative hyperparameters. The CDT loss in \eqref{eq:CDT_loss} was previously introduced by \cite{CDT,VS} (which is a special case of \eqref{eq:vs_multiclass} when ignoring the additive $\iota_c$). Here, we also introduce the LDT loss in \eqref{eq:LDT_loss} as an alternative parameterization.
\vspace{2pt}
\noindent\textbf{CDT vs LDT.}~Observe the following subtle  distinction: CDT associates $\deltab$ with the class label of the classifiers $\w_c$, while LDT associates the same hyperparameters with the label of the feature vectors $\h_i$.
{Our initial motivation for introducing LDT is the following observation.
}
\begin{lemma}\label{lem:motivation}
    {Assume {binary} linearly separable data and training of linear classifiers without regularization. The LDT classification rule coincides with the rule of the binary VS loss assuming same $\delta$-tuning. On the other hand, minimizing  CDT results in  the same classification rule as CE, irrespective of the $\delta$-tuning.}
\end{lemma}
In other words, for binary linear settings CDT does not improve over CE, while LDT does so by reducing to the binary VS loss of Prop.~\ref{propo:vs_binary}. {While Lem.~\ref{lem:motivation} motivates LDT, our results below show that the intuition gained from binary linear settings can be restrictive. Indeed, we show that both LDT and CDT losses induce rich behaviors in the multiclass learned-feature regime.} 

\vspace{3pt}
\noindent\textbf{Unconstrained-features cost-sensitive SVM.}~We minimize the losses in \eqref{eq:two_losses} without explicit regularization. Note that in the UFM, minimization over the embedding map is not parameterized in terms of $\thetab$, as say in \eqref{eq:vs_multiclass}. Thus, the minimization is (joint) over classifiers $\W$ and embeddings $\Hb$. Specifically, consider performing this minimization using gradient flow (i.e. GD with infinitesimal step-size.) Then, by interpreting the UFM as a two-layer linear model it can be shown following \cite{lyu2019gradient} that gradient flow will converge (asymptotically in time) in direction to a KKT point of the following two non-convex minimizations for CDT and LDT losses respectively: 
$\min_{\W,\Hb}~\|\W\|_F^2+ \|\Hb\|_F^2$
\begin{subequations}\label{eq:CS-SVMs}
\begin{align}\label{eq:svm_cdt}
&\text{subj. to}\quad (\delta_{y_i}\w_{y_i}-\delta_c\w_c)^T\h_i \geq 1,~i\in[n], c\neq y_i,\\
\label{eq:svm_ldt}
&\text{subj. to}\quad \delta_{y_i}(\w_{y_i}-\w_c)^T\h_i \geq 1,~i\in[n], c\neq y_i.
\end{align}
\end{subequations}
Note the resemblence to the CS-SVM minimization of Prop.~\ref{propo:vs_binary}. But unlike that, the problems here are non-convex since minimization is also over $\Hb$. We refer to \eqref{eq:CS-SVMs} as unconstrained CS-SVM or simply CS-SVM. 
%
%
%
\begin{remark}
It is straightforward to  extend our results to a modified objective $\|\W\|_F^2+ \beta\|\Hb\|_F^2$, for some $\beta > 0$, as also suggested in \cite{seli}. The global solutions of the two objectives have a one-to-one correspondence, differing only by an appropriate scaling factor.
\end{remark}

%% file: sections/new_theorem.tex
%
In this section, we characterize the global minimizers $(\Wopt, \Hopt)$ of the non-convex programs in \eqref{eq:svm_cdt} and \eqref{eq:svm_ldt}. {We use $\M_{d\times k}=[\mub_1,\cdots,\mub_k]$ to denote the mean-embeddings of $\Hb$, i.e. $\mub_c = (1/n_c)\sum_{i:y_i=c}\h_i, \forall c \in [k]$}. For simplicity, we focus on a STEP-imbalanced setting. In this case, it is reasonable to assume (and we do so) that $\deltab$ also shares this STEP structure.
\begin{definition}[$(R,\rho)$-STEP imbalance and STEP logit adjustment] \label{def:step}
{In a setting with imbalance ratio {$R\geq1$} and minority fraction $\rho\in(0,1)$, an $(R,\rho)$-STEP imbalanced dataset has $\rho k$ minority classes with $\nmin$ samples each, and $\rhobar k = (1-\rho) k$ majority classes with $R\nmin$ samples. For STEP logit adjustment, the hyperparameter vector $\deltab$  shares this step structure: {for majorities $\delta_c=\dmaj>0$ and for  minorities $\delta_c=\dmin>0$.}
}
\end{definition}
Our results about CDT/LDT describe the geometry of the CS-SVM solutions in terms of an encoding matrix {$\Zhat$, which we call $(\deltab,R)$-\SEL~matrix and define below together with its SVD.}

%
\begin{definition}[{$(\deltab,R)$-\SEL~matrix}]\label{def:sel}
For hyperparameters $\deltab \in \R_+^{k}$, minority fraction $\rho$ {($\bar\rho:=1-\rho$)}, and $k$ number of classes, define $\Xib \in \R^{k\times k}$ such that $\forall c,j \in [k]$,
\begin{align*}
    \Xib[c,j] = \begin{cases}
    \delta_{c}^{-1}\left(1-\nicefrac{\delta_{c}^{-2}}{\sum_{c'\in[k]}\delta_{c'}^{-2}}\right) & ,\,c=j \\
-\delta_{c}^{-1}\left(\nicefrac{\delta_{j}^{-2}}{\sum_{c'\in[k]}\delta_{c'}^{-2}}\right) & ,\,c\neq j
    \end{cases}\,.
\end{align*}
Then, {for a rational imbalance ratio $R\in\mathbb{Q_+}$},\footnote{{This assumption is not restrictive since under STEP imbalance $R:=n_{\text{maj}}/n_{\text{minor}}$ for integers  $n_{\text{maj}}, n_{\text{minor}}$.
}} the $(\deltab,R)$-Simplex-Encoding Label (\SEL) matrix $\Zhat \in \R^{k\times n}$ with {$n := \alpha k (R\rhobar + \rho $)} is defined as,
\begin{align}
    \Zhat = \begin{bmatrix}
    \Xib_{1:\rhobar k} \otimes {\ones_{\alpha R}^T} & \Xib_{(\rhobar k + 1):k} {\otimes \ones_{\alpha}^T}
    \end{bmatrix},
\end{align}
{where $\alpha \in \N$ is such that $\alpha R$ is an integer.}
Further let 
\begin{align}\label{eq:SVD_zhat}
    \Zhat = \Vb
    \Lambdab \begin{bmatrix}
    \Ub^T_{1:\rhobar k} {\otimes \ones_{\alpha R}^T} & \Ub^T_{(\rhobar k + 1):k} {\otimes \ones_{\alpha}^T}
    \end{bmatrix},
\end{align}
be the compact SVD of $\Zhat$, where $\Lambdab\in\R^{(k-1)\times (k-1)}$ is a positive diagonal matrix and $\Ub\in\R^{k\times (k-1)}$, $\Vb\in\R^{k\times(k-1)}$ have orthonormal columns.
\end{definition}
{The pattern of the $(\deltab,R)$-\SEL~matrix $\Zhat$ is clearly determined by the imbalance ratio $R$ and the hyperparameters $\deltab$. However, it also depends on {the number of classes $k$ and the minority ratio $\rho.$ We choose to drop the latter dependence from the name $(\delta,R)$-SEL since our results focus on $R, \deltab$ and $k,\rho$ are easily understood from context.}
When $\deltab=\ones_k$, {$\Zhat$ takes a special form: it reduces to a matrix with entries $1-1/k$ and $-1/k$, which \cite{seli} calls the SEL matrix and shows that it characterizes the implicit geomtery of the CE loss for imbalanced data. Our definition is strictly more general allowing us}
to describe the implicit geometry learned by CDT/LDT losses.} {We gather useful properties about the eigen-structure of $\Zhat$ in Sec.~\ref{sec:eigen_SEL}.
} Here, we note that $\Zhat^T\diag{(\deltab)}^{-1}\ones_k=0$. Thus, $\operatorname{rank}(\Zhat)=k-1$.
%
%
%
%
%
%
The $(\deltab,R)$-\SEL~matrix and its SVD induce a geometry, which is central to our results and we define it next.

\begin{definition}[{$(\deltab,R)$-\SELI~geometry}]\label{dfn:seli}
Consider a $(\deltab,R)$-\SEL~matrix $\Zhat$, with SVD factors $\Ub$, $\Lambdab$ and $\Vb$ as defined in \eqref{eq:SVD_zhat}. The classifier and mean-embeddings matrices $\W, \M \in \R^{d \times k}$ follow the $(\deltab,R)$-\SELI~geometry if the following conditions are satisfied:

\vspace{5pt}
\noindent\emph{\textbf{(i)}}~~$\W^T\W \propto \Vb \Lambdab \Vb^T$, \ \ 
\noindent\emph{\textbf{(ii)}}~~$\M^T\M \propto \Ub \Lambdab \Ub^T$,\ \ 
\noindent\emph{\textbf{(iii)}}~$\W^T\M \propto \Vb \Lambdab \Ub^T=\Xib$.
\end{definition}
{The first two statements characterize the relative norms and pair-wise angles of classifiers and mean embeddings, respectively. The third statement determines the relative margins between classes. The characterization is in terms of the SVD factors of an  appropriate SEL-type encoding matrix. In Sec.~\ref{sec:SELI_properties}, we  derive closed-form expressions for the norms, angles and margins as a function of $R, k, \deltab$ by {explicitly computing} the SVD factors of $\Zhat$. Setting $(\deltab=\ones_k,R)$ recovers the \SELI~geometry \citep{seli}, and $(\deltab=\ones_k,R=1)$ the ETF geometry  \citep{NC}.}
{We are now ready to state our main result. See Sec.~\ref{sec:proof_SVM} for proofs.}
\begin{theorem}
\label{thm:SVM-VS}
Suppose $d\geq k-1$ and $(R,\rho)$-STEP imbalance setting with STEP logit adjustments. Let $(\Wopt,\Hopt)$ be any minimizers of either \eqref{eq:svm_cdt} and \eqref{eq:svm_ldt}, and $\Mopt$ be the optimal class-wise mean-embeddings. 
Then, the following statements are true:

\item[\textbf{(i)}]\label{thm:NC} 
\noindent\emph{\textbf{[NC]}} All embeddings collapse to their class means, i.e., $\forall i\in[n]$ it holds that $\h^*_{i}=\mub^*_{y_i}.$
\item[\textbf{(ii)}]\label{thm:CDT} 
\noindent\emph{\textbf{[CDT \eqref{eq:svm_cdt}]}}~For CDT, $(\Wopt,\Mopt)$ follow the $(\deltab,R)$-\SELI~geometry.
\item[\textbf{(iii)}]\label{thm:LDT} 
\noindent\emph{\textbf{[LDT \eqref{eq:svm_ldt}]}}~For LDT, $(\Wopt,\Mopt \diag{(\deltab)})$ follow the $(\ones_k,\Tilde{R})$-\SELI~geometry, where $\Tilde{R} := R\big({\dmin}/{\dmaj}\big)^2$, provided $\Tilde{R}\in \Qc_+.$\footnote{This a technical requirement. In our experiments we apply the same formulas even when $\Tilde{R}$ is not rational.}
%
%
%
%
\end{theorem}
Thm.~\ref{thm:SVM-VS} describes the geometry of both classifiers and embeddings {that correspond to solutions of the non-convex CS-SVM for either CDT or LDT}.
Statement (i) shows that all optimal embeddings within the same class 
are equal. Thus, to analyze their geometry, it suffices to study their respective class means, which we arrange as columns of $\Mopt$. Statements (ii) and (iii) describe the optimal classifiers and mean-embeddings in terms of the  geometry in Defn.~\ref{dfn:seli}. {Hence, we can find the angles and norms (up to a constant) of the classifiers/embeddings. It is also easy to see that the geometry only depends on the ratio $\Delta:=\dmaj/\dmin$ and not on the absolute magnitude of the hyperparameters.} 

When $\deltab=\ones_k$, i.e., when the model is trained by CE loss, both statements (ii) and (iii) reduce to the \SELI~geometry of \cite{seli}. Further assuming $R=1$ (i.e., a balanced training set), {recovers the ETF geometry \citep{NC}}. For general $R$ and tuning of $\deltab$, the LDT/CDT geometries are different than both the SELI and ETF geometries. We visualize changes in the geometry in Fig.~\ref{fig:visual}. 

\noindent\textbf{Angles and Norms.}~{Expressing the geometry of the optimal solutions in terms of  Defn.~\ref{dfn:seli}, enables us to} derive explicit closed-form expressions for the angles between individual classifiers and embeddings, as well as, their norms. {For example, the norm ratio for the classifiers is given by Eqn.~\eqref{eq:eq_intro}.} As an example for angle formulas, we can show for any $R$ and $\Delta$ that:
\begin{align}
    \text{CDT:} \nn \ &\cos({\w_\text{{min}}},{\w_\text{{min}}'})={ \frac{-2 + 2\sqrt{R+1}\left(\sqrt{1 + \Delta^{2}}\right)^{-3}}{k-2+2\sqrt{R+1}\left(\sqrt{1 + \Delta^{2}}\right)^{-3}}},\\
    \text{LDT:} \  &\cos({\w_\text{{min}}},{\w_\text{{min}}'})=\frac{-2\Delta + \sqrt{(R+\Delta^2)/2}}{(k-2)\Delta+\sqrt{(R+\Delta^2)/2}}.\label{eq:cosmin}
\end{align}
See Sec.~\ref{sec:SELI_properties} for the complete list of closed-form formulas, all derived thanks to  Thm.~\ref{thm:SVM-VS}. Such explicit formulas  allow studying optimal tunings and interesting asymptotics as $R$ increases. We show these next. 

\noindent\textbf{Special tunings.}~We emphasize two notable special cases of geometries that arise respectively for LDT and CDT when setting $\delta_c=\sqrt{n_c}\Leftrightarrow \Delta=\sqrt{R}$. 
\begin{corollary}[{Achieving alignment with CDT}]\label{cor:cdt_sqrt_r}
In \eqref{eq:svm_cdt}, set $\Delta=\sqrt{R}$. Then, $\cos(\wopt_{y_i},\hopt_i) = 1, \forall i\in[n]$, i.e., each feature embedding $\hopt_i$ perfectly aligns with its  corresponding classifier $\wopt_{y_i}$. 
\end{corollary}
%
This results from the angle calculations detailed in Sec.~\ref{sec:SELI_properties}. {While this simple tuning leads to perfect alignment of classifiers and mean-embeddings geometries, it does not guarantee equal norms or maximal separation. Thus, the geometry is in general still different from the ETF geometry for balanced data. In contrast, we show next that under the same tuning the implicit geometry of the LDT is an ETF modulo the scaling of the embeddings.}

%
\begin{corollary}[Achieving ETF with LDT]\label{cor:ldt_sqrt_r}In \eqref{eq:svm_ldt}, set $\Delta=\sqrt{R}$. Then, $(\Wopt, \Mopt \diag{(\deltab)})$ follows the ETF geometry.
\end{corollary}
{Cor.~\ref{cor:ldt_sqrt_r} follows immediately from Thm.~\ref{thm:LDT} by noting that $\delta_c=\sqrt{n_c}$ yields $\Tilde{R}=1$ and the $(\ones_k,1)$-SELI geometry coincides with the ETF geometry. This implies that classifiers and embeddings are perfectly aligned, but also all classifiers $\w_c^*, c\in[k]$ have equal norms, and both the classifiers and embeddings are maximally separated, i.e., $\cos(\w^*_c,\w^*_{c'}) = \cos(\h^*_c,\h^*_{c'}) = -{1}/{(k-1)}.$ 
Notably, this holds irrespective of the imbalance ratio $R$. See Fig.~\ref{fig:visual} for the visualization.
}
%
%
%
%
%
%
\newcommand{\subf}[2]{%
  {\small\begin{tabular}[t]{@{}c@{}}
  #1\\#2
  \end{tabular}}%
}
\begin{figure}[t]
\centering
\begin{tabular}{ccc}
\hspace{-0.2cm}\subf{\includegraphics[width=0.3\linewidth]{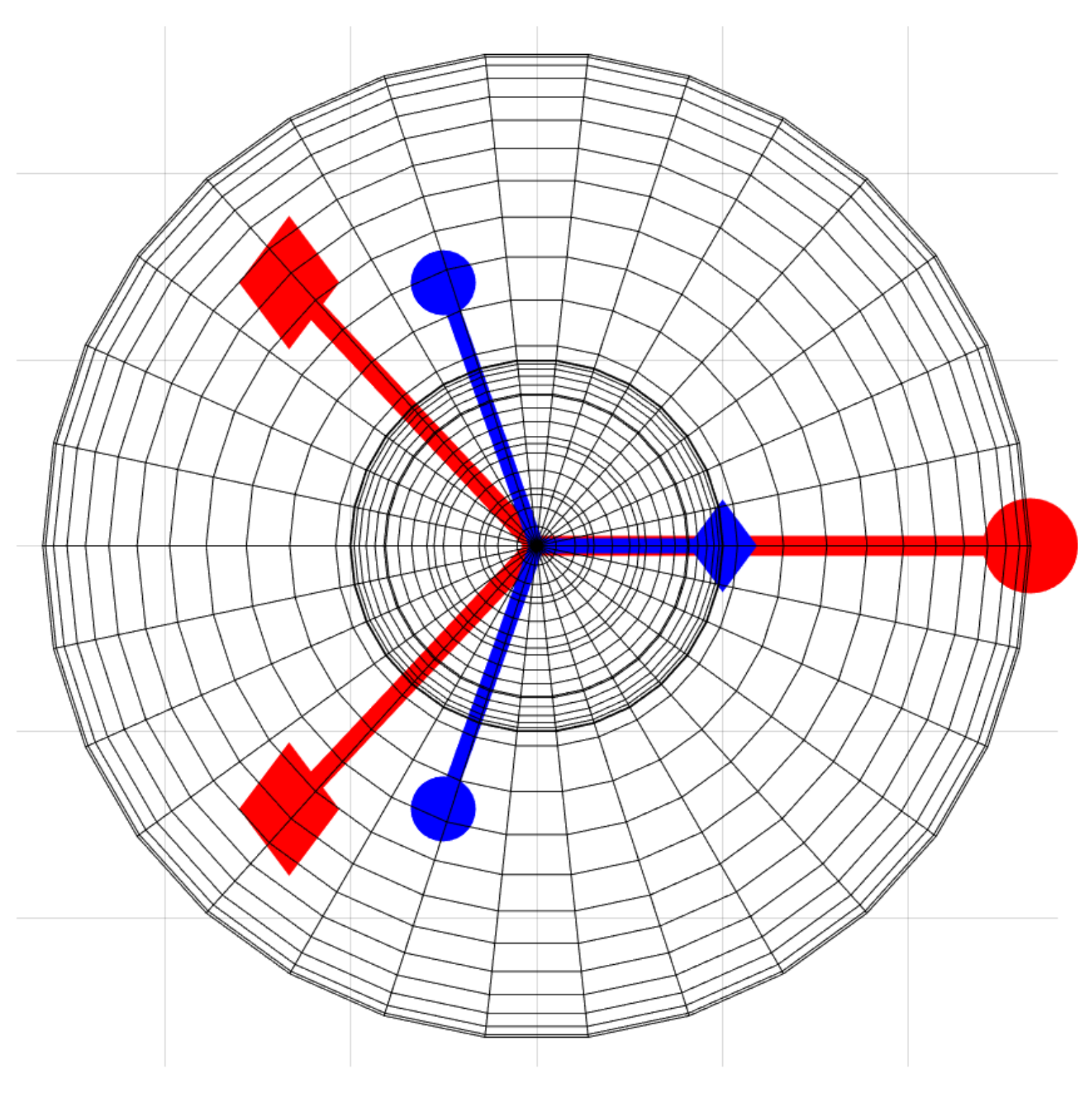}}
{\textbf{CE}}
&
\hspace{-0.25cm}\subf{\includegraphics[width=0.3\linewidth]{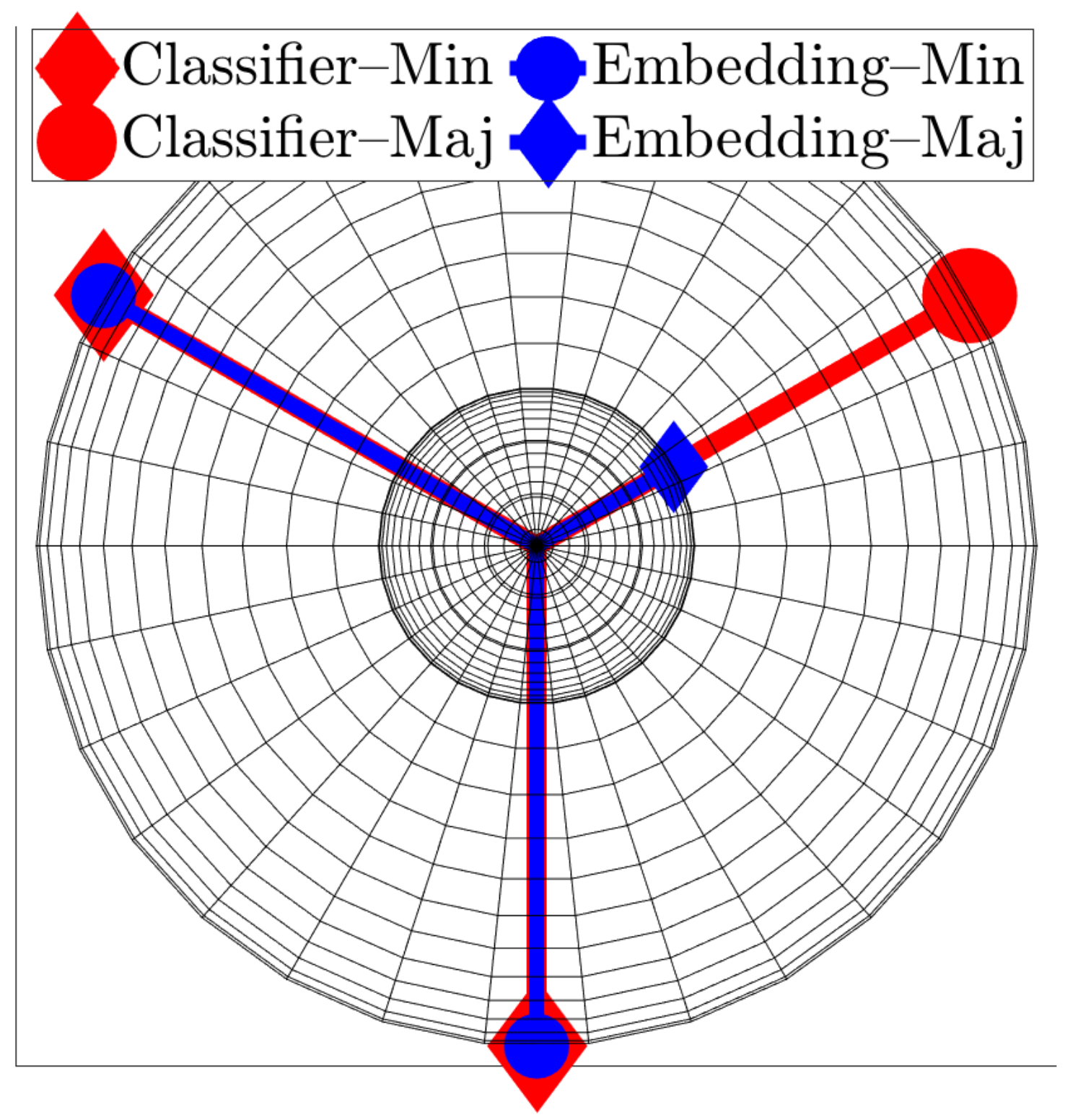}}
{\textbf{LDT}}
&
\hspace{-0.25cm}\subf{\includegraphics[width=0.3\linewidth]{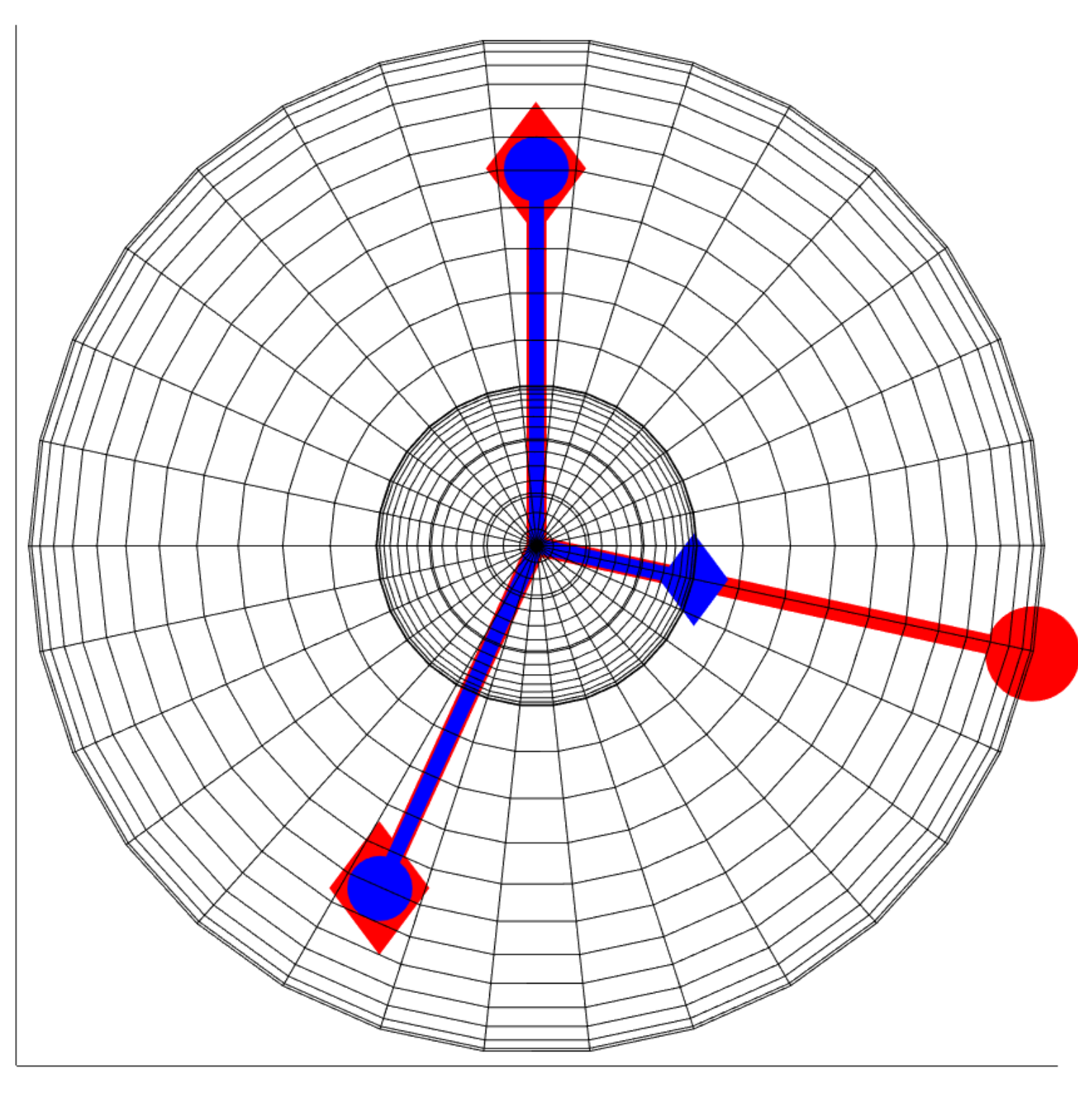}}
{\textbf{CDT}}
\\
\end{tabular}
\vspace{-3pt}
\captionsetup{width=0.95\linewidth}
\caption{{Geometries induced by CE, LDT and CDT for the respective unconstrained features CS-SVM minimizers. $k=3$ with 2 minority and 1 majority classes, imbalance ratio $R=10$ and ratio of hyperparameters $\Delta=\dmaj/\dmin=\sqrt{R}$. See Cors.~\ref{cor:cdt_sqrt_r} and \ref{cor:ldt_sqrt_r}.}}
\label{fig:visual}
\end{figure}


\vspace{3pt}
\noindent\textbf{Mitigating Minority Collapse.}~\cite{fang2021exploring} discovered that when $R\rightarrow\infty$, the minority classifiers collapse, i.e., $\cos(\wmin,\wmin')\rightarrow1$ for any two minority classes. We show here that CDT and LDT losses can mitigate this effect when appropriately tuned. For this, we simply evaluate our closed-form formulas in \eqref{eq:cosmin} in the limit $R\rightarrow\infty$. To obtain non-trivial results, we allow the hyperparameter $\Delta$ to scale with $R$, i.e., set $\Delta=R^\gamma$ for constant $\gamma\in \R.$ This gives the following two results.


\begin{corollary}[Mitigating classifier collapse with LDT]\label{cor:ldt_min_col}
In \eqref{eq:svm_ldt}, set $\Delta={R^{\gamma}}, \gamma \in \R$. Then, as $R\rightarrow\infty$ the minority/majority angles satisfy
\begin{center}
    \begin{tabular}{|c|c|c|c|}
    \hline
    $\cos({\w_c},{\w_c'})$& $\gamma < 1/2$ & $\gamma = 1/2$ & $\gamma > 1/2$ \\
    \hline
    $c, c' \in$ \emph{minority} & $1$ & $-\frac{1}{k-1}$ & $\frac{1-2\sqrt{2}}{1+\sqrt{2}(k-2)}$\\
    \hline
    $c, c' \in$ \emph{majority} & $\frac{1-2\sqrt{2}}{1+\sqrt{2}(k-2)}$ & $-\frac{1}{k-1}$ & $1$ \\
    \hline
\end{tabular}
\end{center}
\end{corollary}
\vspace{1pt}
\begin{corollary}[Mitigating classifier collapse with CDT]\label{cor:cdt_min_col}
In \eqref{eq:svm_cdt}, set $\Delta={R^{\gamma}}, \gamma \in \R$. Then, as $R\rightarrow\infty$ the minority/majority angles satisfy 
\begin{center}
    \begin{tabular}{|c|c|c|c|}
    \hline
    $\cos({\w_c},{\w_c'})$& $\gamma < 1/6$ & $\gamma = 1/6$ & $\gamma > 1/6$ \\
    \hline
    $c, c' \in$ \emph{minority} & $1$ & $0$ & $-\frac{2}{k-2}$\\
    \hline
    \hline
    $\cos({\w_c},{\w_c'})$& $\gamma < 0$ & $\gamma = 0$ & $\gamma > 0$ \\
    \hline
    $c, c' \in$ \emph{majority} & $-\frac{2}{k-2}$ & $\frac{1-2\sqrt{2}}{1+\sqrt{2}(k-2)}$ & $0$\\
    \hline
\end{tabular}
\end{center}
\end{corollary}
From Cor.~\ref{cor:ldt_min_col}, LDT with $\gamma \geq 1/2$ avoids the minority collapse. However, for $\gamma > 1/2$, majority classifiers collapse instead. Thus, we find that $\gamma=1/2$ the only choice that keeps both majority and minority classifiers from collapsing. {In fact, for this choice the angles of majorities and minorities are all equal, as expected by Cor.~\ref{cor:ldt_sqrt_r}.} On the other hand, from Cor.~\ref{cor:cdt_min_col}, CDT avoids minority collapse for any choice of $\gamma \geq 1/6$. Also, in this entire range the majority classifiers do not collapse either. Thus, for $R\rightarrow\infty$, CDT offers a wide tuning range for $\gamma$ that avoids classifier collapse. Compare this to the single value of $\gamma=1/2$ for LDT. This suggests that the CDT geometry is more robust to small changes in the hyperparameter $\gamma$ compared to LDT geometry. 
Specifically for $\gamma=1/2$, when classifiers and features are aligned in both CDT and LDT (see Cors.~\ref{cor:cdt_sqrt_r} and \ref{cor:ldt_sqrt_r}), the CDT minority angles are larger from the LDT angles since ${-2}/{(k-2)}<{-1}/{(k-1)}$; see also Fig.~\ref{fig:visual}. 



%% file: sections/results_ufm.tex
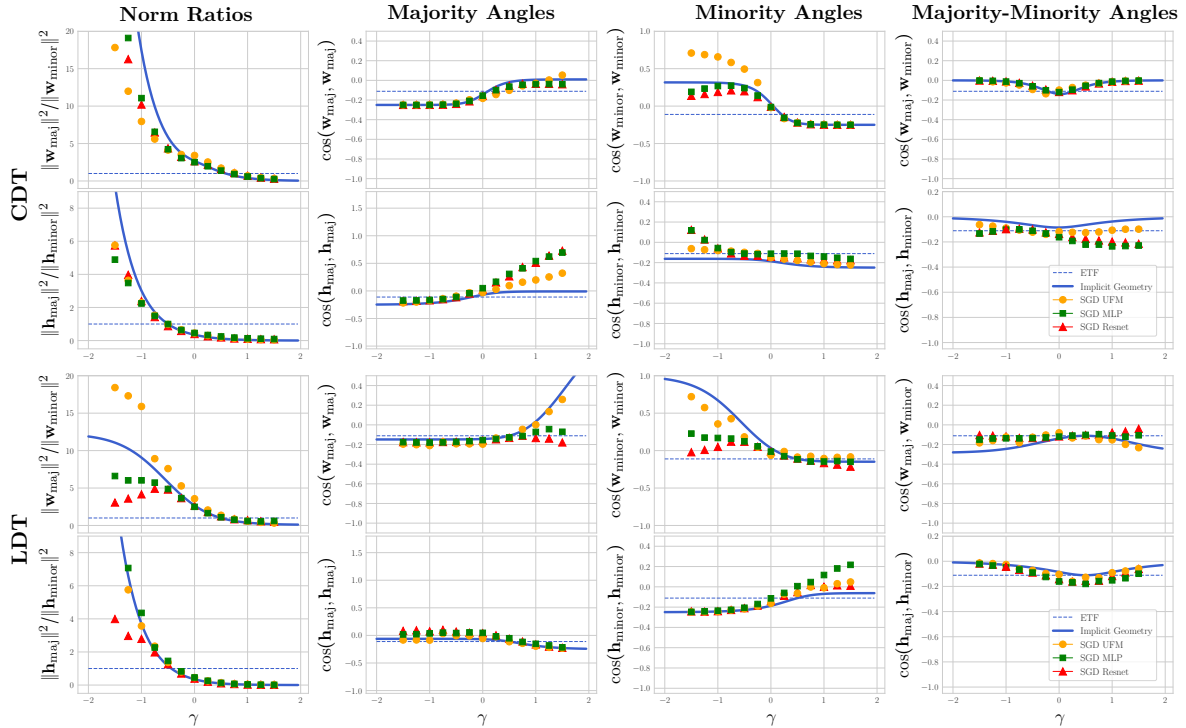
\begin{figure*}[t]
	\begin{subfigure}{0.8\textwidth}
	    \centering
            \begin{tikzpicture}
    			\node at (0,0) 
    			{\includegraphics[scale=0.183]{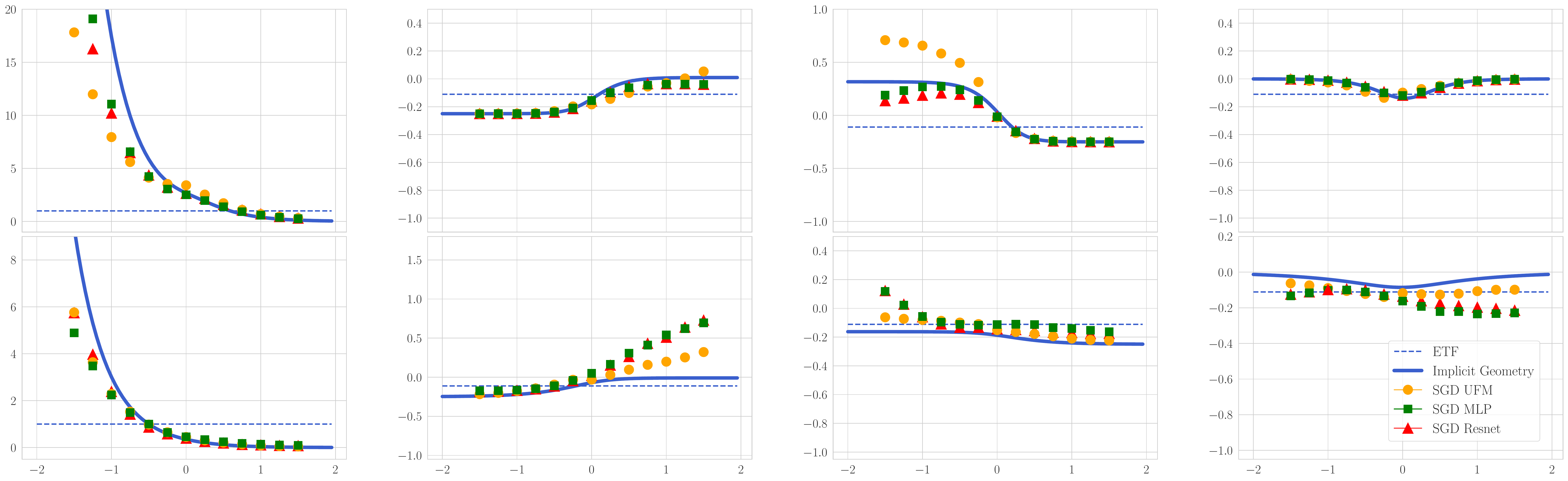}};
    			\node at (-5.7,2.4) [scale=0.7]{\textbf{Norm Ratios}};
    			\node at (-1.95,2.4) [scale=0.7]{\textbf{Majority Angles}};
    			\node at (2.1,2.4) [scale=0.7]{\textbf{Minority Angles}};
    			\node at (5.6,2.4) [scale=0.7]{\textbf{Majority-Minority Angles}};
    			
    			\node at (-7.9,0.0)  [scale=0.75, rotate=90]{\textbf{CDT}};
    			
    			\node at (-7.5,1.25)  [scale=0.6, rotate=90]{$\norm{\wmaj}^2 / \norm{\wmin}^2$};
    			\node at (-7.5,-1.0)  [scale=0.6, rotate=90]{$\norm{\hmaj}^2 / \norm{\hmin}^2$};
    			
    			\node at (-3.85,1.25)  [scale=0.6, rotate=90]{$\cos(\wmaj,\wmaj)$};
    			\node at (-3.85,-1.0)  [scale=0.6, rotate=90]{$\cos(\hmaj,\hmaj)$};
    			
    			\node at (-0.05,1.25)  [scale=0.6, rotate=90]{$\cos(\wmin,\wmin)$};
    			\node at (-0.05,-1.0)  [scale=0.6, rotate=90]{$\cos(\hmin,\hmin)$};
    			
    			\node at (3.75,1.25)  [scale=0.6, rotate=90]{$\cos(\wmaj,\wmin)$};
    			\node at (3.75,-1.0)  [scale=0.6, rotate=90]{$\cos(\hmaj,\hmin)$};
    			
    			
    		\end{tikzpicture}
    	\vspace{-20pt}
    	
	    \centering
		\begin{tikzpicture}
			\node at (0,0) 
			{\includegraphics[scale=0.183]{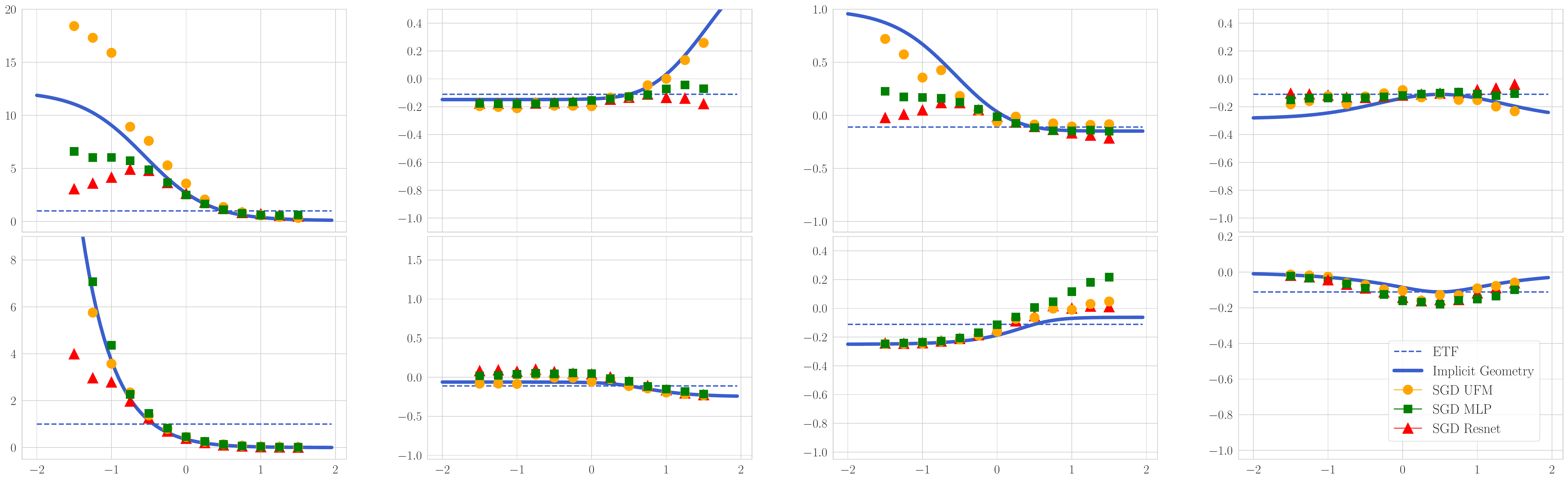}};
				
    			\node at (-7.9,0.0)  [scale=0.75, rotate=90]{\textbf{LDT}};
    			
    			\node at (-7.5,1.25)  [scale=0.6, rotate=90]{$\norm{\wmaj}^2 / \norm{\wmin}^2$};
    			\node at (-7.5,-1.0)  [scale=0.6, rotate=90]{$\norm{\hmaj}^2 / \norm{\hmin}^2$};
    			
    			\node at (-3.85,1.25)  [scale=0.6, rotate=90]{$\cos(\wmaj,\wmaj)$};
    			\node at (-3.85,-1.0)  [scale=0.6, rotate=90]{$\cos(\hmaj,\hmaj)$};
    			
    			\node at (-0.05,1.25)  [scale=0.6, rotate=90]{$\cos(\wmin,\wmin)$};
    			\node at (-0.05,-1.0)  [scale=0.6, rotate=90]{$\cos(\hmin,\hmin)$};
    			
    			\node at (3.75,1.25)  [scale=0.6, rotate=90]{$\cos(\wmaj,\wmin)$};
    			\node at (3.75,-1.0)  [scale=0.6, rotate=90]{$\cos(\hmaj,\hmin)$};
    			
			\node at (-5.6,-2.4) [scale=0.6]{$\gamma$};
			\node at (-1.8,-2.4) [scale=0.6]{$\gamma$};
			\node at (2.0,-2.4) [scale=0.6]{$\gamma$};
			\node at (5.8,-2.4) [scale=0.6]{$\gamma$};
			
		\end{tikzpicture}
	\end{subfigure}

	\vspace{-5pt}
    \captionsetup{width=0.95\linewidth}
	 \caption{Comparison of models trained by SGD (markers) minimizing the CDT \eqref{eq:CDT_loss}/LDT \eqref{eq:LDT_loss} loss and the global minimizers of the CS-SVM in \eqref{eq:svm_cdt}, \eqref{eq:svm_ldt} as given by Thm.~\ref{thm:SVM-VS} (solid line) in a ($10$, $1/2$)-STEP imbalanced setting. The dashed line marks the perfectly symmetric ETF geometry of balanced data \citep{NC}. See Sec.~\ref{sec:app_exp_details} for more details.}
	\label{fig:convergence_to_theory}
\end{figure*}

For both CDT and LDT loss, we examine the convergence of the models trained by SGD to the implicit geometry proposed by Thm.~\ref{thm:SVM-VS}. We train (i) UFM, (ii) MLP on MNIST, and (iii) ResNet18 on CIFAR10. All the models are trained in a $(R=10,\rho=1/2)$-STEP imbalanced setting. We further use STEP logit adjustment, and choose $\Delta=R^\gamma$ with $\gamma\in[-1.5,1.5]$. We train the UFM by minimizing unregularized CDT/LDT, while for MLP and ResNet models, following the setup in \cite{NC}, we use a small weight-decay ($10^{-5}$). We defer other experimental details to Sec.~\ref{sec:app_exp_details}.

{Fig.~\ref{fig:convergence_to_theory} illustrates the empirical geometry discovered by SGD vs the prediction of Thm.~\ref{thm:SVM-VS}. 
For the trained classifiers and embeddings, we compute: (1) squared ratios of majority-minority norms,
 (2) cosine of angles between pairs of majority-majority, minority-minority, majority-minority for classifiers and mean-embeddings.} For each choice of $\gamma$ and loss function, we compute each metric on all the respective pairs, and compare their average to the closed-form expressions that result from Thm.~\ref{thm:SVM-VS} (see Sec.~\ref{sec:SELI_properties}). 

As reported in the figures, the empirical quantities follow the predicted theoretical trends. However, convergence becomes more challenging for the deep-net models, particularly for larger $|\gamma|$. Moreover, we encounter cases with non-zero training error for CDT loss for large $|\gamma|$ values. In addition to $\gamma$, the imbalance ratio $R$ also affects the convergence to theory (see Sec.~\ref{sec:OptIm_Discuss} for details). Further, the theory gives a more accurate prediction of the mean-embeddings' geometry in case of the LDT, and of the classifiers' in case of the CDT loss. This is consistent for both UFM and deep-net models. For LDT, the prediction is well followed by UFM and ResNet empirics around the interesting value of $\gamma=0.5$, with an exception of the majority classifier angles in the ResNet experiments. The mismatch is less severe for the 6-layer MLP. Also, as predicted by the theorem, for  $\gamma = 0.5$ ($\Delta =\sqrt{R}$), the LDT geometry
is the ETF, up to a scaling on the features: In Fig. \ref{fig:convergence_to_theory} the LDT cosine plots intersect with the ETF angles, i.e., $-1/(k-1)$, thus achieving equiangularity and maximal angular separation. The classifier norm ratios also attain the value $1$, which along with the equiangularity describe an ETF structure for classifiers. 
%

While the experiments in Fig.~\ref{fig:convergence_to_theory} correspond to a finite imbalance ratio of $R=10$, there is resemblance to the asymptotic behavior of the classifier angles on LDT-trained UFM. Cor.~\ref{cor:ldt_min_col} suggests $\gamma=0.5$ is the only choice for $R\rightarrow \infty$ that avoids minority or majority classifiers collapsing. A similar trend is seen in Fig.~\ref{fig:convergence_to_theory}, where the cosine of the minority classifiers goes towards $1$ for $\gamma<0.5$, while that of the majority classifiers approaches $1$ for the complementary open interval of $\gamma>0.5$.  {On the other hand, CDT does not attain equiangularity, but  majority and minority angles are well controlled for a wider range of $\gamma$. This suggests that the CDT geometry is more robust to small changes in the hyperparameter $\gamma$ compared to LDT geometry.}

\begin{remark} \label{rem:center}
In all our experiments with CDT and LDT, we center the embeddings before computing norms and angles. This is consistent with centering performed for experiments with balanced data in \cite{NC,zhu2021geometric,seli}. In our case, the exact centering vector is different for each loss function. {Additionally, we have found that centering improves convergence not just in deep-net experiments as in previous works, but also in  UFM experiments. See Sec.~\ref{sec:CDT_center}/\ref{sec:LDT_center} for details on both CDT/LDT losses.}
\end{remark}

%% file: sections/test_results.tex

Up to this point, we have demonstrated that various CE parameterizations lead to distinct implicit geometries for classifiers and embeddings during the training process. In this section, we explore the degree to which these implicit geometries influence performance during testing, or in other words, generalization. Specifically, we provide preliminary results on the generalization of models trained with CDT/LDT losses:  In Sec.~\ref{sec:gmm}, we propose and investigate a simple model that aims to capture the link between generalization and implicit geometry. In Sec.~\ref{sec:test_numerical}, we present preliminary empirical results from experiments on real data, which we compare to our model's predictions.  Finally, in Sec.~\ref{sec:post-hoc},  we demonstrate that our analysis can offer valuable guidance for developing enhanced algorithms by utilizing the implicit geometry for a post-hoc modification of LDT, resulting in improved generalization.

\subsection{Impact of Geometry}\label{sec:gmm}

In order to assess generalization, it is necessary to define the geometry of test embeddings rather than just those from training. This is typically challenging for neural networks in general. Here, we simplify the scenario by assuming that the mean-embeddings during testing are similar to their training counterparts in a way that we formalize below. Using this model, we aim to acquire broader insights into the impact of various implicit geometries on test performance.


We are interested in the balanced test error that weighs all classes equally, unlike the standard error that relies on class priors. This is a standard evaluation metric used in data-imbalanced training regimes in previous works \citep[e.g.,][]{CDT, TengyuMa, Menon, VS,li2021autobalance}. To evaluate the balanced test error we assume that the embeddings are concentrated around their class-means with some small variance. 
Specifically, suppose $(\W,\M)$ are the classifiers and mean-embeddings induced by CDT/LDT loss at the end of training. We assume that the emeddings $\h\in\R^d$ belonging to class $c\in[k]$ follow an isotropic Gaussian distribution with mean $\mub_c$ and variance $\sigma_c^2$, i.e., $\h|(y=c) \sim \mathcal{N}(\mub_c, \sigma_c^2\,\Id_d)$. With this assumption, the balanced error rate can be found as follows,
\begin{align}\label{eq:error_rate}
    \mathcal{R}_{\text{bal}} = \frac{1}{k} \sum_{y\in[k]}\, \mathbb{P}_{\h\sim \mathcal{N}(\mub_y,\sigma_y^2\,\Id_d)} \big(\max_{c\neq y} \,(\w_c - \w_y)^T\h\geq 0 \big). 
\end{align}
In this model, we assume that the degree of within-class variation $\sigma_c^2$ depends only on the class size at the training stage. Specifically, we assume $\sigma_c^2 \propto 1/n_c ^ \alpha$ for some $\alpha>0$.\footnote{{We have empirically verified the approximate log-linear dependence of $\sigma_c^2$ on $\text{log}(R)$ on models trained with CE under different imbalance ratios on CIFAR10.}} Equivalently, under the STEP imbalance assumption, we model the variations as follows,
\begin{align*}
    \sigma_c^2 \propto \begin{cases}
    1, \quad & \text{if $c$ is a majority class}\\
    R^\alpha, \quad & \text{if $c$ is a minority class}
    \end{cases},
\end{align*}
where $\alpha$ models the impact of the imbalance ratio on the embeddings from minority classes. {In other words, the embeddings from majority classes are more concentrated around their means, while the embeddings from minority classes spread more as the imbalance ratio $R$ increases.}
To only capture the role of the geometry, we keep the total SNR of the model fixed by scaling the mean-emebddings to ensure $$\frac{\norm{\mub_{\text{maj}}}_2^2}{\sigma^2_\text{maj}} + \frac{\norm{\mub_{\text{minor}}}_2^2}{\sigma^2_\text{minor}} = \frac{\norm{\mub_{\text{maj}}}^2_2}{1} + \frac{\norm{\mub_{\text{minor}}}^2_2}{ R^\alpha} = \text{constant},$$
{across different geometries}. To compare the optimality of $(\deltab,R)$-SELI geometries, we compute the error rate \eqref{eq:error_rate} by Monte-Carlo simulations. {In Fig.~\ref{fig:GMM}, for $R=10,\,k=10$ and $\alpha=1$, we illustrate the error rate for the geometries induced by CDT (left) and LDT (right) loss for different values of $\Delta=R^\gamma$. We can obtain analogous results for other values of $R$, $k$ and $\alpha$.} Varying $\Delta$ introduces a trade-off between the error on the majority and minority classes. Specifically, as $\Delta$ increases (i.e., we assign relatively larger $\delta_c$ to majority classes), the model classifies the minority classes more accurately (despite their larger within-class variance). On the other hand, the error on majorities raises significantly. {The geometry with the lowest error is achieved by  $\gamma \in [0,1]$. We note that the optimal value of $\gamma$ varies for different choices of $\alpha$.}

\begin{figure}[t]
\centering
\begin{subfigure}{.46\textwidth}
  \centering
  \begin{tikzpicture}
	\node at (0,0.0)
    {\includegraphics[width=\linewidth]{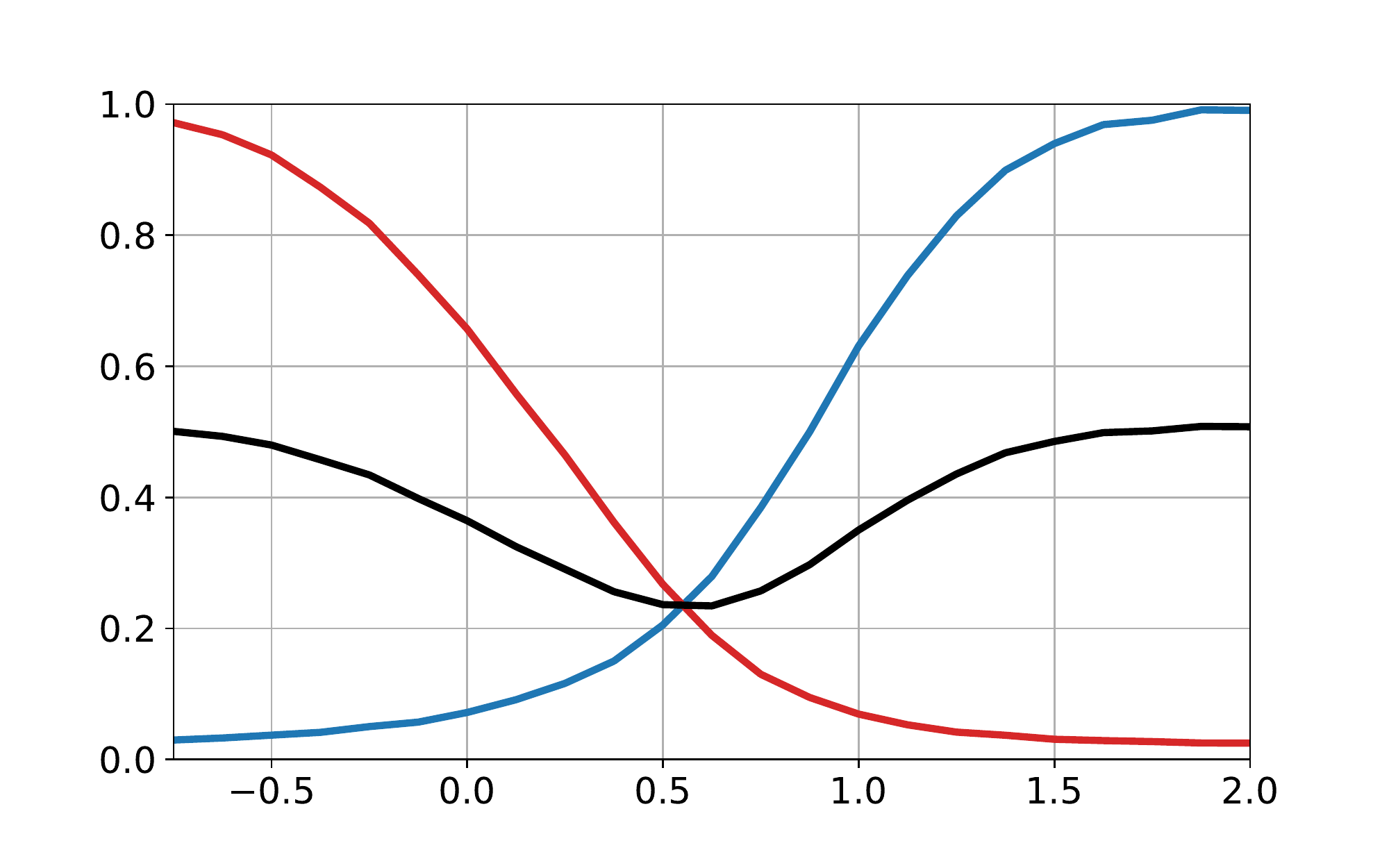}};
	\node at (0.3,-2.3) [scale=0.89]{$\gamma$};
        \node at (0.3,2.0) [scale=0.89]{CDT Loss};
        \node at (-3.5,-0.0)  [scale=0.75, rotate=90]{Error};
  \end{tikzpicture}
\end{subfigure}%
\begin{subfigure}{.46\textwidth}
  \centering
  \begin{tikzpicture}
	\node at (0,0.0)
{\includegraphics[width=\linewidth]{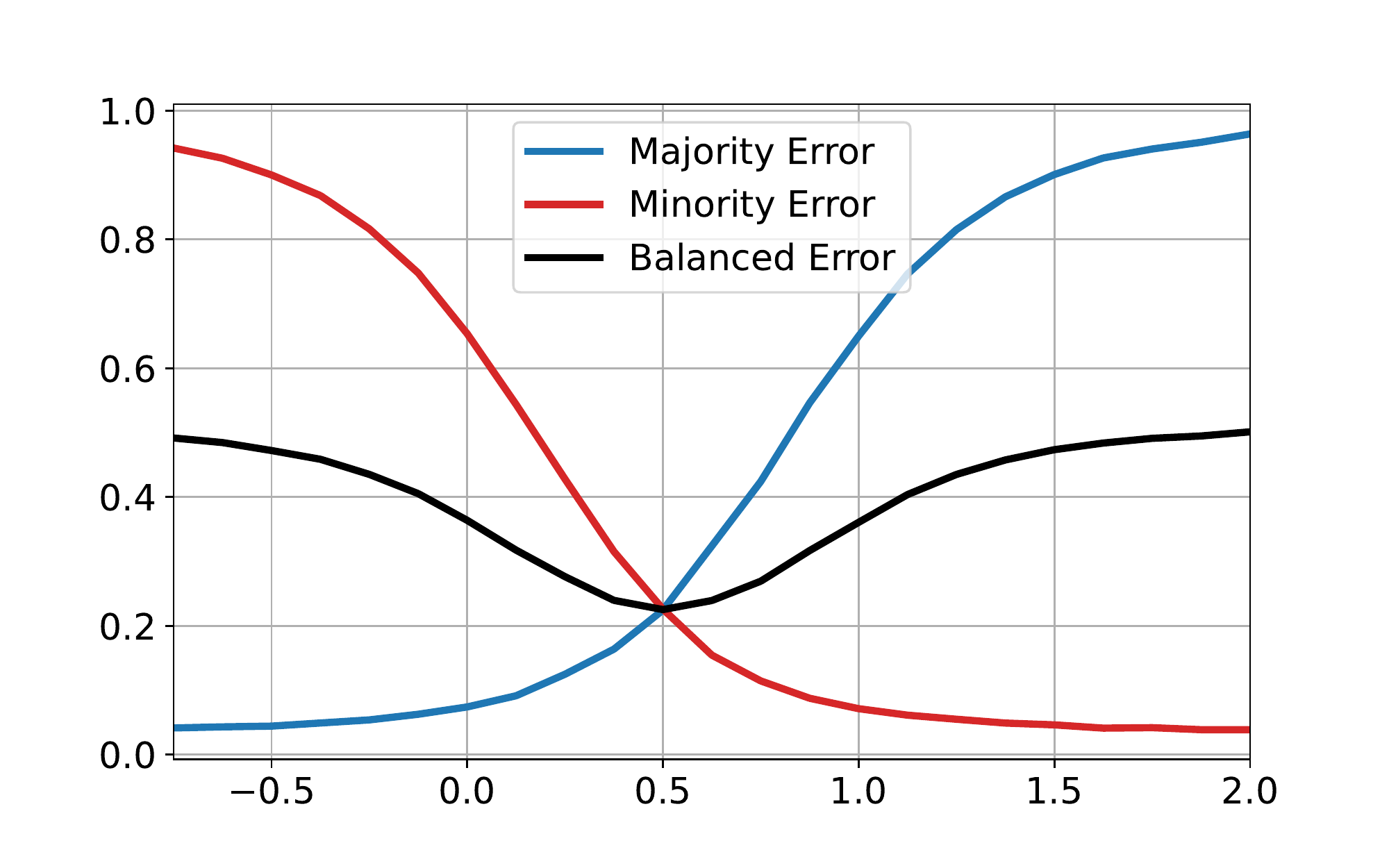}};
	\node at (0.3,-2.3) [scale=0.89]{$\gamma$};
        \node at (0.3,2.0) [scale=0.89]{LDT Loss};
  \end{tikzpicture}
\end{subfigure}%
\captionsetup{width=0.95\linewidth}
\vspace{-0.0cm}
\caption{Classification error of the $(\deltab,10)$-SELI geometry under the Gaussian mixture distribution on the embeddings, described in Sec.~\ref{sec:gmm}. $\gamma$ controls the implicit geometry by  $\Delta=\delta_{\text{maj}}/\delta_\text{minor}=R^\gamma$.}
\label{fig:GMM}
\end{figure}
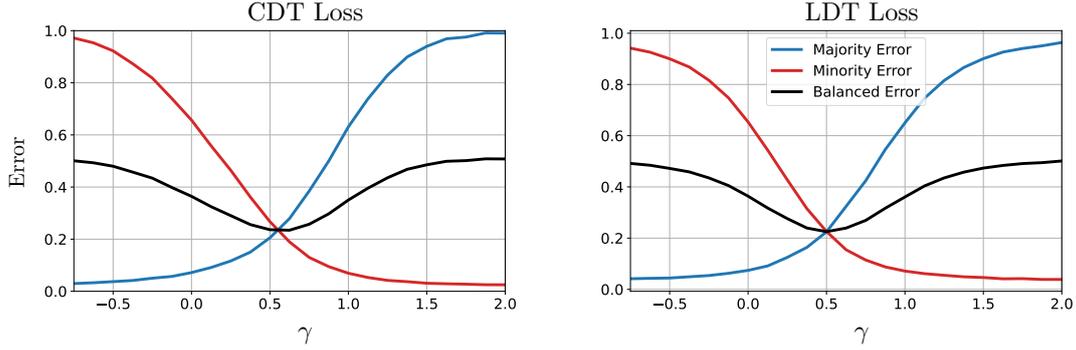

\subsection{Numerical Results on Generalization}\label{sec:test_numerical}
In this section, we present preliminary empirical observations on the balanced test error $\mathcal{R}_{\text{bal}}$ achieved by minimizing CDT/LDT loss functions. {We compare the test accuracy of CDT, LDT and wCE loss by evaluating the performance of ResNet18 trained on CIFAR10 and of MLP trained on MNIST and Fashion-MNIST for each loss function.} In order to have results comparable to  state-of-the-art, we perform data augmentation as in \cite{TengyuMa,CDT} on all three datasets.
 For all three losses, we control their hyperparameters by a single variable $\gamma$: we choose ${{\omega_\text{minor}}}/{\omega_\text{maj}} := R^\gamma$ for the weights of wCE  and ${\delta_{\text{maj}}}/{\delta_{\text{minor}}} = R^\gamma$ for CDT/LDT loss. {We also normalize the $\delta$ values similar to Sec.~\ref{sec:num_results}} Notice that $\gamma = 0$ represents the CE loss for any choice of loss function. In Fig.~\ref{fig:test_acc_resnet_perclass}, we present the average performance across 10 independent runs for each value of $\gamma$ and for each loss function. 

CDT loss has the best balanced accuracy (averaged across different iterations) on the test set for $\gamma \in \left[0.5,1.0\right]$, with the highest value being $90.35\%$ for $\gamma = 0.75$. On the other hand, LDT loss does not exhibit major improvements compared to CDT. The highest test accuracy for LDT is achieved usually around $\gamma = 0.5$ which according to Cor.~\ref{cor:ldt_sqrt_r}, is the same value that leads to the ETF geometry up to scaling of embedding norms. wCE has the lowest accuracy among the losses with the best test performance being comparable to CE ($\gamma = 0$) across different experiments. 

For CDT, we observe that larger values of $\gamma$ lead to better test performance for minority classes, but worse performance for majorities which is consistent with our analysis in Sec.~\ref{sec:gmm}. However, for LDT, the test accuracy for majority classes does not drop as much with larger $\gamma$ values and the accuracy for minority classes peaks at $\gamma \in [0.5,0.75]$ and drops at either side. This is in contrast to our analysis from Sec.~\ref{sec:gmm} where we expect LDT to have a similar trend to that of CDT's. In addition, Fig.~\ref{fig:test_acc_resnet_perclass} suggests that LDT's test accuracy is less sensitive to $\gamma$ when compared to CDT overall. However, the variation in the range $\gamma \geq 0.5$ is consistent with our analytical observation that LDT's behavior is more sensitive than CDT for $\gamma \geq 0.5$. 
These observations motivate further investigation into the geometry of embeddings during the evaluation phase.  

\begin{figure*}[t]
\centering
\begin{subfigure}{.95\textwidth}
  \centering
  \begin{tikzpicture}
	\node at (0,0.0)
    {\includegraphics[width=0.95\linewidth]{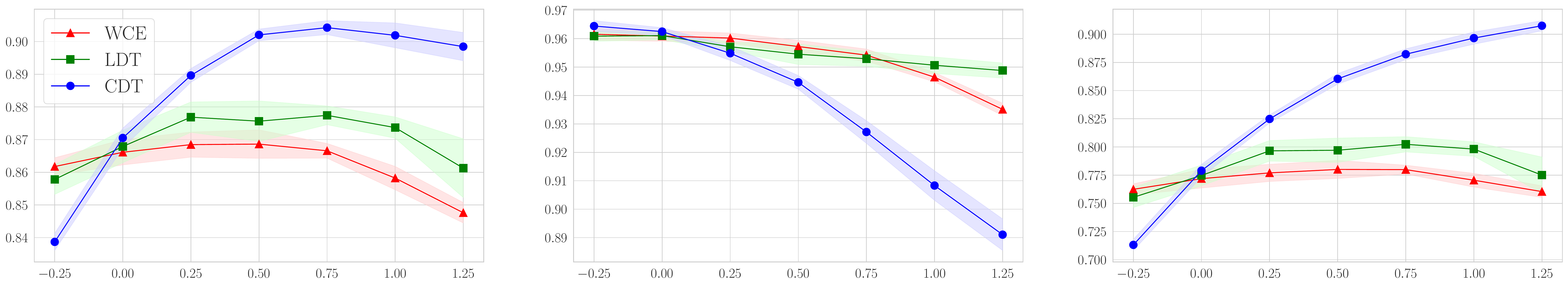}};
	\node at (-4.8,1.5)  [scale=1.0]{\textbf{Balanced Accuracy}};
	\node at (0.2,1.5)  [scale=1.0]{\textbf{Majority Accuracy}};
	\node at (5.0,1.5)  [scale=1.0]{\textbf{Minority Accuracy}};
    \node at (-7.4,-0.0)  [scale=0.8, rotate=90]{\textbf{CIFAR10}};
    \node at (-7.8,-0.0)  [scale=0.8, rotate=90]{\textbf{ResNet }};
  \end{tikzpicture}
  \begin{tikzpicture}
  \node at (0,0.0)
    {\includegraphics[width=0.95\linewidth]{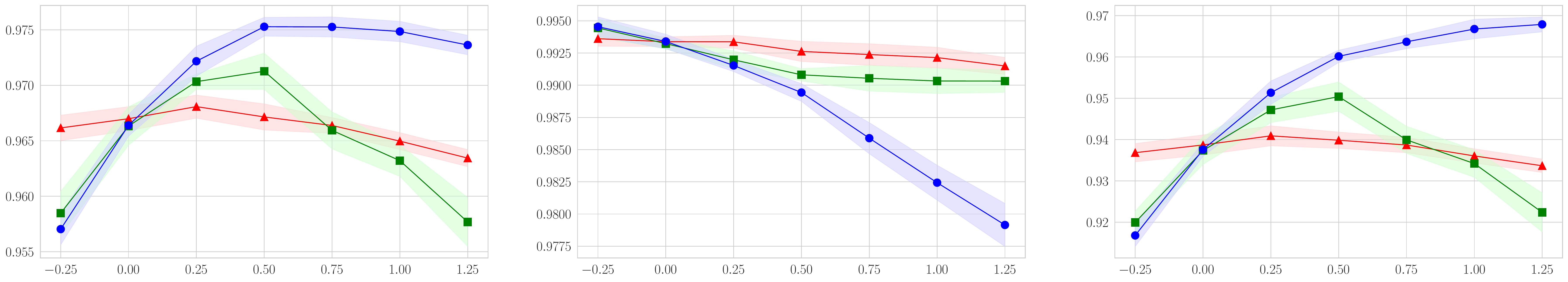}};
    \node at (-7.4,-0.0)  [scale=0.8, rotate=90]{\textbf{MNIST}};
    \node at (-7.8,-0.0)  [scale=0.8, rotate=90]{\textbf{MLP }};
  \end{tikzpicture}
  \begin{tikzpicture}
  \node at (0,0.0)
    {\includegraphics[width=0.95\linewidth]{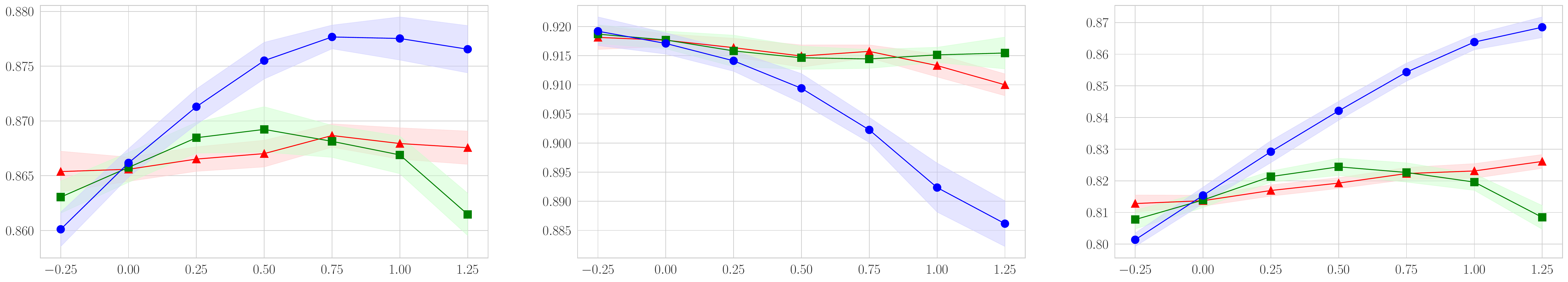}};
    \node at (-7.4,-0.0)  [scale=0.8, rotate=90]{\textbf{Fashion-MNIST}};
    \node at (-7.8,-0.0)  [scale=0.8, rotate=90]{\textbf{MLP }};
    \node at (-0.0, -1.5) [scale=0.9] {$\gamma$}; 
  \end{tikzpicture}
\end{subfigure}%
\vspace{-0.25cm}
\captionsetup{width=0.95\linewidth}
\caption{Test accuracy for ResNet18 and MLP trained on ($R=10$, $\rho$=1/2)-STEP imbalanced CIFAR10, MNIST and Fashion-MNIST using CDT, LDT and wCE losses for different hyperparameter values. We report the average accuracy and its standard deviation over 10 independent runs. Best test accuracies for CDT and LDT are generally achieved at $\gamma = 0.75$ and $\gamma = 0.5$ respectively with wCE showing no major improvement compared to the other two losses. }
\label{fig:test_acc_resnet_perclass}
\end{figure*}
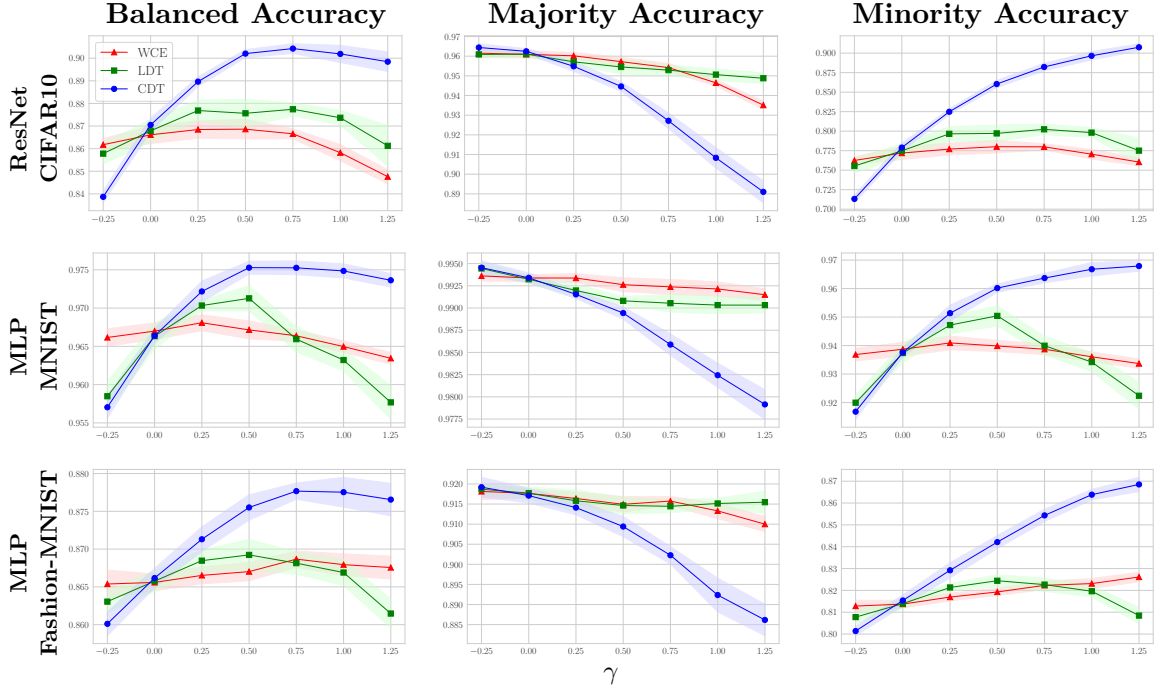

\subsection{Post-hoc Rescaled LDT (R-LDT)}\label{sec:post-hoc}
Through the experiments in Sec. \ref{sec:test_numerical}, we observe that the LDT loss has inferior test performance compared to CDT. In the following section, we show that the knowledge of implicit geometry can be leveraged to design a simple post-hoc ``rescaling'' scheme that boosts the balanced accuracy of LDT trained models.

To motivate the idea, consider the case of $\gamma = 0.5$ to motivate the rescaling scheme. Recall from Cor. \ref{cor:ldt_sqrt_r} that after training with LDT ($\gamma=0.5$), the geometry of embeddings and classifiers form an ETF, up to a scaling factor $\sqrt{R}$ on the majority embeddings. Specifically, the minority mean embeddings are larger in norm than the majority mean embeddings by a factor of $\sqrt{R}$. However, the classifier vectors attain equal norms. To boost minority performance, one can scale down the majority classifier norms by the factor $\sqrt{R}$. Following this idea, our post-hoc algorithm scales the trained majority classifier vectors by $\Delta^\beta = R^{0.5\beta}$, for a tunable $\beta$. We vary $\beta$ in $\left[-1,1\right]$ in steps of $0.25$. The schematic in Fig. \ref{fig:posthoc_ldt_fig} demonstrates the geometric effect of post-hoc rescaling in LDT for $\beta\in \{-1,-0.5,0\}$. Alg. \ref{alg:cap} formally describes the post-hoc rescaling scheme, for the general case where LDT is parameterized by a hyperparameter $\gamma$.

Through experiments, we demonstrate that this technique improves the test performance of LDT-trained models. We perform the experiments for the same setting as in Sec. \ref{sec:num_results} for CIFAR10 with a ResNet18 model.
\begin{figure*}[t]
\centering
\begin{subfigure}{.4\textwidth}
  \centering
    \includegraphics[width=0.95\linewidth]{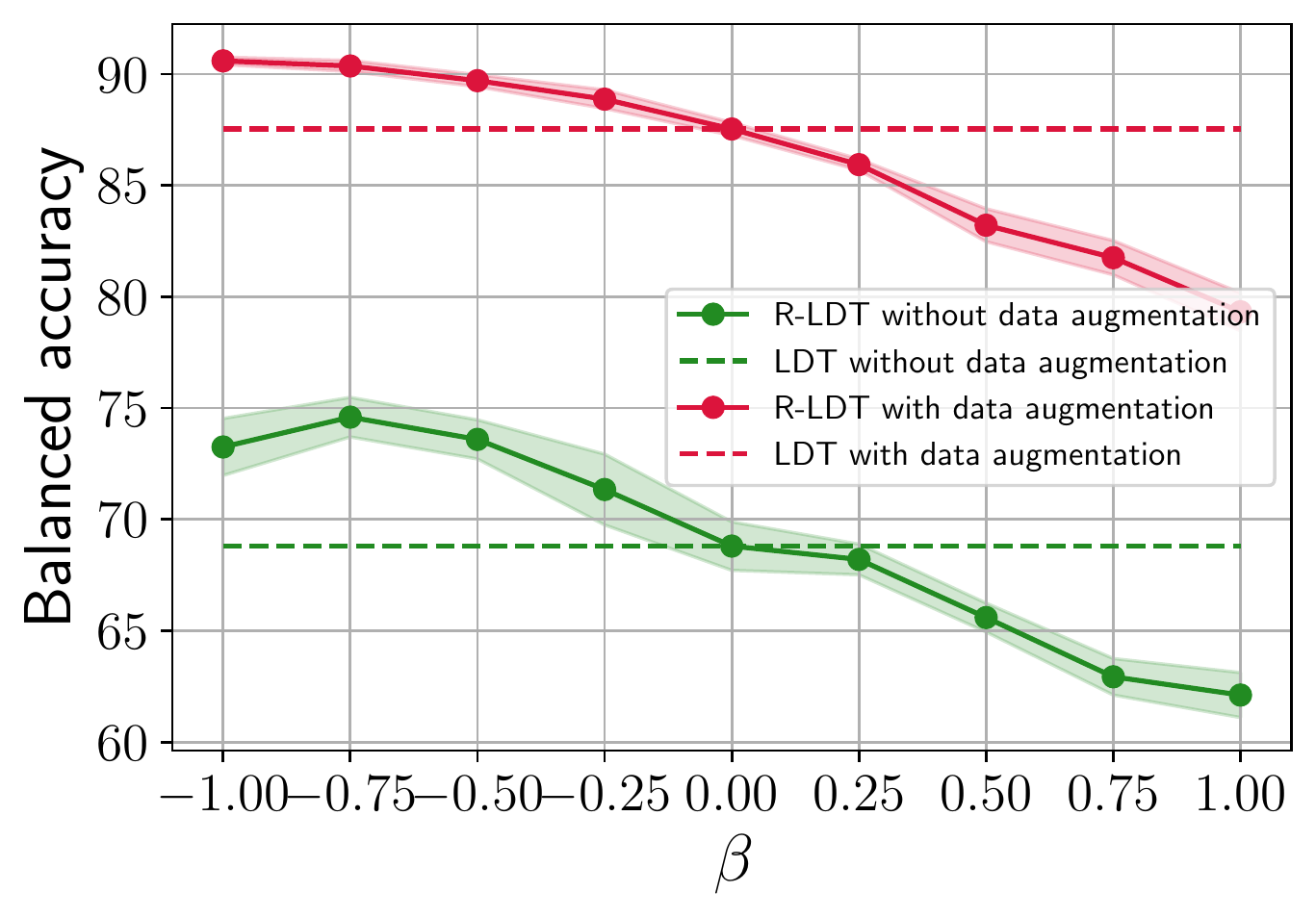}
    \caption{}
    \label{fig:posthoc_ldt_test}
\end{subfigure}%
\begin{subfigure}{.6\textwidth}
    \hspace{0.4cm}\includegraphics[width=0.9\linewidth]{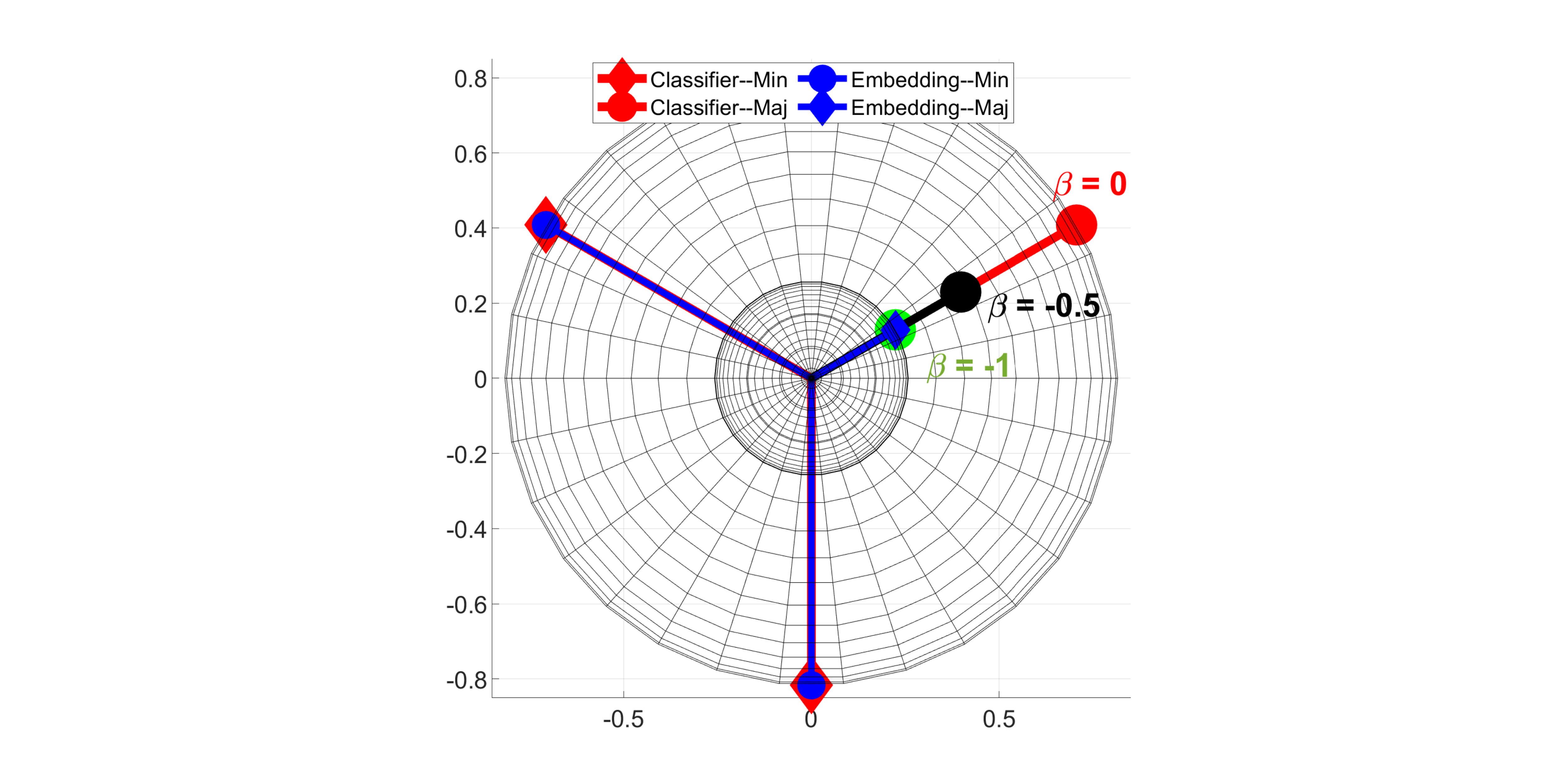}
\caption{}
\label{fig:posthoc_ldt_fig}
\end{subfigure}%
\caption{(a) Balanced test accuracies of post-hoc rescaled LDT for $\gamma=0.5$; (b) The geometric effect of post-hoc rescaling in LDT for $\beta\in \{-1,-0.5,0\}$, with $\gamma = 0.5$. {\color{teal}$\beta=-1$} shrinks the majority classifier vector by $\sqrt{R}$, $\beta=-0.5$ by $R^{1/4}$, while {\color{red}$\beta=0$} retains the LDT geometry.}
\end{figure*}
 Fig. \ref{fig:posthoc_ldt_test} shows the balanced accuracy of this scheme for $\gamma = 0.5$, with and without augmentation on training data. The values indicate the balanced accuracy over classes of CIFAR10 test set, averaged over 5 instances of the experiments. 
 Note that rescaling improves the test performance of LDT trained models. Specifically, our experiments reveal a performance gain of $6\%$ when training without data augmentation, and $3\%$ with data augmentation. The optimal performance is achieved for a value of $\beta=-1$ when using augmentation, and $\beta=-0.75$ when not using augmentation. This fact is consistent with our motivation for scaling down the majority classifiers by $\sqrt{R}$. The above experimental observations give guidance on the geometries that favor generalization, while also suggesting that LDT can be sub-optimal on its own. It is therefore of interest for future studies to design loss functions that attain implicit geometries with better generalization performance without a need for post-hoc rescaling.
 \begin{algorithm}
  \caption{Post-hoc Rescaled LDT (R-LDT)}\label{alg:cap}
  \begin{algorithmic}[1]
    \Procedure{R-LDT}{Imbalance ratio $R$, hyperparameters $\gamma,\beta$}
\State \textbf{Phase}-I
\State Train with LDT loss with $\Delta=R^\gamma$. 
\State Let $\w_1,...,\w_k$ be the resulting classifier vectors
\State \textbf{Phase}-II 
      \For{$c$ \texttt{a majority class}}
        \State $\w_c \gets \w_c \times R^{\beta\gamma}$
      \Comment{Scale down majority classifier}
      \EndFor
      \State \textbf{return} $\w_c, \forall c \in [k]$
    \EndProcedure
  \end{algorithmic}
\end{algorithm}

%% file: sections/more_related_work.tex
Our paper is motivated by and contributes to two recent thrusts in the literature. The first seeks structural properties of the models learned by deep neural networks trained far beyond the zero-train error regime \citep[e.g.,][]{NC,fang2021exploring,galanti2021role,graf2021dissecting,han2021neural,hui2022limitations,ULPM,lu2020neural,mixon2020neural,tirer2022extended,xie2022neural,zhu2021geometric,zhou2022optimization,seli}. The second one investigates approaches to coping with class imbalances in overparameterized model training \citep[e.g.,][]{byrd2019effect,sagawa2019distributionally,sagawa2020investigation,TengyuMa,kang2020decoupling,KimKim,Menon,CDT,VS,wang2021importance,jitkrittum2022elm}. We already discussed some of the most closely related works within each thrust (as well as a few recent works \citep{fang2021exploring,xie2022neural,yang2022we} at the thrusts' intersection) in the introduction (see paragraph on \emph{Related Work}). 
The goal of this section is to outline main take-aways of our work in the form of both contributions and limitations, together with some pointer for future directions. 


\noindent\textbf{Contributions.}~We extend the scope of the geometry characterizations of the embeddings and classifiers learned by deep-nets  initiated by \cite{NC}. To the best of our knowledge, all prior works study the geometries for either the CE or mean-square loss. Instead, we formulate a more general geometry that describes two alternative CE parameterizations and includes the previous geometries as special cases. Unlike previous works, our new geometry is parameterized in terms of the loss hyperparameters, thus it involves rich structures (in terms of angles and norm-ratios) as these hyperparameters vary. Yet, like in previous works, the geometry is rather simple to describe, either implicitly in terms of a special encoding matrix or explicitly in terms of closed-form formulas for the angles and norms. We arrive at this new geometry by analyzing the simplified unconstrained features-model (specifically, its cost-sensitive version in Eqns. \eqref{eq:svm_cdt},\eqref{eq:svm_ldt}). Thus, we also extend the scope of the UFM model beyond the previously studied CE and square loss. Finally, we undertake an  implicit-geometry view  to  loss modifications for imbalanced learning. Unlike the previously considered implicit-bias view in \cite{byrd2019effect,sagawa2020investigation,VS,wang2021importance}, which is limited to linear (thus, fixed-feature) models and/or binary settings, our approach applies to learned-feature models and multiclass settings.

\noindent\textbf{Limitations.}~In the spirit of previous works \citep{NC,fang2021exploring,seli} that our result builds upon, it also shares some of the same limitations. First, the characterizations of the involved geometries are asymptotic in the number of training epochs. That is, while  as training progresses the classifiers/embeddings geometries are expected to converge to some prescribed limit, this convergence can be (very) slow. The specific convergence behavior that we see for CDT/LDT losses is of similar nature to the convergence for the CE loss in \cite{NC,zhu2021geometric,seli}. For CDT/LDT losses, we also observe that convergence speed can vary significantly for varying values of the hyperparameters. This issue  appears already for the UFM itself and is consistent for deeper architectures and complex data (see Sec.~\ref{sec:OptIm_Discuss}). Second, the level of convergence that can be reached in realistic training settings generally varies between architectures, data models and the loss that is optimized. For example, we find that CDT classifier geometry converges very well to its prescribed limit, but the same is not true for the embeddings geometry for the same loss or for the classifiers geometry for the LDT. Consistently, the experiments in \cite{NC} show different levels of convergence between different metrics (e.g. classifiers vs embeddings, norms vs angles) and different architectures/datasets. 
Third, despite some initial efforts {(including the preliminary results discussed in Sec.~\ref{sec:test_results})}, there is no explicit known link between different geometries and generalization. It is becoming apparent that this is one of the most pressing questions in the emerging literature thrust and we expect more investigations to follow in this direction. 
Finally, similar to \cite{seli}, we rely on the results of \citet{lyu2019gradient,ji2020directional} on convergence of gradient flow in homogeneous networks to the KKT points of the appropriate CS-SVM problems. While the UFM belongs to the category of homogeneous networks, analysis of more complex models could help shed light on other aspects of training deep-nets such as the worse convergence of embeddings. Recent works \citep[e.g.,][]{le2022training,jacot2022implicit}) have considered extending the nature of implicit bias to non-linear networks.

\noindent\textbf{Outlook and future directions.}~While it is important to realize these shortcomings, it is equally important realizing that the quest for implicit geometries is by nature highly non-trivial: we seek geometry characterizations for classifiers and mean-embeddings that are learned by different complex deep architectures over different complex datasets. Specializing to our setting, we further have different losses (LDT vs CDT), different hyperparameters for each loss,  and different imbalance ratios. Paraphrasing \cite{NC}: one might anticipate that the classifier and embeddings being by-product of training in such complex environments display no underlying structure. In view of these, we find the level of agreement of the empirically measured angles/norms to the respective (closed-form) $(\deltab,R)$-SELI geometry values rather striking. For example, see first row of Fig.~\ref{fig:convergence_to_theory}. Similarly,  inspecting Fig. \ref{fig:UFM_CDT_LDT_WCE_Comparison}, why should one expect a priori that there is a single, simple formula parameterized by the loss  hyperparameters that captures the norm-ratio behaviors of the classifiers learned by a 6-layer MLP on MNIST and a ResNet18 on CIFAR10? In view of these, we deem our findings encouraging and supportive of the quest set by the emerging literature on such structural characterizations. At the same time, our findings are suggestive of several research directions that are important investigating further. First, while the UFM has proven powerful to be predictive of behaviors across different levels of imbalances and different losses, a major limitation remains that it does not capture the required centering needed for the embeddings (see Remark \ref{rem:center}). This is a common theme also in previous works and is further highlighted here since in the new geometries the ``correct'' centering, done at a heuristic level in our experiments, is more intricate as it involves scaling with hyperparameter values. Second, while we characterize global minima of the CS-SVMs, it is not yet known whether SGD converges to those minima under all our settings. Third, is it possible to speed up training so that convergence to the asymptotic limits is faster? Finally, more investigations are required both on theory and experiments to distill connections between geometries and generalization. We hope that some of our findings motivate further such investigations, which are otherwise beyond the scope of this paper. 




%% file: sections/Zhat-V2.tex
\newpage
\noindent\textbf{Notation.}~For matrix $\Vb\in\R^{m\times n}$, $\Vb[i,j]$ denotes its $(i,j)$-th entry, $\vb_j$ denotes the $j$-th \emph{column}, $\Vb^T$ its transpose.
$\Vb_{j:k}\in\R^{m\times (k-j+1)}$ chooses columns $j,j+1,\ldots,k$ of $\Vb$, and $\Vb^T_{j:k}\in\R^{n\times (k-j+1)}$ does so on $\Vb^T$.
We denote $\|\Vb\|_F, \|\Vb\|_2$, and, $\|\Vb\|_*$ the Frobenius, spectral, and, nuclear norms of $\Vb$.  $\tr(\Vb)$ denotes the trace of $\Vb$. We use $\Vb \propto \mathbf{X}$ whenever the two matrices are equal up to a scalar constant. {For a vector $\mathbf{v}\in\R^k$, $\text{diag}(\mathbf{v})\in\R^{k\times k}$ is the diagonal matrix with $\mathbf{v}$ on its diagonal.} $\odot$ and $\otimes$ denote Hadammard and Kronecker products, respectively.
We use $\ones_m$ to denote an  $m$-dimensional vector of all ones and $\Id_m$ for the $m$-dimensional identity matrix. For vectors/matrices with all zero entries, we simply write $0$, as dimensions are easily understood from context. $\eb_{j}$ is the $j$-th standard basis vector, a column vector with a single non-zero entry of $1$ in the $j$-th entry. {Finally, we denote the set of positive rational numbers by $\mathbb{Q}_+$.}

\section{Proof of Lemma \ref{lem:motivation}}\label{sec:lemma}
\begin{lemma}[Binary]
    Consider $k=2$, linear model, separable data and minimizing un-regularized LDT/CDT/binary-CE/binary-VS losses. The LDT rule coincides with the classification rule of the binary VS loss assuming same $\delta$-tuning. On the other hand, minimizing  CDT results in  the same classification rule as CE, irrespective of the $\delta$-tuning. 
\end{lemma}
\begin{proof}
    Let $\W^\CDT$, $\W^\LDT \in\R^{d\times 2}$ denote the CDT and LDT classifiers respectively. The corresponding classification rules is:
    $
    (\w_1-\w_2)^T\x  \mathrel{\mathop\gtrless\limits^{\hat{y}(\W) = 1}_{\hat{y}(\W) = 2}} 0
    $
    for $\W=\begin{bmatrix}\w_1,\w_2\end{bmatrix}$ either $\W^\CDT$ or $\W^\LDT$, respectively. 
    On the other hand, the CE or binary VS loss decision rule is 
    $
         \x^T\w_* \mathrel{\mathop\gtrless\limits^{\hat{\upsilon}(\w_*) = 1}_{\hat{\upsilon}(\w_*) = -1}} 0,
    $
    where $\w_*$ denotes a minimizer of either the CE or the binary VS loss. Here, we use $\upsilon\in\{\pm1\}$ to denote the label encoding for binary CE loss, differentiating from the multiclass encoding $y\in\{1,2\}$ above. From the above two, we conclude that $\w_\star=\alpha(\w_1-\w_2), \alpha>0$ implies $\hat{y}(\W)=1\Longleftrightarrow\hat{\upsilon}(\w_*)=1$ (eqv. $\hat{y}(\W)=2\Longleftrightarrow\hat{\upsilon}(\w_*)=-1$).
    
    Since we minimize all losses without regularization and data are separable, it suffices by implicit bias \citep{soudry2018implicit,VS} to consider the solutions to the corresponding max-margin problems, i.e.,
    \begin{subequations}
    \begin{align}
        \W^\CDT&:=\arg\min_\W\,\|\W\|_F^2\quad\text{subj. to}~~(\delta_{y_i}\w_{y_i}-\delta_c\w_c)^T\h_i\geq 1, c\neq y_i
        \label{eq:linear_cdt_svm}
        \\
        \W^\LDT&:=\arg\min_\W\,\|\W\|_F^2\quad\text{subj. to}~~\delta_{y_i}(\w_{y_i}-\w_c)^T\h_i\geq 1, c\neq y_i \label{eq:linear_ldt_svm}
        \\
        \w_\star^\CE&:=\arg\min_\w\,\|\w\|_2^2\quad\text{subj. to}~~\upsilon_i\w_\star^T\h_i\geq 1,
        \label{eq:linear_ce_binary}
        \\
        \w_\star^\text{VS,binary}&:=\arg\min_\w\,\|\w\|_2^2\quad\text{subj. to}~~\upsilon_i\delta_{\upsilon_i}\w_\star^T\h_i\geq 1,
        \label{eq:linear_vs_binary}
    \end{align}
    \end{subequations}
    
    First, we show $\hat{y}(\W^\LDT)=1\Longleftrightarrow\hat{\upsilon}(\w_*^\text{VS,binary})=1$ provided that the LDT and VS loss parameters are matching, i.e. $\delta^\LDT_{1}=\delta^\text{VS,binary}_{1}$ and $\delta^\LDT_{2}=\delta^\text{VS,binary}_{-1}$.  This follows from the fact that $\w_1^\LDT+\w_2^\LDT=0$ (see Lem.~\ref{lem:linear_centering}). Thus, the minimization in \eqref{eq:linear_ldt_svm} does not change by adding the constraint $\w_2=-\w_1.$ But then, the solution set of \eqref{eq:linear_ldt_svm} is the same as the solution set of the minimization 
    $$
    \min_{\w_1}\,\|\w_1\|^2\quad\text{subj. to}~~\begin{cases}
    2\delta_1\w_1^T\h_i\geq 1 &i: y_i=1
    \\
    -2\delta_2\w_1^T\h_i\geq 1 & i: y_i=2
    \end{cases}.
    $$
    Comparing this to \eqref{eq:linear_vs_binary}, it follows immediately that $\w_1^\LDT=\w_\star^\text{VS,binary}/2.$ Hence, $\w_1^\LDT-\w_2^\LDT=\w_\star^\text{VS,binary},$ which proves the desired. 
    
    
    Second, we show that $\hat{y}(\W^\CDT)=\hat{y}(\w_*^\text{CE}).$ This is a consequence of the fact that $\delta_1^{-1}\w_1^\CDT+\delta_2^{-1}\w_2^\CDT = 0$ (see Lem.~\ref{lem:linear_centering}). Indeed, we then have that the solution set of \eqref{eq:linear_cdt_svm} does not change by adding the constraint $\w_2=-(\delta_2/\delta_1)\w_1$. But then, optimization is equivalent to:
    $$
    \min_{\w_1}\,\|\w_1\|^2\quad\text{subj. to}~~\begin{cases}
    (\delta_1\w_1-\delta_2\w_2)^T\h_i = (\delta_1+\delta_2^2/\delta_1)\w_1^T\h_i\geq 1 &i: y_i=1
    \\
    (\delta_2\w_2-\delta_1\w_1)^T\h_i = -(\delta_1+\delta_2^2/\delta_1)\w_1^T\h_i\geq 1 & i: y_i=2
    \end{cases}.
    $$
    Comparing this to \eqref{eq:linear_ce_binary}, we find that $\w_1^\CDT=\frac{\delta_1}{\delta_1^2+\delta_2^2}\w_\star^\CE.$ Thus also, $\w_2^\CDT=-\frac{\delta_2}{\delta_1^2+\delta_2^2}\w_\star^\CE$. In conclusion, $\w_1^\CDT-\w_2^\CDT=\frac{\delta_1+\delta_2}{\delta_1^2+\delta_2^2}\w_\star^\CE$, from which the desired follows since $\delta_1,\delta_2>0.$ 

\end{proof}

\begin{lemma}\label{lem:linear_centering}
    For the CDT/LDT-SVM classifiers $\W^\CDT,\W^\LDT$ defined in \eqref{eq:linear_ldt_svm} and \eqref{eq:linear_cdt_svm}, it holds that $\w_1^\LDT+\w_2^\LDT=0$ and $\delta_1^{-1}\w_1^\CDT+\delta_2^{-1}\w_1^\CDT=0$.
\end{lemma}
\begin{proof}We prove the claim for CDT. The proof for LDT is the same and is omitted for brevity. We use a  symmetrization argument as follows. Set $$\bar\w:=(\delta_1^{-1}\w_1^\CDT+\delta_2^{-1}\w_2^\CDT)/(\delta_1^{-2}+\delta_2^{-2}),$$ and assume for the sake of contradiction that $\bar\w\neq 0$. Consider a new classifier defined as  $\wt\w_1=\w_1^\CDT-\delta_1^{-1}\bar\w$ and $ \wt\w_2=\w_2^\CDT-\delta_2^{-1}\bar\w$. Clearly, it holds that $\delta_1\wt\w_1-\delta_2\wt\w_2=\delta_1\w_1^\CDT-\delta_2\w_2^\CDT$. Thus, $[\wt\w_1,\wt\w_2]$ is feasible in \eqref{eq:linear_ldt_svm}. Moreover,  $$\|\wt\w_1\|_2^2+\|\wt\w_2\|_2^2=\|\w_1^\CDT\|_2^2+\|\w_2^\CDT\|_2^2-(\delta_1^{-2}+\delta_2^{-2})\|\bar\w\|^2<\|\w_1^\CDT\|_2^2+\|\w_2^\CDT\|_2^2.$$ But, these together contradict the optimality of $\W^\CDT.$
\end{proof}

\section{Eigen-Structure of the $({\delta}, R)$-\SEL~Matrix}\label{sec:eigen_SEL}

In this section, we compute the eigen-structure of $(\deltab, R)$-\SEL~matrix $\Zhat$ (Defn.~\ref{def:sel}) for a $(\deltab, R)$-STEP imbalanced setting with STEP logit adjustments. For simplicity, we let $\delmin=1$, $\delmaj=\Delta$ and $\alpha=1$ (i.e. $R \in \N$).\footnote{To relax these assumptions, we only need to change the scale of the eigen-factors. Particularly, singular values should be scaled by $\sqrt{\alpha}/\sqrt{\delmin}$, and $\Ub$ by $1/\sqrt{\alpha}$. Thus, the results easily extend for general $\delmin$ and $\alpha$, i.e., rational $R$.}
For $m\in[k]$, define $\Pb_{m}\in\R^{m\times(m-1)}$ as an orthonormal basis of the subspace orthogonal to $\ones_m$, i.e. $\Pb_m\Pb_m^T=\Id_m-\frac{1}{m}\ones_m\ones_m^T$ and $\Pb_m^T\Pb_m=\Id_{m-1}$, and $\Sb_m(\sigma):=\Id_m-\sigma\ones_m\ones_m^T \in\R^{m}.$ {Throughout the rest of the paper, we let $\Ub_\otimes= \begin{bmatrix}
			\Ub^T_{1:\rhobar k} {\otimes \ones_{\alpha R}^T} & \Ub^T_{(\rhobar k + 1):k} {\otimes \ones_{\alpha}^T}
		\end{bmatrix}^T$}. 
		
	%
	%
	%
	%
	\begin{lemma}[$(\deltab, R)$-\SEL~matrix SVD]
		\label{lem:Zhat_SVD_general}
		Let $R \in \N$ and $\Zhat\in\R^{k\times n}$ be the $(\deltab, R)$-\SEL~matrix described in Defn. \ref{def:sel}, where recall that $n= k (R\rhobar + \rho)$. Define the SVD of $\Zhat$ as follows,
		\begin{align*}
			\Zhat=\Vb\Lambdab\begin{bmatrix}
				\Ub^T_{1:\rhobar k} {\otimes \ones_{ R}^T} & \Ub^T_{(\rhobar k + 1):k} 
			\end{bmatrix}=:\Vb\Lambdab\Ub_\otimes^T,
		\end{align*}
	and further let $\Vb=[\Vb_{\emph{maj}},\vb,\Vb_{\emph{min}}]$ and $\Ub=[\Ub_{\emph{maj}},\ub,\Ub_{\emph{min}}]$. Then, the SVD factors are given by the following equations:
		\begin{align}
			\Lambdab &= \diag{\left(
				\begin{bmatrix}
					\frac{\sqrt{R}}{\Delta}\ones_{(\rhobar k-1)}^T & \sqrt{\frac{\rhobar+R\rho}{\rhobar + \rho \Delta^2}} &\ones_{(\rho k-1)}^T
				\end{bmatrix}
				\right)}\,\label{eq:Lambda_general},
			\\
			\Vb_{\text{maj}}&=\begin{bmatrix}
				\Pb_{\rhobar k} \\
				0_{(\rho k)\times{(\rhobar k-1)}} 
			\end{bmatrix}
			\quad
			\vb = \frac{1}{\sqrt{k(\rhobar + \rho \Delta ^ 2) }}
			\begin{bmatrix}
				-\Delta \sqrt{\frac{\rho}{\rhobar}}\ones_{\rhobar k} \\ \sqrt{\frac{\rhobar}{\rho}}\ones_{\rho k}
			\end{bmatrix}
			\quad
			\Vb_{\text{min}}=\begin{bmatrix}
				0_{(\rhobar k)\times{(\rho k-1)}} \\
				\Pb_{\rho k}
			\end{bmatrix} \label{eq:V_general},
			\\
			\Ub_{\text{maj}}&=\begin{bmatrix}
				\frac{1}{\sqrt{R}}\Pb_{\rhobar k}
				\\
				0_{(\rho k)\times{(\rhobar k-1)}}  
			\end{bmatrix}
			\quad
			\ub = \frac{1}{\sqrt{k(\rhobar+R\rho)}}\begin{bmatrix}
				-\sqrt{\frac{\rho}{\rhobar}}\ones_{\rhobar k} \\ \sqrt{\frac{\rhobar}{\rho}}\ones_{\rho k}
			\end{bmatrix}
			\quad
			\Ub_{\text{min}}=
			\begin{bmatrix}
				0_{(\rhobar k)\times{(\rho k-1)}} \\
				\Pb_{\rho k}
			\end{bmatrix} \,. \label{eq:U_general}
		\end{align}
		
	\end{lemma}
	\begin{proof}
	    
		To prove the lemma, we only need to verify the correctness of the formulas. In particular: (1) $\Ub_\otimes$ and $\Vb$ are unitary matrices, and (2)   $\Vb\Lambdab\Ub_\otimes^T=\Zhat$. By recalling that $\Pb_m^T\Pb_m=\Id_{m-1}$ and $\Pb_m^T\ones_m=0$ for $m \in \{\rho k, \rhobar k\}$, it is easy to confirm $\Vb^T\Vb=\Id_{k-1}$ and $\Ub_\otimes^T\Ub_\otimes=\Id_{k-1}$. Since $\Ub_\otimes$ and $\Zhat$ have the same pattern of repeated columns, proving $\Vb\Lambdab\Ub^T=\Xib$ verifies the decomposition. So, we start by expressing $\Xib$ in block-form as follows:
		\begin{align}\label{eq:Zhat_block}
			\Xib  = 
			\begin{bmatrix}
				\Delta^{-1}\Sb_{\rhobar k}\big(\frac{1}{k(\rhobar + \rho \Delta ^ 2) }\big)
				&
				-\frac{\Delta}{k(\rhobar + \rho \Delta^2)}\ones_{\rhobar k}\ones^T_{\rho k}
				\\
				-\frac{1}{k(\rhobar + \rho \Delta ^ 2) }\ones_{\rho k}\ones^T_{\rhobar k}
				&
				\Sb_{\rho k}\big(\frac{\Delta ^ 2}{k(\rhobar + \rho \Delta ^ {2})}\big)
			\end{bmatrix}.
		\end{align} 
		Now, we can verify the equation by direct calculations:
		\begin{align*}
			\Vb\Lambdab\Ub^T &= \frac{\sqrt{R}}{\Delta}\Vb_{\text{maj}}\Ub_{\text{maj}}^T + \sqrt{\frac{\rhobar+R\rho}{\rhobar + \rho \Delta^2 }}\vb\ub^T + \Vb_{\text{min}}\Ub_{\text{min}}^T
			\\
			&=\begin{bmatrix}
				\Delta^{-1}\Pb_{\rhobar k}\Pb_{\rhobar k}^T & 0\\
				0 & 0
			\end{bmatrix}
			+
			\frac{1}{k(\rhobar + \rho \Delta^2)}
			\begin{bmatrix}
				\Delta\frac{\rho}{\rhobar}\ones_{\rhobar k}\ones_{\rhobar k}^T & -\Delta\ones_{\rhobar k}\ones_{\rho k}^T
				\\
				-\ones_{\rho k}\ones_{\rhobar k}^T & \frac{\rhobar}{\rho}\ones_{\rho k}\ones_{\rho k}^T
			\end{bmatrix}
			+
			\begin{bmatrix}
				0 & 0
				\\
				0 & \Pb_{\rho k}\Pb_{\rho k}^T
			\end{bmatrix}
			\\
			&=
			\begin{bmatrix}
				\Delta^{-1}\big(\Id_{\rhobar k}-\frac{1}{k(\rhobar + \rho\Delta^2)}\ones_{\rhobar k}\ones_{\rhobar k}^T\big) & -\frac{\Delta}{k(\rhobar+\rho \Delta^2)}\ones_{\rhobar k}\ones_{\rho k}^T\\
				-\frac{1}{k(\rhobar + \rho\Delta^2)}\ones_{\rho k}\ones_{\rhobar k}^T & \Id_{\rho k}-\frac{\Delta^2}{k(\rhobar+\rho\Delta^{2})}\ones_{\rho k}\ones_{\rho k}^T
			\end{bmatrix}
			\\
			&= \Xib.
		\end{align*}
	\end{proof}
	With the eigen-structure of $\Zhat$ at hand, we prove a useful property of the singular space in Lem.~\ref{eq:SVD_elementwise}. We will use this property later in Sec.~\ref{sec:proof_SVM} to characterize the solutions of the CS-SVM corresponding to CDT loss in \eqref{eq:svm_cdt}.
%
%
%
	\begin{lemma}\label{eq:SVD_elementwise}
		Recall the setting of Lem.~\ref{lem:Zhat_SVD_general} and the SVD $\Zhat=\Vb\Lambdab{\Ub_\otimes}^T$. The matrix $\Bb^*={\Ub_\otimes}\Vb^T$ satisfies the following element-wise \emph{strict} inequalities: $\Bb^*\odot\Zhat^T>0.$
	\end{lemma}
	\begin{proof}
		We compute ${\Bb}^* := \begin{bmatrix}
			\Bb^*_{11} & \Bb^*_{12}
			\\
			\Bb^*_{21} & \Bb^*_{22}
		\end{bmatrix}$
		by plugging in the explicit SVD expressions in Lem.~\ref{lem:Zhat_SVD_general}.
		\begin{align}
			{\Ub_\otimes}\Vb^T = 
			\begin{bmatrix}
				\frac{1}{\sqrt{R}}\Pb_{\rhobar k}\Pb_{\rhobar k}^T\otimes\ones_R & 0
				\\
				0 & 0
			\end{bmatrix} + 
			\frac{1}{k\sqrt{\rhobar+R\rho}\,\sqrt{\rhobar+\rho\Delta^{2}}}
			\begin{bmatrix}
				\Delta\frac{\rho}{\rhobar}\ones_{\rhobar k}\ones_{\rhobar k}^T \otimes\ones_R & -\ones_{\rhobar k}\ones_{\rho k}^T \otimes\ones_R \\
				-\Delta\ones_{\rho k}\ones_{\rhobar k}^T & \frac{\rhobar}{\rho}\ones_{\rho k}\ones_{\rho k}^T
			\end{bmatrix}
			+
			\begin{bmatrix}
				0& 0
				\\
				0 & \Pb_{\rho k}\Pb_{\rho k}^T 
			\end{bmatrix}
			\nn.
		\end{align} 
		Simplifying the expressions, we have
		\begin{align*}
			\Bb^*_{11} &= 
			 \frac{1}{\sqrt{R}}\left(\Id_{\rhobar k}-\frac{1}{\rhobar k}\left(1-\Delta\sqrt{\frac{{R \rho}}{(R+{\rhobar}/{\rho})(\rhobar+\rho\Delta^2)}}\right)\ones_{\rhobar k}\ones_{\rhobar k}\right)
			\otimes\ones_R \,, \nn
			\\
			\Bb^*_{12} &= -\frac{1}{k\sqrt{\rhobar+R\rho}\,\sqrt{\rho+\rhobar\Delta^{2}}}\ones_{\rhobar k}\ones_{\rho k}^T \,\otimes\ones_R \,, \nn\\
			\Bb^*_{21} &= -\frac{\Delta}{k\sqrt{\rhobar+R\rho}\,\sqrt{\rho+\rhobar\Delta^{2}}}\ones_{\rho k}\ones_{\rhobar k}^T \,, \nn
			\\
			\Bb^*_{22} &= 
			\Id_{\rho k}-\frac{1}{\rho k}\left(1-\sqrt{\frac{\rhobar}{{(1 + R \left({\rho}/{\rhobar}\right))(\rhobar+\rho\Delta^{2})}}}\right)\ones_{\rho k}\ones_{\rho k}^T.
		\end{align*} 
		From \eqref{eq:Zhat_block}, we can write $\Zhat$ in block-form: 
		\begin{align*}
			\Zhat^T= 
			\begin{bmatrix}
				\Delta^{-1}\Sb_{\rhobar k}\big(\frac{1}{k(\rhobar + \rho \Delta ^ 2) }\big)\otimes\ones_R
				&
				-\frac{1}{k(\rhobar + \rho \Delta^2)}\ones_{\rhobar k}\ones^T_{\rho k} \otimes\ones_R
				\\
				-\frac{\Delta}{k(\rhobar + \rho \Delta ^ 2) }\ones_{\rho k}\ones^T_{\rhobar k}
				&
				\Sb_{\rho k}\big(\frac{\Delta ^ 2}{k(\rhobar + \rho \Delta ^ {2})}\big)
			\end{bmatrix}.
		\end{align*}

		The signs of the off-diagonal blocks of both $\Zhat$ and ${\Bb}^*$ are negative. To inspect the sign agreement of the on-diagonal blocks, it is enough to see the following inequalities are always strictly satisfied,
		\begin{align*}
			1>1-\Delta\sqrt{\frac{{R \rho}}{(R+{\rhobar}/{\rho})(\rhobar+\rho\Delta^{2})}}>0\qquad\text{and}\qquad 1>1-\sqrt{\frac{\rhobar}{(1+R\left(\rho/\rhobar\right))(\rhobar + \rho\Delta^2)}} >0.
		\end{align*}
		\end{proof}

%% file: sections/proof_SVM-V2.tex

{One of the paper's main contributions is introducing the $(\deltab,R)$-SELI geometry (Defn.~\ref{dfn:seli}) as the ``correct'' formalization that is able to capture the implicit geometries of \emph{both} the CDT and LDT losses for all imbalance-ratio values $R$.\footnote{Since CE loss is a special case of CDT/LDT loss for $\deltab=\ones_k$, the new geometry includes the previously introduced SELI \citep{seli} and ETF \citep{NC} geometries as special cases.} This property is captured by Thm.~\ref{thm:SVM-VS}: thanks to the generality of Defn.~\ref{dfn:seli}, both CDT and LDT geometries, albeit different to each other, are formalized in terms of appropriate parameterizations of the same geometry. 
This unifying and concise formalization of the theorem is central to our work. For example, the eigenstructure properties of the $(\delta,R)$-SEL matrix in \Sec~\ref{sec:eigen_SEL} and the closed-form angles/norm-formulas in \Sec~\ref{sec:SELI_properties} apply immediately to both losses. Instead in this section, when proving Thm.~\ref{thm:SVM-VS}, we find it more appropriate to treat the two losses separately: the proofs for CDT and LDT losses are included in \Sec~\ref{sec:app_cdt_proof} and \Sec~\ref{sec:app_ldt_proof}, respectively.} 

{ Our proof in \Sec~\ref{sec:app_cdt_proof} for CDT generalizes the proof of \citet[Thm.~1]{seli}, which only applies for the CE loss (a special case of CDT). At a high-level, the key innovations making this possible are: (i) formalizing the $(\delta,R)$-SEL matrix (see Defn. \ref{def:sel}) as the appropriate generalization of the SEL matrix in \citet{seli}; (ii) expressing the dual of the CS-SVM corresponding to CDT (Eqn.~\eqref{eq:svm_cdt}) in a form that involves the $(\delta,R)$-SEL and showing that it admits an explicit solution. }

{Our proof in \Sec~\ref{sec:app_ldt_proof} for LDT relies on the following reduction idea: we prove that it is possible to re-parameterize the CS-SVM corresponding to LDT (Eqn.~\eqref{eq:svm_ldt}) such that it reduces to a weighted version of the standard unconstrained-features SVM (UF-SVM) for CE loss, {CS-SVM with $\deltab = \ones_k$}, (see Prop.~\ref{prop:seli_beta}), albeit the new UF-SVM is over an artificial dataset with different imbalance ratio that is only introduced for the purpose of the proof. This reduction, together with the general formalization of the $(\delta,R)$-SEL matrix, then allows us to leverage \citet[Thm.~1]{seli}. 
}

\subsection{CDT Loss: Theorem.~\ref{thm:SVM-VS} \textbf{(ii)}}\label{sec:app_cdt_proof}

Consider the CS-SVM of \eqref{eq:svm_cdt}:
\begin{align}\label{eq:svm_cdt_app}
\mathrm{p}_*=\min_{\W,\Hb}~\frac{1}{2}\|\W\|_F^2+ \frac{1}{2}\|\Hb\|_F^2
\quad\quad\text{sub. to}\quad (\delta_{y_i}\w_{y_i}-\delta_{c}\w_c)^T\h_i \geq 1,~i\in[n], c\neq y_i,
\end{align}
and let the optimal parameters of the problem be $(\W^*,\Hb^*)$. We start by setting $\X = \begin{bmatrix} \W^T \\ {\Hb}^T \end{bmatrix} \begin{bmatrix} \W & \Hb \end{bmatrix}\in\R^{(k+n)\times (k+n)}$ and relaxing \eqref{eq:svm_cdt_app} as follows,
\begin{align}\label{eq:cvx_X}
	\mathrm{q}_*  &= \min_{\X\succeq 0}~~~~~~
	\frac{1}{2}\tr\Big(\X\Big)\\
	\nn&~~~\text{sub.~to}~~~~\delta_{y_i}\X[y_i,k+i]-\delta_c\X[c,k+i]\geq 1,~ \forall i\in[n], c\neq y_i.
\end{align}
Clearly, $\mathrm{p}_* \geq \mathrm{q}_*$. Our key insight in the analysis of \eqref{eq:cvx_X} is writing its dual in a way that involves explicitly the $(\deltab,R)$-\SEL~matrix. 
{Specifically, let $\hat{\Z}$ be the $(\deltab,R)$-SEL matrix of Defn.~\ref{def:sel} with $\alpha=n_\text{min}$. Then, we can formulate the dual of \eqref{eq:cvx_X} as follows:}
\begin{align}\label{eq:dual2_X}
	\mathrm{d}_*  &=\max_{\Bb\in\R^{n\times k} }~~~~ \tr(\Zhat\Bb)
	\\&~~~~\text{sub.~to}~~ \begin{bmatrix}
	    \Id_d & -\Bb^T \\
	    -\Bb & \Id_n
	\end{bmatrix}
	\succeq 1 \nn
	\\&~~~~~~~~~~~~~~~~ 
	\Bb\Db^{-1}\ones_k=0\nn
	\\&~~~~~~~~~~~~~~~~ \Bb\odot\Zhat^T \geq 0 \label{eq:other}\,,
\end{align}
where $\Bb$ contains the dual variables and $\Db = \diag{\left(\deltab\right)}\in\R^{k\times k}$. It is easy to see that strong duality holds for the convex problem \eqref{eq:cvx_X} by satisfying Slater's condition. Thus, using the optimal solution of \eqref{eq:dual2_X}, we can characterize the optimizers \eqref{eq:cvx_X}.

To solve \eqref{eq:dual2_X}, we first relax the problem by ignoring constraint \eqref{eq:other}, and substituting the first constraint using Schur-complement argument:
\begin{align}
	\max_{\|\Bb\|_2\leq 1}~~ \tr(\Zhat\Bb)\qquad\text{sub. to}~~ \Bb\Db^{-1}\ones_k=0.\label{eq:unconstrained_B}
\end{align}
The optimal value of \eqref{eq:unconstrained_B} is $\norm{\Zhat}_*$ and ${{\Bb^*}}=\Ub_\otimes\Vb^T$ is the unique solution (see \citet[Lem.~C.1]{seli})\footnote{\citet[Lem.~C.1]{seli} holds for $(\ones_k,R)$-SEL matrix $\Zhat$, but inspecting the  proof  it remains unchanged for the general $(\deltab,R)$-SEL matrix.}. 
By Lem.~\ref{eq:SVD_elementwise}, $\Bb^*$ is strictly feasible in the relaxed condition \eqref{eq:other}. Therefore, the relaxation in \eqref{eq:unconstrained_B} is tight and $\Bb^*$ is in fact the dual optimal of \eqref{eq:cvx_X}. Since, strong duality holds for \eqref{eq:cvx_X}, we also have $\mathrm{q}_*=\norm{\Zhat}_*$ and the optimizer $\X^*$ can be found by the complementary slackness conditions:
\begin{align*}
    \forall i\in[n], c\neq y_i:\quad & \Bb^*[i,c](1-\delta_{y_i}\X[y_i,k+i]+\delta_c\X[c,k+i])=0\\
    &\begin{bmatrix}
        \Id_k & -{\Bb^*}^T\\
        -\Bb^* & \Id_n
    \end{bmatrix} \X = 0.
\end{align*}
Let $\X^* = \begin{bmatrix} \X^*_{11} & \X^*_{12} \\ \X^*_{21} & \X^*_{22}\end{bmatrix}$, and recall that $\Bb^*$ satisfies \eqref{eq:other} strictly. Then, the complementary slackness conditions imply:
\begin{align*}
    \forall i\in[n], c\neq y_i:\quad & 1-\delta_{y_i}\X^*_{12}[y_i,i]+\delta_c\X^*_{12}[c,i]=0\\
    &\X^*_{11} = {\Bb^*}^T{\X^*}^T_{12}, \quad \X^*_{22} = {\Bb^*}\X^*_{12}, \quad \X^*_{12}={\Bb^*}^T\X^*_{22}.
\end{align*}
From the last condition, it is straightforward to see $\X^*_{22}=\Ub_\otimes\tilde{\Lambdab}\Vb^T$, $\X^*_{12}=\Vb\tilde{\Lambdab}\Ub_\otimes^T$ and $\X^*_{11}=\Vb\tilde{\Lambdab}\Vb^T$ for some $\tilde{\Lambdab}\in\R^{(k-1)\times(k-1)}$. Now, using the first condition, we have,
\begin{align}\label{eq:X12}
    &\delta_c^{-2}\delta_{y_i}\X^*_{12}[y_i,i]-\delta_c^{-1}\X^*_{12}[c,i] = \delta_c^{-2}\nn\\
    \stackrel{\sum_{c\neq y_i}}{\implies}&  \delta_{y_i}\X^*_{12}[y_i,i]\sum_{c\neq y_i} \delta_c^{-2} -\sum_{c\neq y_i}\delta_c^{-1}\X^*_{12}[c,i] = \sum_{c\neq y_i}\delta_c^{-2}\nn\\
    \implies &\delta_{y_i}\X^*_{12}[y_i,i]\sum_{c\in[k]} \delta_c^{-2} -\sum_{c\in[k]}\delta_c^{-1}\X^*_{12}[c,i] = \sum_{c\neq y_i}\delta_c^{-2}\nn\\
    \stackrel{(i)}{\implies} &\delta_{y_i}\X^*_{12}[y_i,i]\sum_{c\in[k]} \delta_c^{-2} = \sum_{c\neq y_i}\delta_c^{-2}\nn\\
    \implies&\X^*_{12}[y_i,i] = \delta_{y_i}^{-1} \big(1-\delta_{y_i}^{-2}/\sum_{c'\in[k]} \delta_{c'}^{-2}\big), \quad \X^*_{12}[c,i] = -\delta_{c}^{-1} \big(\delta_{y_i}^{-2}/\sum_{c'\in[k]} \delta_{c'}^{-2}\big)\nn\\
    \implies &\X^*_{12} = \Zhat.
\end{align}
In (i), we use the fact that $\Vb^T\Db^{-1}\ones_k=0$ and thus ${\X^*_{12}}^T\Db^{-1}\ones_k=0$. By \eqref{eq:X12}, and using $\Vb^T\Vb=\Ub_\otimes^T\Ub_\otimes=\Id_{k-1}$ it is easy to show $\tilde{\Lambdab} = \Lambdab$ and thus,
\begin{align}\label{eq:X_*}
    \X^* = \begin{bmatrix}
    \Vb \\ \Ub_\otimes
    \end{bmatrix}\Lambdab\begin{bmatrix}
    \Vb^T & \Ub_\otimes^T
    \end{bmatrix}.
\end{align}
Now, it remains to show all the optimizers of \eqref{eq:svm_cdt_app} can be constructed by $\X^*$ and that the relaxation in \eqref{eq:cvx_X} is tight. First, choose some partial orthonormal matrix $\Rb\in\R^{(k-1)\times d}$ with $\Rb\Rb^T=\Id_{k-1}$, and construct $\W^*=\Rb^T\Lambdab^{1/2}\Vb^T$ and $\Hb^*=\Rb^T\Lambdab^{1/2}\Ub_R^T$. Then, $(\W^*,\Hb^*)$ is by construction feasible in \eqref{eq:svm_cdt_app} and, $$\mathrm{q}_*\leq\mathrm{p}_*\leq\frac{1}{2}\norm{\W^*}_F^2+\frac{1}{2}\norm{\Hb^*}_F^2 = \frac{1}{2}\tr(\X^*)=\mathrm{q}_*.$$
Therefore, $\mathrm{q}_*=\mathrm{p}_*$ and indeed the relaxation is tight. On the other hand, if $(\tilde{\W},\tilde{\Hb})$ is a minimizer of \eqref{eq:svm_cdt_app}, $\tilde{\X} = \begin{bmatrix} \tilde{\W}^T \\ \tilde{\Hb}^T\end{bmatrix}\begin{bmatrix} \tilde{\W} & \tilde{\Hb}\end{bmatrix}$ is feasible and optimal in \eqref{eq:cvx_X} (since $\mathrm{q}_*=\mathrm{p}_*$), which implies $\tilde{\X}$ should satisfy \eqref{eq:X_*}. Hence, any minimizer of the CS-SVM \eqref{eq:svm_cdt_app} satisfies,
\begin{align}\label{eq:opt_sol_cdt}
    \begin{bmatrix}
    {\W^*}^T \\ {\Hb^*}^T
    \end{bmatrix}\begin{bmatrix}
    {\W^*} & {\Hb^*}
    \end{bmatrix} = \begin{bmatrix}
    \Vb \\ \Ub_\otimes
    \end{bmatrix}\Lambdab\begin{bmatrix}
    \Vb^T & \Ub_\otimes^T
    \end{bmatrix}.
\end{align}
The statement of the theorem is easy to see by \eqref{eq:opt_sol_cdt}. Specifically, by noting that $\Ub_\otimes$ has repeated columns, {all the embeddings belonging to the same class are equal (NC occurs) and,}
\begin{align}\label{eq:cdt_opt_final}
    {\W^*}^T{\W^*} = \Vb \Lambdab \Vb^T, \quad {\Mb^*}^T{\Mb^*} = \Ub \Lambdab \Ub^T, \quad {\W^*}^T{\Mb^*} = \Xib.  
\end{align}

\begin{remark}\label{remark:alpha}
For simplicity of exposition, we set $\alpha=\nmin$ when using the $(\deltab,R)$-SEL matrix to formulate the dual problem. However, it is easy to see that by choosing some other $\alpha'$, the SVD factors would only change by a scaling factor. In particular, let $\tau=\sqrt{\alpha'/\nmin}$, then $\Ub$ and $\Lambdab$ will be scaled by a factor of $1/\tau$ and $\tau$ respectively, and $\Vb$ remains unchanged. Hence, \eqref{eq:cdt_opt_final} changes as follows,
\begin{align*}
    {\W^*}^T{\W^*} = \tau \Vb \Lambdab \Vb^T, \quad {\Mb^*}^T{\Mb^*} = \frac{1}{\tau} \Ub \Lambdab \Ub^T, \quad {\W^*}^T{\Mb^*} = \Xib.  
\end{align*}
\end{remark}



\subsection{LDT Loss: Theorem.~\ref{thm:SVM-VS} \textbf{(iii)}}\label{sec:app_ldt_proof}
We start the proof by restating a result from \cite{seli} regarding the optimal solutions of the unconstrained-features SVM.

\begin{propo}[{\citet[Sec.~C.3]{seli}}]\label{prop:seli_beta}
	 Consider the following $k$-class {$\beta$-weighted} unconstrained-features SVM (UF-SVM):
	\begin{align*}
		(\What_\beta,\Hhat_\beta)\in\arg\min_{\W,\Hb}~\frac{1}{2}\|\W\|_F^2+ \frac{\beta^2}{2}\|\Hb\|_F^2
		\quad\quad{\text{\emph{sub. to}}}\quad (\w_{y_i}-\w_c)^T\h_i \geq 1,~i\in[n],~c\neq y_i,~ c\in[k].
	\end{align*}
	in an $(R,\rho)$-STEP imbalanced setting. For any $\beta > 0$, the NC property holds, and 
	the optimal solutions $(\What_\beta,\Mbhat_\beta)$ follow the $(\ones_k, R)$-SELI geometry. Specifically, 
	\begin{align*}
	    \What_\beta^T\Mbhat_\beta = \Xib, \quad \Mbhat_\beta^T\Mbhat_\beta=\frac{1}{{\tau}}\Ub\Lambdab\Ub^T, \quad \What_\beta^T\What_\beta = {\tau} \Vb\Lambdab\Vb^T,
	\end{align*}
	where $\Vb$, $\Lambdab$, $\Ub$ are the SVD factors of the $(\ones_k,R)$-SEL matrix as described in Defn. \ref{def:sel}, {and $\tau$ is a positive scalar depending on $\beta$ and $n_\emph{minor}$.}
\end{propo}

Consider the k-class CS-SVM problem of \eqref{eq:svm_ldt}, restated below for convenience:
\begin{align}\label{eq:svm_ldt_app}
(\W^*,\Hb^*)\in\arg\min_{\W,\Hb}~\frac{1}{2}\|\W\|_F^2+ \frac{1}{2}\|\Hb\|_F^2
\quad\quad\text{sub. to}\quad \delta_{y_i}(\w_{y_i}-\w_c)^T\h_i \geq 1,~i\in[n], c\neq y_i.
\end{align}
Also recall that $n_c, c\in[k]$ is the number of examples in class $c$. We will relate the above optimization problem to an equivalent UF-SVM, whose solution can be found by Prop. \ref{prop:seli_beta}. The resulting solution will be used to state the minimizers of \eqref{eq:svm_ldt_app}.

First, it is easy to verify the NC property: for a fixed $\W$, the optimization in \eqref{eq:svm_ldt_app} is separable in $\h_i$, and for all the samples in the same class, the separable problems are identical and strongly-convex. Thus, for all $i: y_i=c$ there is a unique minimzer for the fixed $\W$. So, at the optimal solution, all the embeddings within a class are equal to their means, i.e. $\forall i \in [n]: y_i=c, \h_i = \mub_c$.
Defining $\Mb = [\mub_1,\ldots,\mub_k]$, we can re-formulate \eqref{eq:svm_ldt_app} as follows,
\begin{align}\label{eq:ldt_mu}
	(\W^*,\Mb^*)\in\arg\min_{\W,\Mb}~\frac{1}{2}\|\W\|_F^2+ \frac{1}{2}\sum_{c\in[k]} n_c \norm{\mub_c}_2^2
	\quad\quad\text{sub. to}\quad \delta_{c}(\w_c-\w_{c^\prime})^T\mub_c \geq 1,~c,c^\prime\in[k],
\end{align}
and by the NC property, there is a one-to-one correspondence between the optimal solutions of \eqref{eq:svm_ldt_app} and \eqref{eq:ldt_mu}.

Now, let $\Db=\diag(\deltab)$ and $\widetilde{\Mb} = \Mb \Db$, i.e. $\tilde{\mub}_c = \delta_c\mub_c$. Applying this reparametrization to \eqref{eq:ldt_mu}, we have,
\begin{align}\label{eq:ldt_reparam}
	(\W^*,\widetilde{\Mb}^*)\in\arg\min_{\W,\tilde{\Mb}}~\frac{1}{2}\|\W\|_F^2+ \frac{1}{2}\sum_{c\in[k]} \frac{n_c}{\delta_c^2} \norm{\tilde{\mub}_c}_2^2
	\quad\quad\text{sub. to}\quad (\w_c-\w_{c^\prime})^T\tilde{\mub}_c \geq 1,~c,c^\prime\in[k].
\end{align}

Define $\tilde{R} = R\big({\delmin}/{\delmaj}\big)^2$, which is rational by assumption. Thus, there exists $\alpha \in \N$ such that $\alpha \tilde{R}$ is an integer. Now, set $\beta^2 = {n_\text{min}}/({\alpha \delmin^2})$, and re-write \eqref{eq:ldt_reparam} as follows:
\begin{align}
	(\W^*,\widetilde{\Mb}^*)\in\arg\min_{\W,\tilde{\Mb}}~\frac{1}{2}\|\W\|_F^2+ \frac{\beta^2}{2}\sum_{c\in[k]} \tilde{n}_c \norm{\tilde{\mub}_c}_2^2
	\quad\quad\text{sub. to}\quad (\w_c-\w_{c^\prime})^T\tilde{\mub}_c \geq 1,~c,c^\prime\in[k],
\end{align}
where 
\begin{align*}
	\tilde{n}_c = \begin{cases}
		\alpha \tilde{R}, & \text{if } c \in \{1,\ldots,\rhobar k\},\\ 
		\alpha , & \text{if } c \in \{\rhobar k + 1,\ldots,k\}
	\end{cases}
\end{align*}

By a similar argument that led to the equivalence of \eqref{eq:svm_ldt_app} and \eqref{eq:ldt_mu}, it is easy to see $(\W^*,\widetilde{\Mb}^*)$ is the optimal parameters of a {$\beta$-weighted UF-SVM} trained on an imbalanced dataset with imbalance ratio $\tilde{R}$ and $\tilde{n}_c$ samples per class for $c \in [k]$. Thus, $(\W^*,\widetilde{\Mb}^*)$ follows the $(\ones_k, \tilde{R})$-SELI geometry as in Prop. \ref{prop:seli_beta}. The proof is complete by noting that $(\W^*,\widetilde{\Mb}^*) = (\W^*,{\Mb}^*\Db)$. 

%% file: sections/SELI_properties.tex
As stated in the Thm. \ref{thm:SVM-VS}, the optimal parameters of the CS-SVM under the CDT/LDT loss have a unique description in terms of the SVD factors of a corresponding label-encoding matrix. In this section, we use this characterization to derive explicit expressions for the parameters' geometry as a function of $R,\rho,k$ and of the hyper-parameters $\deltab$.

Similar to Sec. \ref{sec:eigen_SEL}, throughout this section, we assume the data is STEP imbalanced and STEP logit adjustment is adopted. For simplicity, we consider the case $\rho=1/2$, $\delmin=1$ and $\delmaj=\Delta$. This choice is without loss of generality since the geometry only depends on the ratio $\delmaj/\delmin$. {We use the closed-form SVD in Sec. \ref{sec:eigen_SEL} derived by assuming $\alpha=1$. It is easy to see that a general $\alpha$ only introduces an appropriate scaling to the SVD factors. (See Remark \ref{remark:alpha}). Thus, using the closed-form expressions in Lemma \ref{lem:Zhat_SVD_general} for the corresponding $\Vb$, $\Lambdab$, and $\Ub$, the optimal parameters satisfy:}
\begin{align}
    {\W^*}^T{\W^*} = \tau \Vb \Lambdab \Vb^T, \quad {\Mb^*}^T{\Mb^*} = \frac{1}{\tau} \Ub \Lambdab \Ub^T, \quad {\W^*}^T{\Mb^*} = \Xib,
\end{align}
{for some positive scalar $\tau$ (that depends on $\nmin$ and $\alpha$). Since, $\tau$ onlys affects the scale of the geometry, in the lemmas we assume $\tau=1$ for brevity.}

%
%
%
%
%
%
{In Sec.~\ref{sec:cdt_geo}, we describe the geometric and asymptotic properties of the solutions of \eqref{eq:svm_cdt}, the CS-SVM under CDT loss. In Sec.~\ref{sec:ldt_geo} we characterize the same properties for problem \eqref{eq:svm_ldt} corresponding to the LDT loss.}
{In the following lemmas, we use $\w_\text{maj}$ when referring to any majority classifier $\w_c, c \in \{1,\ldots,k/2\}$, and $\w_\text{minor}$ for any minority classifier $\w_c, c \in \{k/2+1,\ldots,k\}$. Similarly, $\h_\text{maj}$ denotes any $\h_j$ with $j\in\{i\in[n]: y_i=1,\ldots,k/2\}$ and $\h_\text{minor}$ denotes any $\h_j$ with $j \in \{i\in[n]:y_i=k/2+1,\ldots,k\}$.}

\subsection{CDT Loss}\label{sec:cdt_geo}
\subsubsection{Norms and Angles}\label{sec:cdt_norm_ang}
\begin{lemma}[CDT classifiers]\label{lem:cdt_classifiers}
Let $\Vb$, $\Lambdab$, $\Ub$ be the eigen-factors of the $(\deltab,R)$-SEL matrix.
For the optimal classifier $\W$ of the CS-SVM \eqref{eq:svm_cdt}:

\begin{enumerate}[label={(\alph*)},itemindent=0em]
\item[\emph{\textbf{(a)}}] \label{lem:class_norm}
\noindent\emph{\textbf{(Norms)}} All the majority/minority classes have equal norms, 
\begin{align}\label{eq:w_norm}
    \|\wmaj\|_2^2 = \frac{\sqrt{R}}{\Delta}(1-2/k)+\frac{2\Delta^2\sqrt{R+1}}{k\Big(\sqrt{1+\Delta^{2}}\Big)^{3}},
    \quad \|\wmin\|_2^2 = (1-2/k)+\frac{2\sqrt{R+1}}{k\Big(\sqrt{1+\Delta^{2}}\Big)^{3}},
\end{align}
and the majority-minority norm-ratio is,
\begin{align*}
    \frac{\|\wmaj\|_2^2}{\|\wmin\|_2^2} = \frac{\frac{\sqrt{R}}{\Delta}(k-2){\big({1+\Delta^{2}}\big)^{3/2}}+2\Delta^2\sqrt{R+1}}{(k-2)\big({1+\Delta^{2}}\big)^{3/2}+2\sqrt{R+1}}.
\end{align*}

\item[\emph{\textbf{(b)}}]\label{lem:class_angle}
\noindent\emph{\textbf{(Angles)}} For each pair of majority/minority classifiers the angles are equal and,
\begin{align*}
    \Cos{\wmaj}{\wmaj'}&=\frac{-2\sqrt{R}+2\sqrt{R+1}\left(\sqrt{1 + \Delta^{-2}}\right)^{-3}}{(k-2)\sqrt{R}+2\sqrt{R+1}\left(\sqrt{1 + \Delta^{-2}}\right)^{-3}}
\\
\Cos{\wmin}{\wmin'}&=\frac{-2 + 2\sqrt{R+1}\left(\sqrt{1 + \Delta^{2}}\right)^{-3}}{k-2+2\sqrt{R+1}\left(\sqrt{1 + \Delta^{2}}\right)^{-3}}
\\
\Cos{\wmaj}{\wmin}&=-\frac{2\Delta\sqrt{R+1}}{k\left(\sqrt{1+\Delta^{2}}\right)^3\|\wmaj\|_2\|\wmin\|_2}\,.
\end{align*}
\end{enumerate}
\end{lemma}
\begin{proof}
Let $m=k/2$. From Thm. \ref{thm:SVM-VS}, $\W^T\W=\Vb\Lambdab\Vb^T.$ Using Lem.~\ref{lem:Zhat_SVD_general} we have,
\begin{align*}
\Vb\Lambdab\Vb^T &= \frac{\sqrt{R}}{\Delta}
\begin{bmatrix}
\Pb_m\Pb_m^T & 0 \\ 0 & 0
\end{bmatrix}
+\frac{2\sqrt{R+1}}{k(\sqrt{1+\Delta^2})^{3}}\begin{bmatrix}
\Delta^2\ones_{m}\ones_{m}^T & -\Delta\ones_{m}\ones_{m}^T
\\
-\Delta\ones_{m}\ones_{m}^T & \ones_{m}\ones_{m}^T
\end{bmatrix}
+
\begin{bmatrix}
0 & 0 \\ 0 & \Pb_m\Pb_m^T
\end{bmatrix} \nn\\
&=
\begin{bmatrix}
\frac{\sqrt{R}}{\Delta}\Id_{k/2}-\frac{2}{k}\left(\frac{\sqrt{R}}{\Delta}-\frac{\Delta^2\sqrt{R+1}}{(\sqrt{1+\Delta^{2}})^{3}}\right)\ones_{k/2}\ones_{k/2}^T
&
-\frac{2\Delta\sqrt{R+1}}{k(\sqrt{1+\Delta^2})^{3}}\ones_{k/2}\ones_{k/2}^T
\\
-\frac{2\Delta\sqrt{R+1}}{k(\sqrt{1+\Delta^2})^{3}}\ones_{k/2}\ones_{k/2}^T
&
\Id_{k/2}-\frac{2}{k}\left(1-\frac{\sqrt{R+1}}{(\sqrt{1+\Delta^{2}})^{3}}\right)\ones_{k/2}\ones_{k/2}^T
\end{bmatrix}.
\end{align*}
Inspecting the diagonal entries 
proves the norm equations. 
To prove part \textbf{(b)}, we use the off-diagonals entries that specify the inner-product of each pair of classifiers. Particularly,
\begin{align*}
\wmaj^T\wmaj' &= \frac{-2\sqrt{R} + 2\sqrt{R+1}\left(\sqrt{\Delta^{-2} + 1}\right)^{-3}}{k\Delta},
\\
\w_\text{{minor}}^T\w_\text{{minor}}' &= \frac{-2 + 2\sqrt{R+1}\left(\sqrt{1 + \Delta^{2}}\right)^{-3}}{k},
\\
\w_\text{{minor}}^T\wmaj &= \frac{-2\Delta\sqrt{R+1}}{k\,\left(\sqrt{\Delta^{2}+1}\right)^{3}}\,.
\end{align*}
These equations together with \eqref{eq:w_norm} complete the proof.
\end{proof}

\begin{lemma}
[CDT embeddings]\label{lem:cdt_embeddings}
Let $\Vb$, $\Lambdab$, $\Ub$ be the eigen-factors of the $(\deltab,R)$-SEL matrix.
For the optimal embeddings $\Hb$ of the CS-SVM \eqref{eq:svm_cdt}:

\begin{enumerate}[label={(\alph*)},itemindent=0em]
\item[\emph{\textbf{(a)}}] 
\noindent\emph{\textbf{(Norms)}} All the embeddings in the majority/minority classes have equal norms,
\begin{align*}
    \|\hmaj\|_2^2 = \frac{(1-2/k)}{\Delta\sqrt{R}}+\frac{2}{k \sqrt{R+1} \sqrt{1 + \Delta^{2}}}, \quad
\|\hmin\|_2^2 = (1-2/k)+\frac{2}{k\sqrt{R+1} \sqrt{1 + \Delta^{2}}},
\end{align*}
and the majority-minority norm-ratio is as follows,
\begin{align*}
    \frac{\|\hmaj\|_2^2}{\|\hmin\|_2^2} = \frac{\frac{1}{\Delta\sqrt{R}}(k-2)\sqrt{R+1} \sqrt{1 + \Delta^{2}} + 2}{(k-2)\sqrt{R+1} \sqrt{1 + \Delta^{2}} + 2}.
\end{align*}
\item[\emph{\textbf{(b)}}]
\noindent\emph{\textbf{(Angles)}} For each pair of majority/minority embeddings the angles are equal, and,
\begin{align*}
\Cos{\hmaj}{\hmaj'}&=\frac{-2\sqrt{\Delta^{-2}+1}\sqrt{R+1} + 2\sqrt{R}}{(k-2)\sqrt{\Delta^{-2}+1}\sqrt{R+1} + 2\sqrt{R}}
\\
\Cos{\hmin}{\hmin'}&=\frac{-2\sqrt{\Delta^{2}+1}\sqrt{R+1} + 2}{(k-2)\sqrt{\Delta^{2}+1}\sqrt{R+1} + 2}
\\
\Cos{\hmaj}{\hmin}&=\frac{-2}{k\sqrt{\Delta^{2}+1}\,\sqrt{R+1}\,\|\hmaj\|_2\|\hmin\|_2}\,.
\end{align*}
\end{enumerate}

\end{lemma}

\begin{proof}
    By the NC property, to find the norms and angles of the embeddings, it suffices to analyze the mean-embeddings $\Mb$, for which, following Thm.~\ref{thm:SVM-VS}, we have $\Mb^T\Mb=\Ub\Lambdab\Ub^T.$ By Lemma \ref{lem:Zhat_SVD_general},
\begin{align*}
&\Ub\Lambdab\Ub^T = \frac{\sqrt{R}}{\Delta}
\begin{bmatrix}
\frac{1}{R}\Pb_m\Pb_m^T\otimes\ones_R\ones_R^T & 0 \\ 0 & 0
\end{bmatrix}
+\frac{2}{k\sqrt{R+1}\sqrt{\Delta^2+1}}\begin{bmatrix}
\ones_{Rm}\ones_{Rm}^T & -\ones_{Rm}\ones_{m}^T
\\
-\ones_{m}\ones_{Rm}^T & \ones_{m}\ones_{m}^T
\end{bmatrix}
+
\begin{bmatrix}
0 & 0 \\ 0 & \Pb_m\Pb_m^T
\end{bmatrix} \nn
\\
&\quad=
\begin{bmatrix}
\left(\frac{1}{\Delta\sqrt{R}}\Id_{k/2}-\frac{2}{k}\left(\frac{1}{\Delta\sqrt{R}}-\frac{1}{\sqrt{R+1}\sqrt{\Delta^2+1}}\right)\ones_{k/2}\ones_{k/2}^T\right)\otimes \ones_R\ones_R^T
&
-\frac{2}{k\sqrt{R+1}\sqrt{\Delta^2+1}}\ones_{k/2}\ones_{k/2}^T
\\
-\frac{2}{k\sqrt{R+1}\sqrt{\Delta^2+1}}\ones_{k/2}\ones_{k/2}^T
&
\Id_{k/2}-\frac{2}{k}\left(1-\frac{1}{\sqrt{R+1}\sqrt{\Delta^2+1}}\right)\ones_{k/2}\ones_{k/2}^T
\end{bmatrix}\,.
\end{align*}

The diagonal entries determine the norm of the embeddings as in part \textbf{(a)} and the off-diagonals entries specify the inner-product of each pair of the embeddings. Particularly,
\begin{align*}
\hmaj^T\hmaj' &= -\frac{2}{k}\left(\frac{1}{\Delta\sqrt{R}}-\frac{1}{\sqrt{R+1}\sqrt{\Delta^2+1}}\right)
\\
\hmin^T\hmin' &= -\frac{2}{k}\left(1-\frac{1}{\sqrt{R+1}\sqrt{\Delta^2+1}}\right)
\\
\hmin^T\hmaj &= -\frac{2}{k\sqrt{R+1}\sqrt{\Delta^2+1}}\,.
\end{align*}
Combining these with the norm calculations of part \textbf{(a)} completes the proof. 
\end{proof}

In the next lemma, we calculate the angles between an embedding and its corresponding classifier. Particularly, we give closed-form expression for $\Cos{\w_c}{\h_i}$ for $c\in[k], i: y_i=c$, which can be thought of the degree of alignment between classifiers and embeddings.

\begin{lemma}[CDT: Alignment of classifiers and embeddings]\label{lem:align} The angles between majority/minority embeddings and the their corresponding classifiers are all equal:
\begin{align*}
\Cos{\wmaj}{\hmaj} &= \frac{k\Delta^2 + (k-2)}{k\,\Delta\left(\Delta^{2} + 1\right)\|\wmaj\|_2\|\hmaj\|_2}, \nonumber\\ 
\Cos{\wmin}{\hmin} &= \frac{k\Delta^{-2} + (k-2)}{k\,\left(\Delta^{-2}+1\right)\|\wmin\|_2\|\hmin\|_2}.
\end{align*}

\end{lemma}
\begin{proof}
	Recalling $\W^T\Hb=\Zhat$, for all $c\in[k]$ and $i:y_i= c$ it holds that $\w_c^T\h_i=\Delta^{-1}\left(1-\frac{2}{k(\Delta^2+1)}\right)$ if $c$ is a majority class, and  $\w_c^T\h_i=\left(1-\frac{2}{k(\Delta^{-2}+1)}\right)$ otherwise.
\end{proof}
\subsubsection{Asymptotics}\label{sec:cdt_asymp}
We present the limiting values of the norm-ratios and angles in the asymptotic regime $\Delta={R^{\gamma}}, \gamma \in \R$ and $R\rightarrow\infty$. This parameterization is interesting because it can guide us on how to maintain finite angles between classifiers and embeddings as the imbalance ratio grows large. Specifically, the angles are as shown in Table~\ref{table:cdt_asymptotics}.
\begin{table}[h]
\centering
    \begin{tabular}{|c|c|c|c|}
    \hline
    $\cos({\w_c},{\w_c'})$& $\gamma < 1/6$ & $\gamma = 1/6$ & $\gamma > 1/6$ \\
    \hline
    $c, c' \in$ \text{minority} & $1$ & $0$ & $-\frac{2}{k-2}$\\
    \hline
    \hline
    $\cos({\w_c},{\w_c'})$& $\gamma < 0$ & $\gamma = 0$ & $\gamma > 0$ \\
    \hline
    $c, c' \in$ \text{majority} & $-\frac{2}{k-2}$ & $\frac{1-2\sqrt{2}}{1+\sqrt{2}(k-2)}$ & $0$\\
    \hline
    \hline
    $\cos({\h_c},{\h_c'})$& $\gamma < 0$ & $\gamma = 0$ & $\gamma > 0$ \\
    \hline
    $c, c' \in$ \text{minority} & $0$ & $-\frac{2}{k-2}$ & $-\frac{2}{k-2}$\\
    \hline
    $c, c' \in$ \text{majority} & $0$ & $\frac{2-2\sqrt{2}}{2+(k-2)\sqrt{2}}$ & $-\frac{2}{k-2}$ \\
    \hline
\end{tabular}
\caption{Asymptotic values of angles for CDT with $\Delta={R^{\gamma}}, \gamma \in \R$ and $R\rightarrow\infty$.}
\label{table:cdt_asymptotics}
\end{table}
\subsubsection{Centering}\label{sec:CDT_center}
Assuming that the classifiers follow the geometry in Thm. \ref{thm:SVM-VS}, $\w^*_c,~c\in[k]$ are centered around zero after some re-weighting, i.e. $\sum_{c\in[k]}\w^*_c/\delta_{c} =0$. This is immediate from $
\Vb^T\Db^{-1}\ones_k=0$ and ${\W^*}^T\W^*=\Vb\Lambdab\Vb^T$. \\
%
The embeddings $\h_i,i\in[n]$ are also not centered around zero in general. Instead, it holds that 
\begin{align}\label{eq:cdt_centering}
\sum_{i\in[n]}\frac{1}{n_{y_i}} \h^*_i = 0.
\end{align}
Note that this reduces to $\sum_{i\in[n]}\h^*_i$ for balanced data, and remains unchanged for any choice of the hyperparameters $\deltab$. Eqn. \eqref{eq:cdt_centering} is also equivalent to $\sum_{c\in[k]} \mub^*_c = 0$, with $\mub^*_c, ~c\in[k]$ the mean embeddings of each class.

\subsection{LDT Loss}\label{sec:ldt_geo}
\subsubsection{Norms and Angles}
From Thm. \ref{thm:SVM-VS}, solutions $(\W^*,\M^*\Db)$ of the CS-SVM  under LDT loss in \eqref{eq:svm_ldt}, follow the SELI geometry \citep{seli}, with imbalance ratio $\Tilde{R} = R(\delta_\text{min}/\delta_\text{maj})^2$. Thus, the corresponding norms and angles can be found by analyzing the $(\ones_k,\Tilde{R})$-SELI structure (up to a norm scaling by $\Db$ for the mean embeddings $\Mb^*$). We refer the reader to \citet[Sec. B.1]{seli} for closed form expressions of the $(\ones_k,\Tilde{R})$-SELI. We repeat some key formulas below for showing explicit dependence on $\Delta$.

\begin{corollary}[LDT: Norm ratios and classifier angles]\label{cor:ldt_norms_angles} 
For the optimal solution $(\W,\Hb)$ of the CS-SVM \eqref{eq:svm_ldt}:
\begin{align*}
&\frac{\|\wmaj\|_2^2}{\|\wmin\|_2^2} = \frac{(k-2)\sqrt{R}+{\sqrt{(R+\Delta^2)/2}}}{(k-2)\Delta+{\sqrt{(R+\Delta^2)/2}}},
\\
&\frac{\|\hmaj\|_2^2}{\|\hmin\|_2^2} = \frac{\frac{1}{\sqrt{R}}(k-2)+\frac{1}{\sqrt{(R+\Delta^2)/2}}}{(k-2)\Delta+\frac{\Delta^2}{\sqrt{(R+\Delta^2)/2}}},
\\
&\cos({\wmaj},{\wmaj'})=\frac{-2\sqrt{R}+\sqrt{{(R+\Delta^2)/2}}}{(k-2)\sqrt{R}+\sqrt{{(R+\Delta^2)/2}}},
\\
&\cos({\w_\text{\emph{min}}},{\w_\text{\emph{min}}'})=\frac{-2\Delta + \sqrt{(R+\Delta^2)/2}}{(k-2)\Delta+\sqrt{(R+\Delta^2)/2}}.
\end{align*}
\end{corollary}
\subsubsection{Asymptotics}
Similar to the calculations for the CDT case, we present the limiting values of the norm-ratios and angles in the asymptotic regime $\Delta={R^{\gamma}}, \gamma \in \R$ and $R\rightarrow\infty$. Then, the angles are as given in Table~\ref{table:ldt_asymptotics}:
\begin{table}[h]
\centering
    \begin{tabular}{|c|c|c|c|}
    \hline
    $\cos({\w_c},{\w_c'})$& $\gamma < 1/2$ & $\gamma = 1/2$ & $\gamma > 1/2$ \\
    \hline
    $c, c' \in$ \text{minority} & $1$ & $-\frac{1}{k-1}$ & $\frac{1-2\sqrt{2}}{1+\sqrt{2}(k-2)}$\\
    \hline
    $c, c' \in$ \text{majority} & $\frac{1-2\sqrt{2}}{1+\sqrt{2}(k-2)}$ & $-\frac{1}{k-1}$ & $1$ \\
    \hline
    \hline
    $\cos({\h_c},{\h_c'})$& $\gamma < 1/2$ & $\gamma = 1/2$ & $\gamma > 1/2$ \\
    \hline
    $c, c' \in$ \text{minority} & $-\frac{2}{k-2}$ & $-\frac{1}{k-1}$ & $\frac{2-2\sqrt{2}}{2+\sqrt{2}(k-2)}$\\
    \hline
    $c, c' \in$ \text{majority} & $-\frac{\sqrt{2}}{-\sqrt{2}+k(1+\sqrt{2})}$ & $-\frac{1}{k-1}$ & $-\frac{2}{k-2}$ \\
    \hline
\end{tabular}
\caption{Asymptotic values of angles for LDT with $\Delta={R^{\gamma}}, \gamma \in \R$ and $R\rightarrow\infty$.}
\label{table:ldt_asymptotics}
\end{table}
\subsubsection{Centering}\label{sec:LDT_center}
The optimal classifiers and features $(\W^*, \Mb^*\Db)$ follow the $(\ones_k,\Tilde{R})$-SELI structure. Thus (see \citet[Sec.~B.1.4]{seli}), the classifiers $\w^*_c, c\in[k]$ are centered around zero. However the embeddings are centered around zero after a reweighting that depends both on $\delta_c$ and $n_c,~c\in[k]$. Specifically, $\sum_{c\in[k]} \delta_c{\mub}^*_c = 0$, or equivalently,
\begin{align}\label{eq:ldt_centering}
   \sum_{i\in[n]}{\frac{\delta_{y_i}}{n_{y_i}} {\h}^*_i = 0}. 
\end{align}

%% file: sections/Additional_Experiments.tex
In this section, we provide additional details and discussions on our experiments. 
\subsection{{Additional Experimental Details}} \label{sec:app_exp_details}
In Sec.~\ref{sec:num_results}, we investigated the convergence of SGD steps for CDT/LDT loss in \eqref{eq:CDT_loss}/\eqref{eq:LDT_loss} to the implicit geometries of Thm.~\ref{thm:SVM-VS}. Here, we describe the experimental setup in more details.

\vspace{3pt}
\noindent\textbf{UFM experiments.}
 We train the UFM as a two-layer network (no biases) with $n = 275$ inputs, $d=20$ hidden units and $k=10$ classes, trained on the basis vectors in $\R^n$. The labels for each vector are chosen such that the dataset is $(R=10,\rho={1}/{2})$-STEP imbalanced, with $\nmin=5$ and a batch size of $5$. We further use STEP logit adjustment, and choose $\Delta = R^\gamma$ with $\gamma\in[-1.5,1.5]$. We train all models with the same constant learning rate for $6000$ epochs. We normalize $\deltab$ {so that $\ones_k^T\deltab=k$}, since we empirically observe that for a fixed  ratio $\Delta$ the convergence speed depends on the magnitude of $\deltab$.

\vspace{3pt}
\noindent\textbf{Deep-net experiments.}
We train (i) ResNet18 on CIFAR10, and (ii) an MLP with batch-norm and ReLU activations on MNIST, both under a $R=10$ imbalance ratio. The MLP model consists of 6  fully-connected layers of width 2048, each followed by batch-norm and ReLU activations. We train the models for 350 epochs with an initial learning rate of $0.1$ reduced at epochs 116 and 232 by a factor of $10$, with a batch size of 128. Following the same setting as in \cite{NC, seli}, we set momentum and weight decay to $0.9$ and $10^{-5}$ respectively. We also normalize $\deltab$ to sum to $k$, similar to UFM. We perform the experiments in Fig.~\ref{fig:convergence_to_theory} without any data augmentation, remaining consistent with previous works on neural collapse \citep[e.g.,][]{NC}). 

\vspace{3pt}
\noindent\textbf{Generalization experiments (Sec.~\ref{sec:test_numerical}).} In order to achieve results close to the state-of-the-art, we perform data augmentation as in \cite{TengyuMa,CDT} on all datasets: images are resized to $32\times32$, padded, randomly flipped and cropped. For MLP experimets, we use the same architecture described above. Our best balanced test accuracy is comparable to previous results by \cite{CDT}. We highlight two differences in our experiments. First, we use ResNet18 instead of ResNet32, and second, we train the models without applying the delayed reweighting (DRW) technique \citep{TengyuMa}. Although DRW has been shown to improve performance empirically, we refrained from using it in our experiments to evaluate the performance of CDT and LDT independent of reweighting.


We conducted additional experiments with imbalance ratio $R=2,5,20$. We have not included those results due to their similarity to $R= 10$; however, we will discuss the impacts of higher imbalance ratio $R$ and hyperparameter $\gamma$ in the following section.

\subsection{Speed of Convergence} \label{sec:OptIm_Discuss}
{We empirically observe that the UFM parameters converge more slowly to the global optimizers in Thm.~\ref{thm:SVM-VS} as the imbalance ratio $R$ and hyperparameter $|\gamma|$ increase. A similar observation for large values of $R$ is also reported in \cite{seli}. To illustrate the speed of convergence, we measure the distance of the SGD steps to the predicted implicit geometry during training. In particular, at each step $(\W_t,\M_t)$, we compute $\norm{\frac{\W_t^T\W_t}{\norm{\W_t^T\W_t}} - \frac{{\W^*}^T\W^*}{\norm{{\W^*}^T\W^*}}}_F$  for the classifiers and $\norm{\frac{\M_t^T\M_t}{\norm{\M_t^T\M_t}} - \frac{{\M^*}^T\M^*}{\norm{{\M^*}^T\M^*}}}_F$ for the centered mean-embeddings, where $(\W^*,\M^*)$ are as described by Thm.~\ref{thm:SVM-VS}. Fig.~\ref{fig:OPTIM_Fig} illustrates the convergence behaviour of the parameters for UFM and ResNet18. While as training progresses, the classifiers/embeddings get closer to the predicted geometry, imbalance ratio and hyperparameter values can significantly slow down the convergence. This behaviour appear for both UFM and deep-net experiments.
}

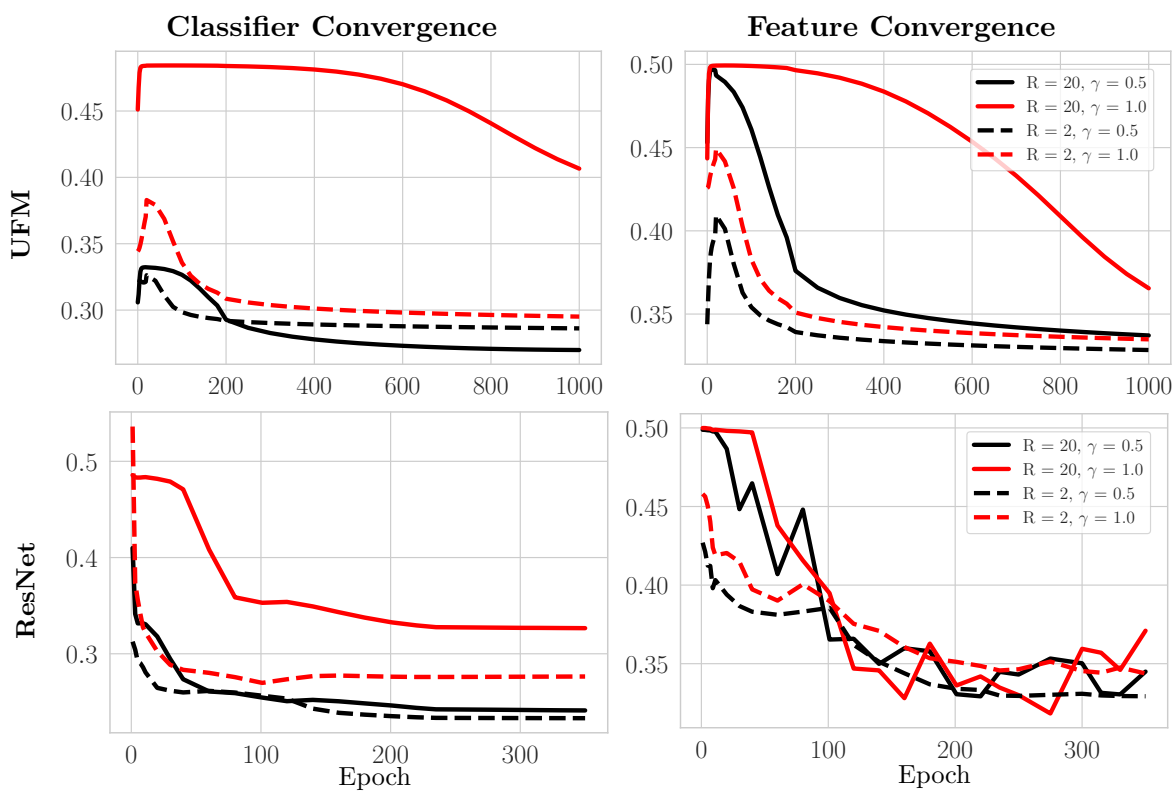
\begin{figure*}[ht]
\centering
\begin{subfigure}{.49\textwidth}
  \centering
  \begin{tikzpicture}
	\node at (0,0.0)
    {\includegraphics[width=0.95\linewidth]{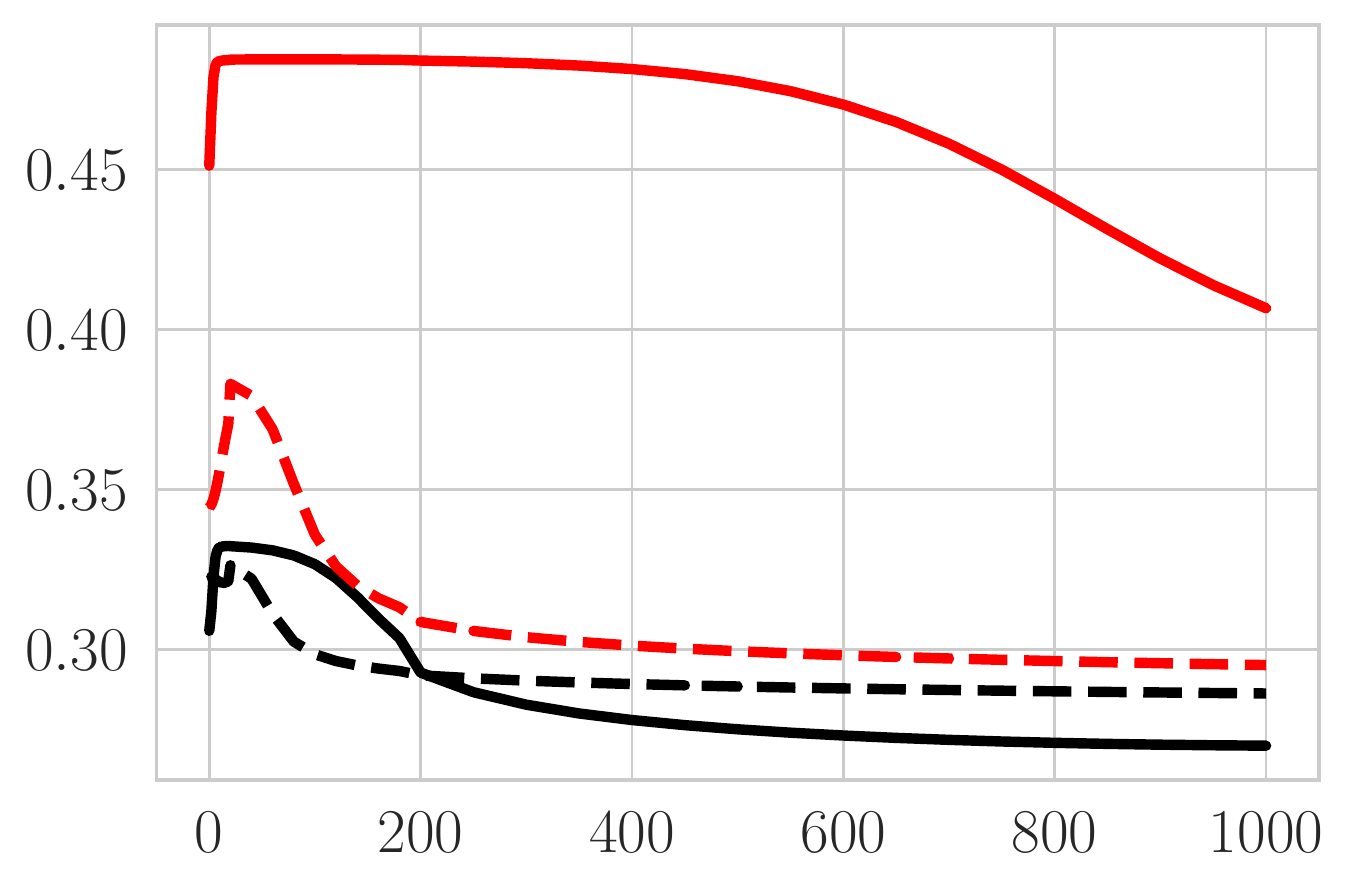}};
	\node at (0.0, 2.6)  [scale=1.0, rotate=0]{\textbf{Classifier Convergence}};
 \node at (-4.1,-0.0)  [scale=0.9, rotate=90]{\textbf{UFM}};
  \end{tikzpicture}
\end{subfigure}
\begin{subfigure}{.49\textwidth}
  \centering
  \begin{tikzpicture}
	\node at (0,0.0)
    {\includegraphics[width=0.95\linewidth]{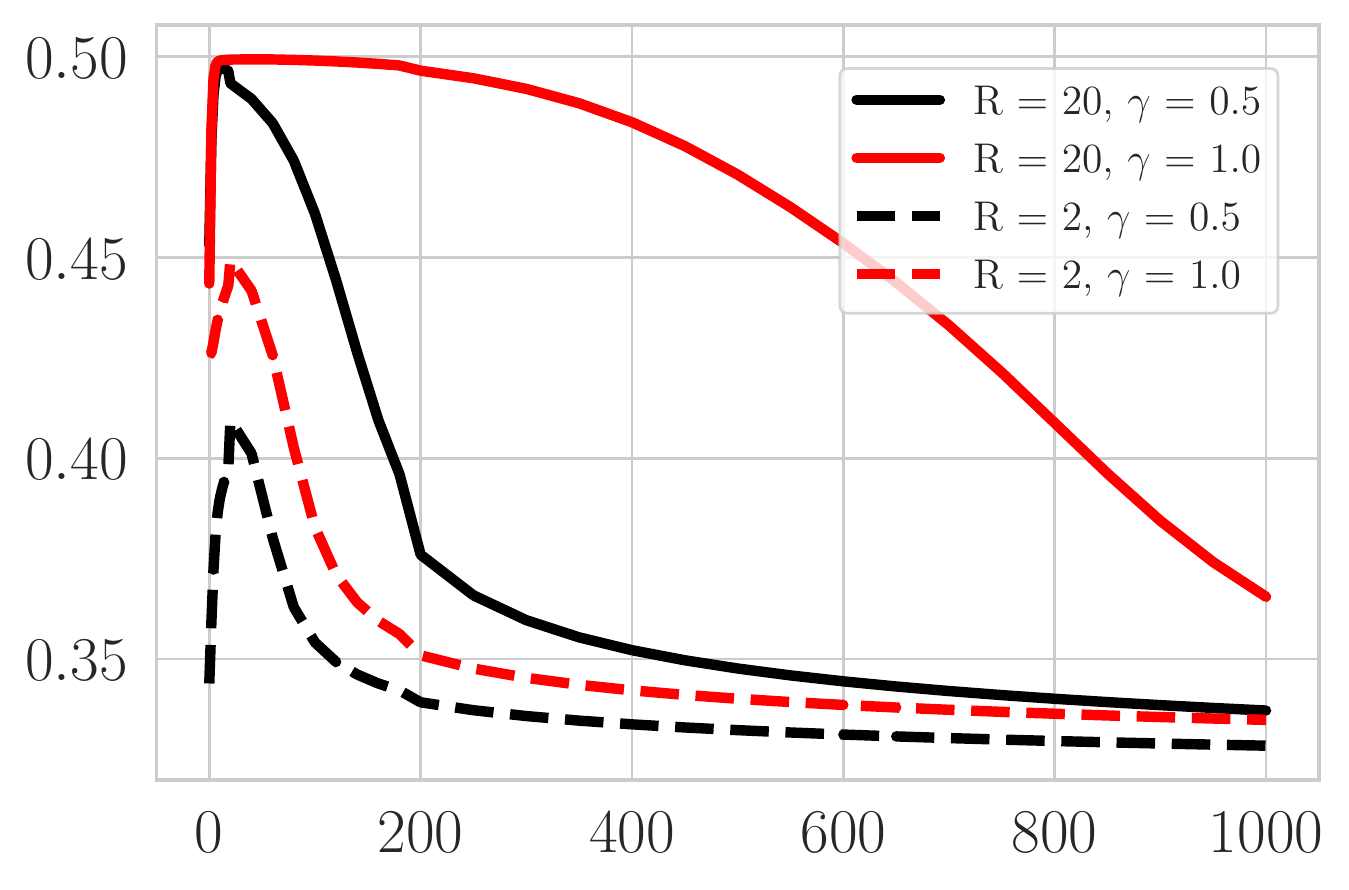}};
    \node at (0.0, 2.6)  [scale=1.0, rotate=0]{\textbf{Feature Convergence}};
  \end{tikzpicture}
\end{subfigure}
\vspace{-0.4cm}

\begin{subfigure}{.49\textwidth}
  \centering
  \begin{tikzpicture}
	\node at (0,0.0)
    {\includegraphics[width=0.95\linewidth]{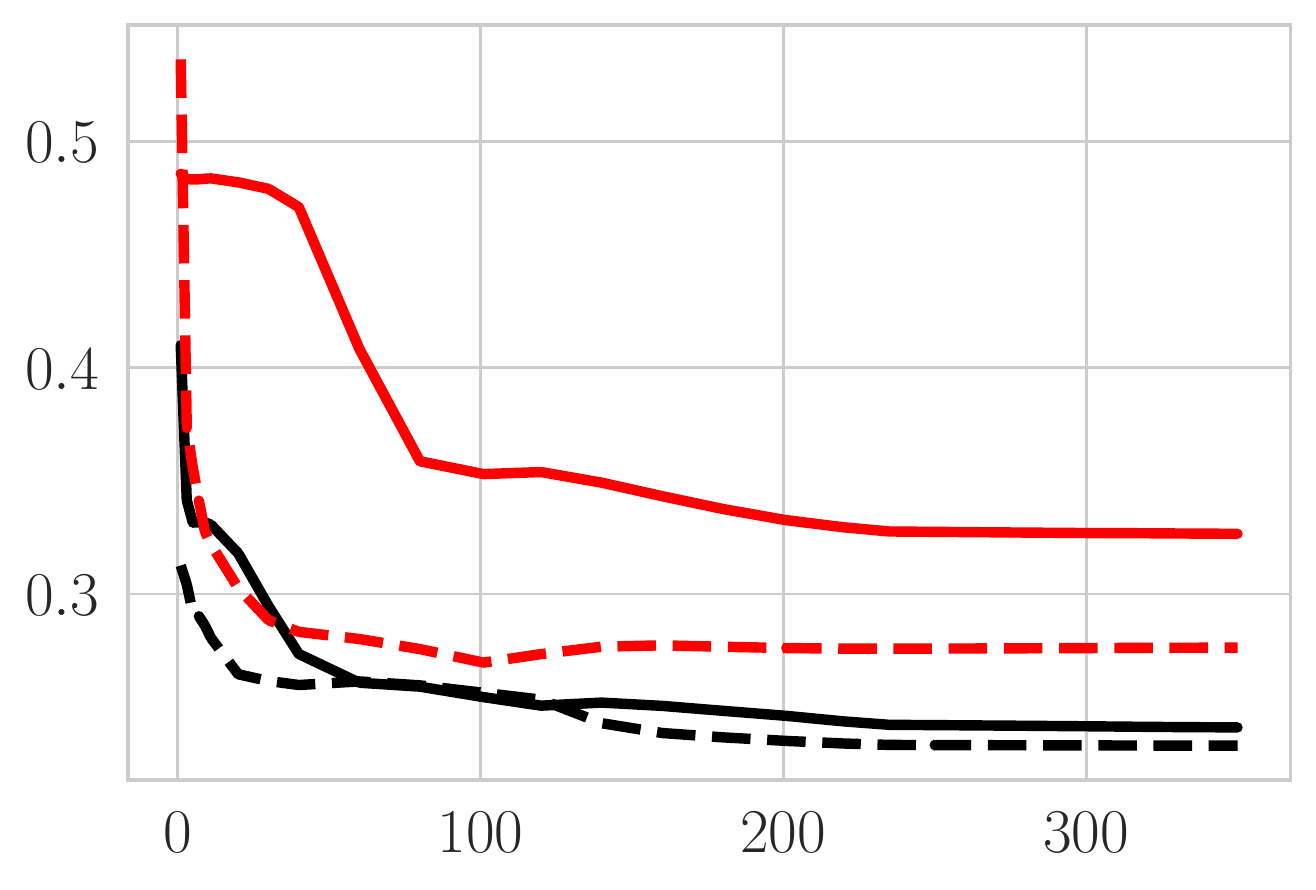}};
	\node at (0.5,-2.5) [scale=0.9]{Epoch};
	\node at (-4.1,-0.0)  [scale=0.9, rotate=90]{\textbf{ResNet}};
  \end{tikzpicture}
\end{subfigure}%
\begin{subfigure}{.49\textwidth}
  \centering
  \begin{tikzpicture}
	\node at (0,0.0)
    {\includegraphics[width=0.95\linewidth]{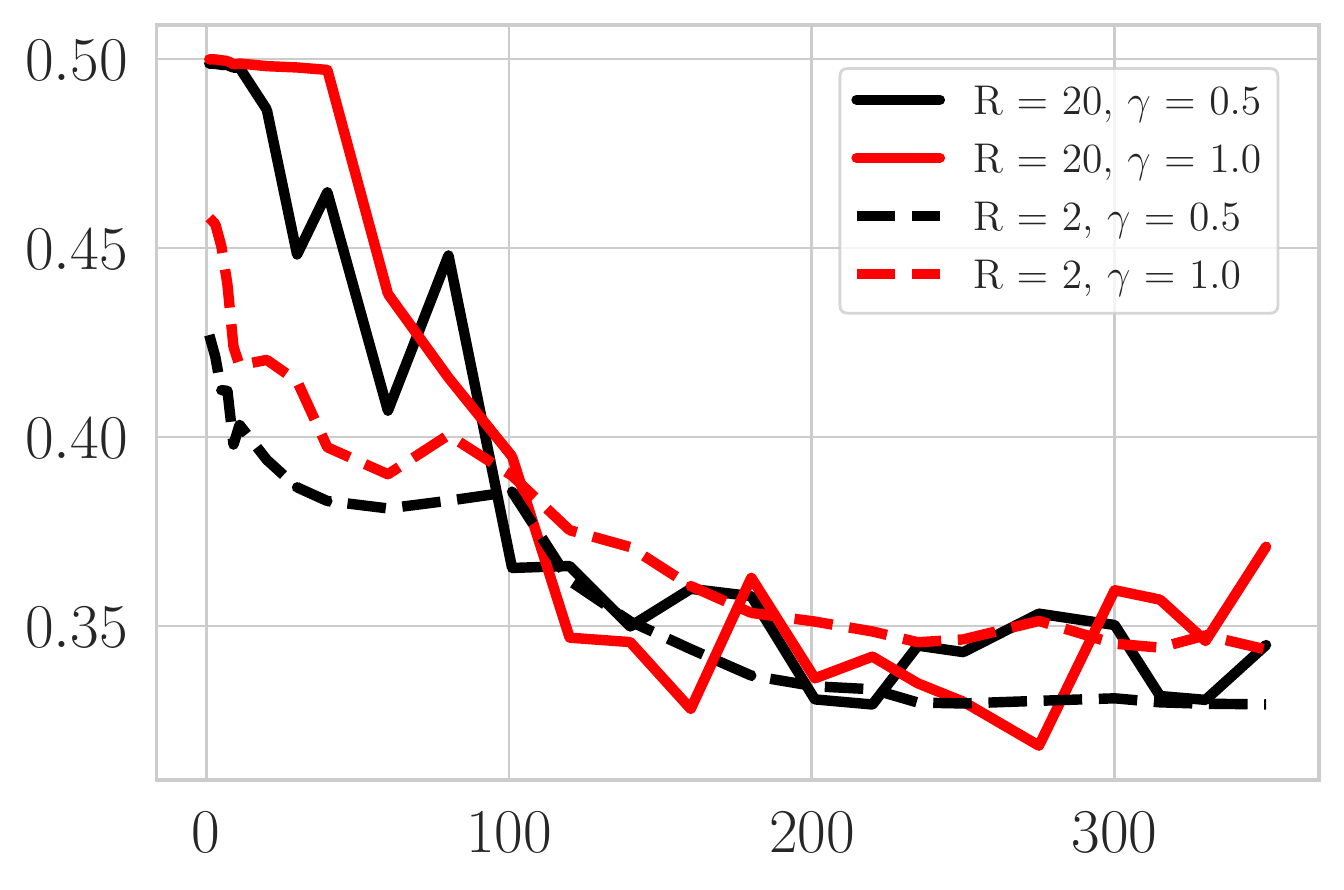}};
	\node at (0.5,-2.5) [scale=0.9]{Epoch};
  \end{tikzpicture}
\end{subfigure}
\captionsetup{width=0.95\linewidth}
\vspace{-0.25cm}
\caption{{Convergence of classifiers and mean-embeddings to the implicit geometry in Thm.~\ref{thm:SVM-VS}: (first row) UFM, (second row) ResNet18 trained on CIFAR10. The models are trained by SGD on CDT loss. Larger $R$ and $\gamma$ lead to slower convergence to the expected structure. 
}}
\label{fig:OPTIM_Fig}
\end{figure*}

{
In addition to the worse convergence, it becomes more challenging to achieve zero training error as $|\gamma|$ increases. We illustrate this in Fig.~\ref{fig:train_acc_resnet_perclass}, where we report the training accuracy of ResNet model trained on imbalanced CIFAR10 per epoch.
We empirically observe that it is in general easier to enter the zero-error regime by LDT loss. On the other hand, we do not achieve $100\%$ training accuracy for large values of $\gamma$ on CDT loss. This is consistent with similar observation on CDT training in \cite{VS}.
}

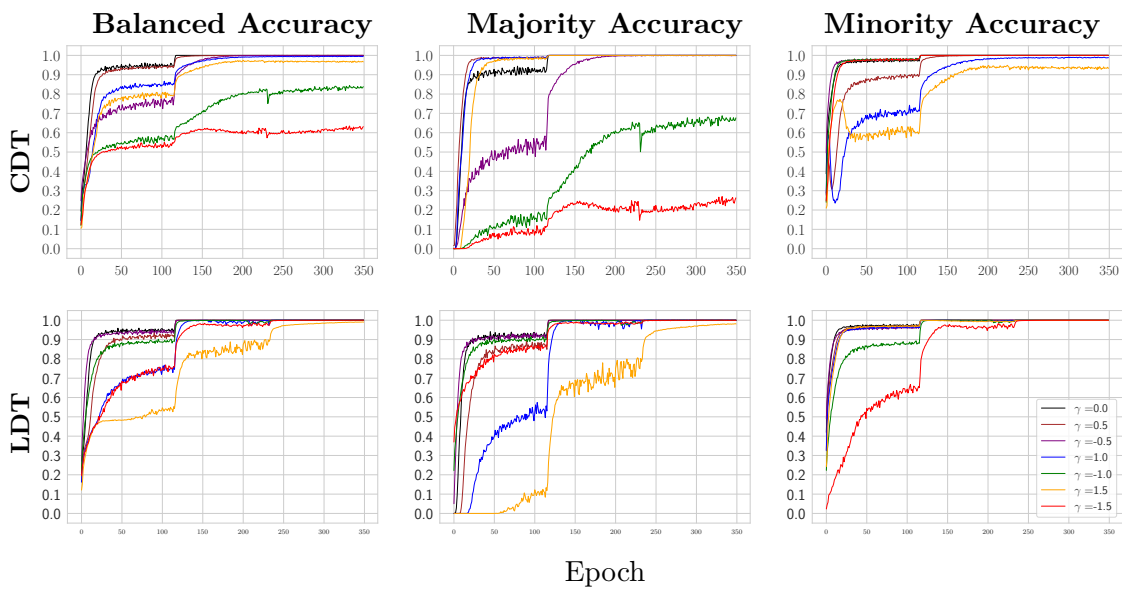
\begin{figure*}[t]
\centering
\begin{subfigure}{.95\textwidth}
  \centering
  \begin{tikzpicture}
	\node at (0,0.0)
    {\includegraphics[width=0.95\linewidth]{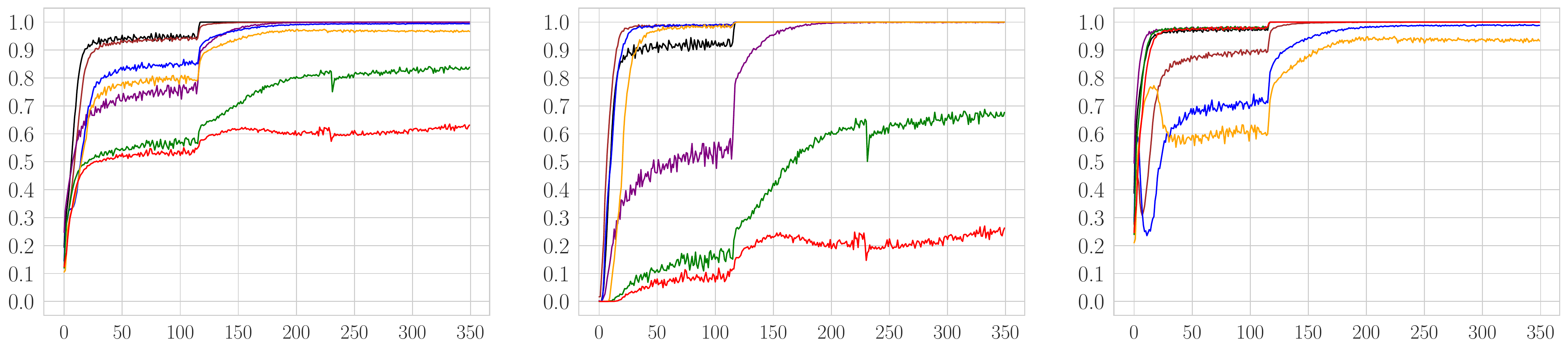}};
	\node at (-4.6,1.8)  [scale=1.0]{\textbf{Balanced Accuracy}};
	\node at (0.3,1.8)  [scale=1.0]{\textbf{Majority Accuracy}};
	\node at (5.0,1.8)  [scale=1.0]{\textbf{Minority Accuracy}};
    \node at (-7.4,-0.0)  [scale=0.9, rotate=90]{\textbf{CDT}};
  \end{tikzpicture}
  \begin{tikzpicture}
  \node at (0,0.0)
    {\includegraphics[width=0.95\linewidth]{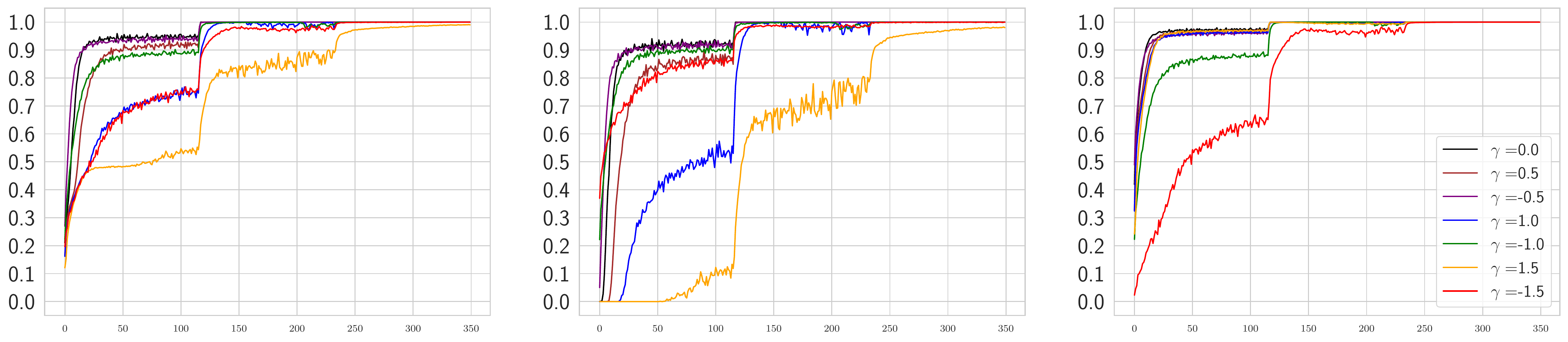}};
    \node at (-7.4,-0.0)  [scale=0.9, rotate=90]{\textbf{LDT}};
    \node at (0.3,-2.0) [scale=1.0]{Epoch};
  \end{tikzpicture}
\end{subfigure}%
\vspace{-0.25cm}
\captionsetup{width=0.95\linewidth}
\caption{Training accuracy across epochs of ResNet18 model trained on ($R=10$, $\rho$=1/2)-STEP imbalanced CIFAR10 dataset with CDT/LDT loss and different values of $\gamma$. It becomes harder to enter zero training error regime for larger $|\gamma|$ with the impact being more noticable on CDT.}
\label{fig:train_acc_resnet_perclass}
\end{figure*}

%% file: refs.bib
@article{galanti2022improved,
  title={Generalization Bounds for Transfer Learning with Pretrained Classifiers},
  author={Galanti, Tomer and Gy{\"o}rgy, Andr{\'a}s and Hutter, Marcus},
  journal={arXiv preprint arXiv:2212.12532},
  year={2022}
}

@article{seli,
  title={Imbalance Trouble: Revisiting Neural-Collapse Geometry},
  author={Thrampoulidis, Christos and Kini, Ganesh R and Vakilian, Vala and Behnia, Tina},
  journal={arXiv preprint arXiv:2208.05512},
  year={2022}
}

@article{zhou2022all,
  title={Are All Losses Created Equal: A Neural Collapse Perspective},
  author={Zhou, Jinxin and You, Chong and Li, Xiao and Liu, Kangning and Liu, Sheng and Qu, Qing and Zhu, Zhihui},
  journal={arXiv preprint arXiv:2210.02192},
  year={2022}
}

@article{wang2021importance,
  title={Is Importance Weighting Incompatible with Interpolating Classifiers?},
  author={Wang, Ke Alexander and Chatterji, Niladri S and Haque, Saminul and Hashimoto, Tatsunori},
  journal={arXiv preprint arXiv:2112.12986},
  year={2021}
}

@article{jitkrittum2022elm,
  title={ELM: Embedding and Logit Margins for Long-Tail Learning},
  author={Jitkrittum, Wittawat and Menon, Aditya Krishna and Rawat, Ankit Singh and Kumar, Sanjiv},
  journal={arXiv preprint arXiv:2204.13208},
  year={2022}
}

@article{yang2022we,
  title={Do We Really Need a Learnable Classifier at the End of Deep Neural Network?},
  author={Yang, Yibo and Xie, Liang and Chen, Shixiang and Li, Xiangtai and Lin, Zhouchen and Tao, Dacheng},
  journal={arXiv preprint arXiv:2203.09081},
  year={2022}
}

@article{VS,
  title={Label-imbalanced and group-sensitive classification under overparameterization},
  author={Kini, Ganesh Ramachandra and Paraskevas, Orestis and Oymak, Samet and Thrampoulidis, Christos},
  journal={Advances in Neural Information Processing Systems},
  volume={34},
  pages={18970--18983},
  year={2021}
}

@inproceedings{graf2021dissecting,
  title={Dissecting supervised constrastive learning},
  author={Graf, Florian and Hofer, Christoph and Niethammer, Marc and Kwitt, Roland},
  booktitle={International Conference on Machine Learning},
  pages={3821--3830},
  year={2021},
  organization={PMLR}
}

@article{galanti2021role,
  title={On the Role of Neural Collapse in Transfer Learning},
  author={Galanti, Tomer and Gy{\"o}rgy, Andr{\'a}s and Hutter, Marcus},
  journal={arXiv preprint arXiv:2112.15121},
  year={2021}
}

@article{hui2022limitations,
  title={Limitations of neural collapse for understanding generalization in deep learning},
  author={Hui, Like and Belkin, Mikhail and Nakkiran, Preetum},
  journal={arXiv preprint arXiv:2202.08384},
  year={2022}
}

@article{lyu2019gradient,
  title={Gradient descent maximizes the margin of homogeneous neural networks},
  author={Lyu, Kaifeng and Li, Jian},
  journal={arXiv preprint arXiv:1906.05890},
  year={2019}
}

@article{zhou2022optimization,
  title={On the Optimization Landscape of Neural Collapse under MSE Loss: Global Optimality with Unconstrained Features},
  author={Zhou, Jinxin and Li, Xiao and Ding, Tianyu and You, Chong and Qu, Qing and Zhu, Zhihui},
  journal={arXiv preprint arXiv:2203.01238},
  year={2022}
}

@article{zhu2021geometric,
  title={A Geometric Analysis of Neural Collapse with Unconstrained Features},
  author={Zhu, Zhihui and Ding, Tianyu and Zhou, Jinxin and Li, Xiao and You, Chong and Sulam, Jeremias and Qu, Qing},
  journal={Advances in Neural Information Processing Systems},
  volume={34},
  year={2021}
}

@article{han2021neural,
  title={Neural collapse under mse loss: Proximity to and dynamics on the central path},
  author={Han, XY and Papyan, Vardan and Donoho, David L},
  journal={arXiv preprint arXiv:2106.02073},
  year={2021}
}

@article{tirer2022extended,
  title={Extended unconstrained features model for exploring deep neural collapse},
  author={Tirer, Tom and Bruna, Joan},
  journal={arXiv preprint arXiv:2202.08087},
  year={2022}
}

@article{lu2022importance,
  title={Importance Tempering: Group Robustness for Overparameterized Models},
  author={Lu, Yiping and Ji, Wenlong and Izzo, Zachary and Ying, Lexing},
  journal={arXiv preprint arXiv:2209.08745},
  year={2022}
}

@article{MaDoWeNeed,
  title={Do We Need Neural Collapse? Learning Diverse Features for Fine-grained and Long-tail Classification},
  author={Ma, Jiawei and You, Chong and Reddi, Sashank J and Jayasumana, Sadeep and Jain, Himanshu and Yu, Felix and Chang, Shih-Fu and Kumar, Sanjiv}
}

@article{le2022training,
  title={Training invariances and the low-rank phenomenon: beyond linear networks},
  author={Le, Thien and Jegelka, Stefanie},
  journal={arXiv preprint arXiv:2201.11968},
  year={2022}
}

@article{jacot2022implicit,
  title={Implicit Bias of Large Depth Networks: a Notion of Rank for Nonlinear Functions},
  author={Jacot, Arthur},
  journal={arXiv preprint arXiv:2209.15055},
  year={2022}
}

@article{xie2022neural,
  title={Neural Collapse Inspired Attraction-Repulsion-Balanced Loss for Imbalanced Learning},
  author={Xie, Liang and Yang, Yibo and Cai, Deng and Tao, Dacheng and He, Xiaofei},
  journal={arXiv preprint arXiv:2204.08735},
  year={2022}
}

@article{lu2020neural,
  title={Neural collapse with cross-entropy loss},
  author={Lu, Jianfeng and Steinerberger, Stefan},
  journal={arXiv preprint arXiv:2012.08465},
  year={2020}
}

@article{mixon2020neural,
  title={Neural collapse with unconstrained features},
  author={Mixon, Dustin G and Parshall, Hans and Pi, Jianzong},
  journal={arXiv preprint arXiv:2011.11619},
  year={2020}
}

@article{fang2021exploring,
  title={Exploring deep neural networks via layer-peeled model: Minority collapse in imbalanced training},
  author={Fang, Cong and He, Hangfeng and Long, Qi and Su, Weijie J},
  journal={Proceedings of the National Academy of Sciences},
  volume={118},
  number={43},
  year={2021},
  publisher={National Acad Sciences}
}

@article{ULPM,
  title={An unconstrained layer-peeled perspective on neural collapse},
  author={Ji, Wenlong and Lu, Yiping and Zhang, Yiliang and Deng, Zhun and Su, Weijie J},
  journal={arXiv preprint arXiv:2110.02796},
  year={2021}
}

@article{ji2020directional,
  title={Directional convergence and alignment in deep learning},
  author={Ji, Ziwei and Telgarsky, Matus},
  journal={Advances in Neural Information Processing Systems},
  volume={33},
  pages={17176--17186},
  year={2020}
}

@misc{kang2020decoupling,
      title={Decoupling Representation and Classifier for Long-Tailed Recognition}, 
      author={Bingyi Kang and Saining Xie and Marcus Rohrbach and Zhicheng Yan and Albert Gordo and Jiashi Feng and Yannis Kalantidis},
      year={2020},
      eprint={1910.09217},
      archivePrefix={arXiv},
      primaryClass={cs.CV}
}

@misc{lin2018focal,
      title={Focal Loss for Dense Object Detection}, 
      author={Tsung-Yi Lin and Priya Goyal and Ross Girshick and Kaiming He and Piotr Dollár},
      year={2018},
      eprint={1708.02002},
      archivePrefix={arXiv},
      primaryClass={cs.CV}
}

@article{NC,
  title={Prevalence of neural collapse during the terminal phase of deep learning training},
  author={Papyan, Vardan and Han, XY and Donoho, David L},
  journal={Proceedings of the National Academy of Sciences},
  volume={117},
  number={40},
  pages={24652--24663},
  year={2020},
  publisher={National Acad Sciences}
}

@misc{CDT,
      title={Identifying and Compensating for Feature Deviation in Imbalanced Deep Learning}, 
      author={Han-Jia Ye and Hong-You Chen and De-Chuan Zhan and Wei-Lun Chao},
      year={2020},
      eprint={2001.01385},
      archivePrefix={arXiv},
      primaryClass={cs.LG}
}

@ARTICLE{KimKim,  author={Kim, Byungju and Kim, Junmo},  journal={IEEE Access},   title={Adjusting Decision Boundary for Class Imbalanced Learning},   year={2020},  volume={8},  number={},  pages={81674-81685},  doi={10.1109/ACCESS.2020.2991231}}

@inproceedings{byrd2019effect,
  title={What is the effect of importance weighting in deep learning?},
  author={Byrd, Jonathon and Lipton, Zachary},
  booktitle={International Conference on Machine Learning},
  pages={872--881},
  year={2019},
  organization={PMLR}
}

@article{sagawa2019distributionally,
  title={Distributionally robust neural networks for group shifts: On the importance of regularization for worst-case generalization},
  author={Sagawa, Shiori and Koh, Pang Wei and Hashimoto, Tatsunori B and Liang, Percy},
  journal={arXiv preprint arXiv:1911.08731},
  year={2019}
}

@article{soudry2018implicit,
  title={The implicit bias of gradient descent on separable data},
  author={Soudry, Daniel and Hoffer, Elad and Nacson, Mor Shpigel and Gunasekar, Suriya and Srebro, Nathan},
  journal={The Journal of Machine Learning Research},
  volume={19},
  number={1},
  pages={2822--2878},
  year={2018},
  publisher={JMLR. org}
}

@article{ji2018risk,
  title={Risk and parameter convergence of logistic regression},
  author={Ji, Ziwei and Telgarsky, Matus},
  journal={arXiv preprint arXiv:1803.07300},
  year={2018}
}

@inproceedings{sagawa2020investigation,
  title={An investigation of why overparameterization exacerbates spurious correlations},
  author={Sagawa, Shiori and Raghunathan, Aditi and Koh, Pang Wei and Liang, Percy},
  booktitle={International Conference on Machine Learning},
  pages={8346--8356},
  year={2020},
  organization={PMLR}
}

@article{li2021autobalance,
  title={AutoBalance: Optimized Loss Functions for Imbalanced Data},
  author={Li, Mingchen and Zhang, Xuechen and Thrampoulidis, Christos and Chen, Jiasi and Oymak, Samet},
  journal={Advances in Neural Information Processing Systems},
  volume={34},
  pages={3163--3177},
  year={2021}
}

@article{khan2017cost,
  title={Cost-sensitive learning of deep feature representations from imbalanced data},
  author={Khan, Salman H and Hayat, Munawar and Bennamoun, Mohammed and Sohel, Ferdous A and Togneri, Roberto},
  journal={IEEE transactions on neural networks and learning systems},
  volume={29},
  number={8},
  pages={3573--3587},
  year={2017},
  publisher={IEEE}
}

@inproceedings{TengyuMa,
  title={Learning imbalanced datasets with label-distribution-aware margin loss},
  author={Cao, Kaidi and Wei, Colin and Gaidon, Adrien and Arechiga, Nikos and Ma, Tengyu},
  booktitle={Advances in Neural Information Processing Systems},
  pages={1567--1578},
  year={2019}
}

@article{Menon,
  title={Long-tail learning via logit adjustment},
  author={Menon, Aditya Krishna and Jayasumana, Sadeep and Rawat, Ankit Singh and Jain, Himanshu and Veit, Andreas and Kumar, Sanjiv},
  journal={arXiv preprint arXiv:2007.07314},
  year={2020}
}
